\documentclass[twoside,11pt]{article}

\usepackage{blindtext}

\usepackage{amsmath}
\usepackage{amssymb}
\usepackage{booktabs}
\usepackage{xltabular}
\usepackage{makecell}
\usepackage{multirow}
\usepackage{bm}
\usepackage{bbm}
\usepackage{fancyhdr}
\usepackage{subcaption}
\usepackage{placeins}
\usepackage{algorithm}
\usepackage{algpseudocode}
\usepackage{enumitem}
\usepackage[preprint]{jmlr2e}
\usepackage{csquotes}

\usepackage{lastpage}
\jmlrheading{23}{2022}{1-\pageref{LastPage}}{1/21; Revised 5/22}{9/22}{21-0000}{Spear and Komorowski and Pope and Moodie}

\ShortHeadings{Kernel weighted importance sampling for OPE in CBs}{Spear and Komorowski and Pope and Moodie}
\firstpageno{1}

\newcommand{\pearldo}{\mathrm{do}}
\newcommand{\argmax}{\mathrm{argmax}}

\begin{document}

\title{Kernel weighted importance sampling for off-policy evaluation in contextual bandits}

\author{\name Joshua Spear \email joshua.spear.21@ucl.ac.uk \\
       \addr Institute of Child Health\\
       University College London\\
       London, WC1E 6BT, UK
       \AND
       \name Matthieu Komorowski \email matthieu@strive-health.com \\
       \addr Strive Health Ltd \\
       Enfield, EN2 8DJ, UK
       \AND
       \name Rebecca Pope \email r.pope@ucl.ac.uk \\
       \addr Institute of Child Health\\
       University College London\\
       London, WC1E 6BT, UK
       \AND
       Erica EM Moodie \email erica.moodie@mcgill.ca \\
       \addr Department of Epidemiology\\
       Biostatistics\\
       McGill University\\
       Montr\'eal, QC H3A 0G4, Canada
       }

\editor{My editor}

\maketitle

\begin{abstract}
This article presents a novel estimator for performing off-policy evaluation using only offline data for contextual bandits. The proposed estimator, Kernel-WIS is demonstrated to be asymptotically consistent and to empirically outperform strong baselines (including weighted importance sampling), particularly under behaviour policy miss-specification. The benefit of Kernel-WIS is derived from combining the bounded property of weighted importance sampling with the linearity of vanilla importance sampling.
\end{abstract}


\section{Introduction}\label{sec:intro}
Off-policy evaluation (OPE) for contextual bandits, or more generally time independent causal inference, requires evaluating the performance of a target policy $\pi_{e}$ using data observed from a (logging) policy $\pi_{l}$. If interest is in the expected value of this performance (i.e., the average treatment effect, average treatment effect in the treated etc), then the OPE task requires approximating $\mathbb{E}_{P_{\pi_{e}}(\tau)}[R]$ (or a conditional version) with data generated according to $P_{\pi_{l}}$, where $R$ is the outcome and $P_{\pi}$ is the probability measure associated with the random variable (policy), $\pi:\mathcal{S}\rightarrow\mathcal{A}$. $\mathcal{S}$ defines the set of states or contexts over which the agent acts and $\mathcal{A}$ defines the set of actions that the agent can take. $S$ and $A$ define random variables with ranges in $\mathcal{S}$ and $\mathcal{A}$, respectively, with associated push through measures implied by the context. Unless otherwise stated, capital letters will refer to random variables ($S$), capital $P$ subscript will refer to the probability measure associated to the random variable ($P_{S}$) and lowercase $p$ will refer to the density ($p_{S}$). Throughout the analysis the following data generating density is assumed:
\begin{align}
	p_{\pi}(\tau) = p_{S}(s)p_{\pi}(a|s)p_{R}(r|a,s).\label{equ:cb_density}
\end{align}

The causal estimand of primary interest for this analysis is the average treatment return in the untreated i.e.,:
\begin{equation}
	\mathbb{E}[R^{\pearldo{}(A=\pi_{e}(S))}|A=\pi_{\beta}(S)],\label{equ:atrut}
\end{equation}
where $R^{\pearldo{}(A=\pi(S))}$ defines the counterfactual reward $R$ when the intervention is set by the policy $\pi$ (\cite{pearl_causality_2009}). For counterfactual policy optimisation, the selected policy ($\pi_{e}$) can depart quite significantly from the logging policy used to generate the data ($\pi_{l}$), particularly in the time dependent setting. As such, the causal estimand is understanding \enquote{how the new policy would have performed on a set of patients who have not been treated under the new policy}. Throughout, the causal identifiability assumptions of no hidden confounding, consistency and positivity are assumed (\cite{Hernan2020}) such that:
\begin{align*}
	\mathbb{E}[R^{\pearldo{}(A=\pi_{e}(S))}|A=\pi_{\beta}(S)] = \mathbb{E}_{P_{\pi_{l}}(\tau)}\Bigg[\frac{p_{\pi_{e}}(A|S)}{p_{\pi_{\beta}}(A|S)}R\Bigg].
\end{align*}

With respect to OPE, the core approaches are outcome and propensity methods, with the addition of composite and doubly robust methods requiring the two aforementioned approaches. Equations \ref{equ:propense_approx} and \ref{equ:outcome_approx} broadly define the propensity and outcome based methods mentioned, where $p_{\pi_{\beta}}(a_{i}|s_{i})$ and $\hat{R}(\pi_{e}(s_{i}),s_{i})$ define the estimated versions of the assumed models for each estimator, respectively. Throughout, the approximate logging policy, $p_{\pi_{\beta}}(a_{i}|s_{i})$, will be referred to as the \enquote{behaviour policy}.\\
\begin{align}
	\mathbb{E}_{P_{\pi_{e}}}[R] &= \mathbb{E}_{P_{\pi_{l}}(\tau)}\Bigg[\frac{p_{\pi_{e}}(A|S)}{p_{\pi_{l}}(A|S)}R\Bigg] \approx \frac{1}{n}\sum_{i=1}^{n}\frac{p_{\pi_{e}}(a_{i}|s_{i})}{p_{\pi_{\beta}}(a_{i}|s_{i})}r_{i} \label{equ:propense_approx} \\
	\mathbb{E}_{P_{\pi_{e}}}[R] &\approx \frac{1}{n}\sum_{i=1}^{n}\hat{R}(\pi_{e}(s_{i}),s_{i})\label{equ:outcome_approx}
\end{align}
Note that, the actions fed to the approximate reward function in equation \ref{equ:outcome_approx} are the actions under the evaluation policy whilst the actions fed to the densities in equation \ref{equ:propense_approx} are actions under the logging policy. For clarity, the logging policy actions are used to fit the approximate outcome model, $\hat{R}$, in equation \ref{equ:outcome_approx}.\\

\subsection{Why importance sampling?}
Within the AI community importance sampling estimators have been largely disregarded, for time-dependent OPE. There exists a plethora of research into importance sampling for the time independent setting however, the limitations of these will be discussed later. More generally, within the causal inference community, it is widely understood that robust evaluation of causal effects requires both propensity and outcome based methods: methods incorporating both outcome and propensity models are doubly robust (\cite{bang_doubly_2005}, \cite{wallace_doublyrobust_2015}); assessing the extent to which propensity and outcome measures agree can provide insight into bias or model misspecification (\cite{wallace_model_2017}) and; propensity based methods can be used as a diagnostic for bias (\cite{Rubin2001}).\\

\subsection{Types of importance sampling estimators}
The estimator in equation \ref{equ:propense_approx} defines the vanilla IS (VIS) estimator however, more often than not, the weighted IS (WIS) estimator is used:
\begin{equation}
	\hat{J}_{\textrm{WIS}} = \Bigg(\frac{1}{n}\sum_{i=1}^{n}w_{i}\Bigg)^{-1}\frac{1}{n}\sum_{i=1}^{n}r_{i}w_{i}.\label{equ:wis_estimator}
\end{equation}
The core challenge in the use of IS methods is the variance of the estimator. The use of the WIS estimator is motived through this variance reduction via the use of control variates. The denominator $\frac{1}{n}\sum_{i=1}^{n}w_{i}$ converges to $1$ in the limit of data ($n$) and thus the bias of the estimator vanishes, under positivity and correct identification of the propensity score model (\cite{Hernan2020}).\\

\subsection{Related work}\label{sec:related_work}
There exists a large amount of research into time independent OPE methods. The focus of the analysis presented is with respect to propensity based methods and as such existing work will focus on these estimators. 

\subsubsection{Monte Carlo importance sampling} Propensity based methods are largely motivated by the Monte Carlo importance sampling  (MCIS) literature. MCIS is concerned with evaluating an integral  with respect to a target measure, using an alternate base measure (\cite{mcbook}). This formulation precisely maps to equation \ref{equ:propense_approx} and thus much of the concepts are transferable. The core difference between OPE and MCIS is in the assumption of access to data. Many techniques for controlling the variance of IS based methods within the context of MCIS are motivated through selecting an appropriate base measure e.g., selecting a base measure with heavy tails to prevent importance weights from diverging to $\infty$. In the context of OPE, this would require selecting a heavy tailed logging density $\bar{p}_{\pi_{l}}$. However, in the context of OPE, data cannot be observed from this heavy tailed density since data are \emph{only} available from $p_{\pi_{l}}$.\\ 

\subsubsection{Importance sampling in OPE} Aside from the WIS estimator, the most prominent approach for variance reduction in importance sampling is clipping (\cite{alquier_logarithmic_2024}). \cite{alquier_logarithmic_2024} demonstrated that a number of other variance reduction methods can be viewed under the same theoretical framework, where clipping is a special case. Deriving probably approximately correct (PAC) estimators for conservative OPE measures i.e., asserting $P(\hat{J}-J<\epsilon)\geq 1-\delta$ has formed the other large body of research. \cite{alquier_logarithmic_2024} succeeded in this for the vanilla IS estimator however, the only existing PAC based analysis of the WIS estimator for the average treatment return in the untreated (equation \ref{equ:atrut}) is \cite{Kuzborskij2021}. The authors proposed a conservative PAC bound for off-policy selection, a simpler problem then OPE. A potential reason for the limited WIS PAC research is the strong dependence of random variables within the WIS estimator.\\

Despite the significant progress in understanding vanilla IS based estimators (and associated variance reducing alternatives), the estimator remains challenging to use for OPE due to it being unbounded. \cite{kallus_intrinsically_2019} highlight this and include $\alpha$-boundedness as a finite sample condition for a well behaved OPE estimator. The WIS estimator naturally induces boundedness however, the dependence structure means the resulting estimator is challenging to control.\\

\subsection{Contributions}
The primary contribution of this work is the Kernel-WIS estimator, a novel estimator which attempts to unify the benefits of both vanilla importance sampling (independence) and weighted importance sampling (boundedness and variance reduction) into a single estimator. Kernel-WIS is shown to be asymptotically consistent under mild conditions (appropriate bandwidth schedule, identification of the propensity model and positivity). The Kernel-WIS estimator is demonstrated to, in the majority of cases assessed: 
\begin{itemize}
	\item Statistically significantly outperform the WIS estimator under non-oracle behaviour policies and;
	\item Perform statistically identically to the WIS estimator under oracle behaviour policies.
\end{itemize}
However, under continuous reward settings the Kernel-WIS estimator statistically significantly underperforms in comparison to the WIS estimator.\\

\section{Methods}

The proceeding introduces the Kernel-WIS estimator for policy evaluation. The estimator is motivated by combining the best of the vanilla and weighted estimators, linearity/independence and boundedness, respectively. Before discussing Kernel-WIS, the State dependent-WIS (State-WIS) estimator is discussed. The State-WIS estimator is also novel to the analysis presented and can be derived from Kernel-WIS in the limit. The results in section \ref{sec:results} will demonstrate that the bias under the State-WIS is so large that it induces worse performance in comparison to baselines. However, considering State-WIS is instructive for understanding Kernel-WIS.\\

\subsection{State-dependent WIS estimator}\label{sec:state_dep_wis_est}
The State-WIS estimator, $\hat{J}_{\textrm{State-WIS}}$, is defined as:
\begin{equation}\label{equ:swis_estimator}
	\hat{J}_{\textrm{State-WIS}} = \frac{1}{n}\sum_{i=1}r_{i}w_{i}\Bigg(\frac{1}{n_{s_{i}}}\sum_{j =1}^{n}w_{j}\mathbbm{1}(s_{j}=s_{i})\Bigg)^{-1},
\end{equation}
where $n_{s_{i}} = \sum_{j=1}^{n}\mathbbm{1}(s_{j}=s_{i})$.\\

\sloppy 
No formal statement is made regarding the consistency of the State-WIS estimator. Intuitively though, the denominator $\frac{1}{n_{s_{i}}}\sum_{j =1}^{n}w_{j}\mathbbm{1}(s_{j}=s_{i})$ is clearly an estimator for the conditional expectation $\mathbb{E}[W|S=s_{i}]$ and, on its own is consistent, assuming positivity and correct identification of the propensity model. Thus the State-WIS estimator can be viewed as using conditional control variates since, $\mathbb{E}[W|S=s_{i}] = \mathbb{E}_{P_{\pi_{e}}}[1|S=s_{i}] = 1$. Additionally, observe that for $s_{i}\neq s_{j}$, the terms of State-WIS are independent.\\

\subsection{Kernel WIS estimator}
The Kernel-WIS estimator, $\hat{J}_{\textrm{Kernel-WIS}}$, is defined as:
\begin{equation}\label{equ:kwis_estimator}
	\hat{J}_{\textrm{Kernel-WIS}} = \frac{1}{n}\sum_{i=1}r_{i}w_{i}\Bigg(\frac{\sum_{j=1}^{n}k_{h}(s_{i},s_{j})w_{j}}{\sum_{j=1}^{n}k_{h}(s_{i},s_{j})}\Bigg)^{-1},
\end{equation}
where $k_{h}:\mathcal{S}\times \mathcal{S} \rightarrow \mathbb{R}_{\geq 0}$ is a kernel function under assumptions 1 to 8 (below).\\

The parameter $h$ is known as the bandwidth which controls the smoothness of the kernel function. The form of the estimator can be motivated from two points of view. Firstly, the denominator defines a kernel estimate of $\mathbb{E}[W|S=s_{i}]$ and thus Kernel-WIS with a given bandwidth could be viewed as an estimator of State-WIS. Alternatively, if $h=g(n,\lambda)$ is allowed to tend to 0 or $\infty$, it is hypothesised that bandwidth can be optimised as a hyperparameter (through $\lambda$) to control finite sample performance, such that WIS (with $h=\infty$) or State-WIS (with $h=0$) is attained in the limit. The consistency results (theorem \ref{ther:asympt_concis_bound_wsa}) are obtained with the limit tending towards State-WIS.\\ 


\subsection{Theoretical results}
The proceeding demonstrates the validity of the Kernel-WIS estimator in terms of asymptotic behaviour. To obtain these results, the Kernel-WIS estimator is viewed as an empirical process of random functions (\cite{vaart_asymptotic_1998}). Before proceeding, the assumptions of the analysis are stated. The implications and relevance of these assumptions are described in table \ref{tbl:assump_justi} (appendix section \ref{sec:ther_results_appendix}).\\

A kernel, $k_{h}(x)$ is any measurable function which satisfies:
\begin{enumerate}
	\item  $\|k_{h}(x)\|<\infty$;
	\item For some $C>0$ and $\nu>0$, $N(\epsilon,\mathcal{K})\leq C\epsilon^{-\nu}$ where $\mathcal{K} = \{k_{h}((x-\cdot)/h^{-\frac{1}{d}}):h>0,x\in \mathbb{R}^{d}\}$;
	\item $\mathcal{K}$ is pointwise measurable;
	\item The support of $k_{h}$ is $[-0.5,0.5]^{d}$;
	\item $\int k_{h}(x)d(x) = 1$ and are of the form: $k_{h}(x) = \frac{k_{h}'(x)}{\mathbb{E}[k_{h}'(x)]}$ where $\lim_{x\rightarrow \infty}k_{h}'(x) = 0$ and $k_{h}'(x)>0$ for finite $x$.
\end{enumerate}
Additional assumptions include:
\begin{enumerate}[resume]
	\item Assume $f$ defines the marginal density of $S$ and, $f$ is uniformly Lipschitz continuous and strictly positive on $I^{\epsilon}$, where $I$ is a compact set in $\mathbb{R}^{d}$ and $I^{\epsilon} = \{\max_{1\leq i \leq d}|s_{i}| \leq \epsilon\}$;
	\item There exists $w_{\textrm{max}}$ such that almost surely, $|W|\mathbbm{1}(S\in I^{\epsilon}) \leq w_{\textrm{max}}$ and $W\geq w_{\textrm{min}} >0$;
	\item $R \in [0,1]$;
	\item Correct identification of the propensity model i.e., $\hat{p}_{\pi_{l}} = p_{\pi_{l}}$.
\end{enumerate}

Throughout, $\|\cdot\|_{\mathcal{A}} = \sup_{\mathcal{A}}|\cdot|$. Many of the assumptions are provided by \cite{einmahl_uniform_2005}, since the results are used extensively. Theorem 2 from \cite{einmahl_uniform_2005} is heavily relied upon and is reproduced in theorem \ref{ther:theorem_2_einmahl}. Theorem \ref{ther:asympt_concis_bound_wsa} provides the novel result from this analysis with the proof provided in appendix section \ref{sec:ther_results_appendix}.\\

\begin{theorem}[Asymptotic consistency of Kernel-WIS]\label{ther:asympt_concis_bound_wsa} Under assumptions 1 to 9, for $0<a_{n}<h<b_{n}<1$, $b_{n} \rightarrow 0$ and $\frac{na_{n}}{\log n}\rightarrow\infty$, the Kernel-WIS estimator is asymptotically consistent:
\begin{align*}
	\hat{J}_{\textrm{Kernel-WIS}} \xrightarrow{a.s.} \mathbb{E}[R]
\end{align*}
where: 
\begin{align*}
	\hat{J}_{\textrm{Kernel-WIS}} = \frac{1}{n}\sum_{i=1}r_{i}w_{i}\Bigg(\frac{\sum_{j=1}^{n}k_{h}(s_{i},s_{j})w_{i}}{\sum_{j=1}^{n}k_{h}(s_{i},s_{j})}\Bigg)^{-1}.
\end{align*}
\end{theorem}

\subsection{Bandwidth selection}\label{sec:bandwidth_selection}
To perform bandwidth selection, the denominator term, $\frac{\sum_{j=1}^{n}k_{h}(s_{i},s_{j})w_{j}}{\sum_{j=1}^{n}k_{h}(s_{i},s_{j})}$, is treated as a Nadaraya–Watson kernel regressor of $\mathbb{E}[W|S]$ and cross-validation is performed. With respect to the overall goal of the Kernel-WIS estimator, this criteria is tangential and it is anticipated that optimising a PAC-Bayes lower bound would be more fruitful. Regardless, the proposed cross-validation approach acts as a first step.\\

Consider a dataset, $d$ of $n$ tuples of the form $(s_{i},a_{i},r_{i})$ and randomly split into two subsets $d_{\textrm{train}}$ of size $n_{\textrm{train}}$ and $d_{\textrm{test}}$ of size $n_{\textrm{test}}$. Dataset $d_{\textrm{train}}$ is used for training the behaviour policy $\pi_{\beta}$ as well as performing bandwidth selection via cross validation. Recall the denominator of the Kernel-WIS estimator aims to approximate the expectation, $\mathbb{E}[W|S]$. This can be used to define a supervised signal in the cross-validation procedure. The algorithm for performing this is defined in algorithm \ref{algo:bw_cross_validation} within which, algorithm \ref{algo:cross_val_fixed_b} is used as a sub-routine.\\

\begin{algorithm}
\caption{Bandwidth selection with $K$-fold cross validation}
\begin{algorithmic}[1]
	\Require Dataset $d_{\textrm{train}} = \{s_{i},w_{i}\}_{i=1}^{n_{\textrm{train}}}$, number of folds $K$, a set of bandwidth initialisations $h=\{h_{1},...,h_{m}\}$ and a minimiser $g$
	\State $S_{\textrm{best}} = 1e^{32}$, $h_{\textrm{best}} = 0$
	\For{$h_{i} \in h$}
		\State $\bar{S},h_{i}' \gets g(\textrm{Algorithm \ref{algo:cross_val_fixed_b}},b_{i},d_{\textrm{train}},K)$
	    \If{$\bar{S}<S_{\textrm{best}}$}
	    	\State $S_{\textrm{best}} \gets \bar{S}$
	    	\State $h_{\textrm{best}} \gets h_{i}'$
	    \EndIf
	\EndFor
\end{algorithmic}
\label{algo:bw_cross_validation}
\end{algorithm}

\begin{algorithm}
\caption{$K$-fold cross validation for a fixed bandwidth, $h$}
\begin{algorithmic}[1]
	\Require Dataset $d_{\textrm{train}} = \{s_{i},w_{i}\}_{i=1}^{n_{\textrm{train}}}$, number of folds $K$, a bandwidth $h$
	\State Randomly shuffle dataset $d_{\textrm{train}}$
	\State Split $d_{\textrm{train}}$ into $K$ disjoint folds $\{D_1, D_2, \dots, D_K\}$
	\State Initialize list of scores $S \gets [\,]$
	\For{$k = 1$ to $K$}
		\State Initialize list of scores $S_{k} \gets [\,]$
		\State $d_{\text{test}}' \gets d_k$
    	\State $d_{\text{train}}' \gets \bigcup_{j \ne k} d_j$
    	\For{$i=1$ to $n_{\textrm{test}}$}
	    	\State $\hat{w}_{i,\text{test}}=\frac{\sum_{j \in d_{\text{train}}}k_{b_{i}}(s_{i,\text{test}}-s_{j})w_{j}}{\sum_{j \in d_{\text{train}}}k_{b_{i}}(s_{i,\text{test}}-s_{j})}$
	    	\State Append $(\hat{w}_{i,\text{test}}-w_{i})^{2}$  to $S$
    	\EndFor
    	\State Append $\frac{1}{n_{\textrm{test}}}\sum_{i \in S_{k}} i$ to $S$
    \EndFor 
 \State \Return $\frac{1}{K}\sum_{i \in S} i$
\end{algorithmic}
\label{algo:cross_val_fixed_b}
\end{algorithm}

For the experiments defined in section \ref{sec:experiments}, the minimiser in algorithm $g$ was the L-BFGS-B algorithm, implemented in scipy via the minimise function (\cite{scipy2020}). An analytic expression for the gradient (described in appendix section \ref{sec:appendix_nw_kernel_grad}) was provided assuming the kernel function, $k_{h}$ was a radial basis function (RBF) kernel: 
\begin{align*}
	k_{h,\textrm{RBF}} = \exp(-0.5h^{-2}s^{T}s')
\end{align*}
where $s$ and $s'$ are vectors of dimension $l$ and $h_{i}$ is a real value ($\geq 0$) or a vector of real values ($\geq 0$).\\

Experiments were initially run by specifying only a single bandwidth initialisation i.e. $h=\{h_{0}\}$ in algorithm \ref{algo:bw_cross_validation}. It was observed however, that the optimisation by $g$ alone was insufficient as solutions rarely escaped the locality of the initial bandwidth $h_{0}$. As such, several bandwidth initialisations $h=\{h_{1},...,h_{m}\}$ were introduced. Introducing the sweep over different initialisations consistently produced improved improved solutions to algorithm \ref{algo:cross_val_fixed_b} with respect to the bandwidth i.e., improved approximations to $\mathbb{E}[W|S=s]$.\\

\section{Experiments}\label{sec:experiments}
The performance of the Kernel-WIS estimator was evaluated under 10 datasets. A semi-simulated setup was used, in the sense that, contextual bandit problem definitions were artificially created such that a ground truth evaluation of policy performance was available, but utilising existing \enquote{real world} datasets. This ensured as much of the simulation was as realistic as possible such as: noise to signal rations and covariate relationships. The proceeding describes the process of generating the contextual bandit setups (section \ref{sec:dataset_desc_gen}) and presents the results (section \ref{sec:results}).\\

\subsection{Dataset generation and description}\label{sec:dataset_desc_gen}
Let $d$ define a set of size $n$ of tuples $(s_{i},a_{i})_{i=1}^{n}$, representing an offline dataset. With respect to the original use case of $d$, all of the datasets considered for this analysis were focused on supervised learning. As such, in the context of supervised learning, the datasets defined tuples of independent random variables $S_{i} \in \mathcal{S}$ with an associated ground truth label, $A_{i} \in \mathcal{A}$. All of the datasets were taken from multi-class classification problems and thus, $\mathcal{A}$ defined a finite set of sequentially ordered integers, beginning from 0 $\{0,...,K-1\}$, of size $K$. As an example, the optdigits dataset includes $64\times 64$ pictures of hand drawn numbers with labels defining the digit in the drawing. As such, the set of actions is $\{0,1,2,3,4,5,6,7,8,9\}$.\\

Let $\pi:\mathcal{S}\rightarrow \mathcal{A}$ define a policy and define a reward function $R(a_{i},\pi(s_{i})) = \mathbbm{1}(a_{i}=\pi(s_{i}))$ and thus the supervised learning problem can be reformulated as a contextual bandit problem where the aim is to maximise $\frac{1}{n}\sum_{i=1}^{n}R(a_{i},\pi(s_{i}))$.\\

Given the supervised learning to contextual bandit setup, different policies were specified to assess the Kernel-WIS estimator under different simulated scenarios.\\

\subsubsection{Policy specification}\label{sec:policy_spec}
Policies were defined in two ways, in line with existing work (\cite{Kuzborskij2021}): as Gibbs measures over the dataset or; learnt by maximising an approximation of the reward. The Gibbs measure policies provided more fine grained control over the policy definition whilst, the learnt policies were deemed to better represent the structure of evaluation policies encountered in real-world applications.\\

\paragraph{Gibbs measure policies} Gibbs measure policies were defined by first specifying two hyperparameters: the temperature $\tau$ and the set of faulty actions. The temperature parameter controlled the concentration of the resulting policy around the greedy action i.e., the magnitude of the ratio $\frac{p_{\pi}(a_{\textrm{max}}|s)}{p_{\pi}(a'|s)}$ where $a_{\textrm{max}} = \argmax_{a}p_{\pi}(a|s)$ and $a'$ is all actions not equal to $a_{\textrm{max}}$. A lower value of $\tau$ corresponded to higher value of $\frac{p_{\pi}(a_{\textrm{max}}|s)}{p_{\pi}(a'|s)}$ i.e., a more concentrated policy or a policy that was more confident regarding the greedy action. The faulty actions, $\textrm{fa}$, defined the set of actions for which the policy, $\pi$ was suboptimal. Formally, the deterministic Gibbs measure policy was defined as:
\begin{equation}\label{equ:gibbs_action_slct}
	\pi_{\textrm{Gibbs}}(s_{i}) = a_{i}\mathbbm{1}(a_{i} \notin \textrm{fa}) + (a_{i}+1 \mod K)\mathbbm{1}(a_{i} \in \textrm{fa})
\end{equation} 
where $K$ defines the dimension of the action space. In words, the selected action is defined as equal to the optimal action in the dataset if the optimal action is not a member of the set of the faulty actions. Otherwise, the $a_{i}+1$ action is selected except for if $a_{i} = K-1$ in which case the greedy action is defined as the first action, 0. This is reason for the modulus function.\\

The density of the Gibbs measure policy, for a given action $a_{j}$ conditional on the state $s_{i}$ was defined as:
\begin{align*}
	p_{\pi_{\textrm{Gibbs}}}(a_{j}|s_{i}) = \frac{\exp(\mathbbm{1}(a_{j}=\pi_{\textrm{Gibbs}}(s_{i}))\tau^{-1})}{\sum_{k}\exp(\mathbbm{1}(a_{k}=\pi_{\textrm{Gibbs}}(s_{i}))\tau^{-1})}.
\end{align*}

The stochastic Gibbs measure policy was defined by randomly sampling an action according to $p_{\pi_{\textrm{Gibbs}}}(a_{j}|s_{i})$.\\

\paragraph{Learnt policies} 

Learnt policies were exclusively used for evaluation policies. These were derived by maximising (via gradient ascent for 20 epochs) the approximate expected return under the density defined in equation \ref{equ:learnt_policy_density}, with respect to a set of action wise parameters. The approximate expected return was derived using the vanilla importance sampling estimator or the weighted importance sampling estimator, for the IS and self-normalised importance weighted (SNIW) policies, respectively.
\begin{equation}\label{equ:learnt_policy_density}
	p_{\textrm{learnt}}(a_{j}|s_{i}) = \frac{\exp(\beta_{j}^{T}s_{i}\tau^{-1})}{\sum_{k}\exp(\beta_{k}^{T}s_{i}\tau^{-1})}
\end{equation}
where $\tau$ defines the data independent temperature, fixed a priori.\\

\subsubsection{Dataset description}

The datasets used for the analysis were downloaded from the OpenML repository (\cite{OpenML2025}). A selection of the datasets used in previous work pertaining to importance sampling estimators for contextual bandits (\cite{Kuzborskij2021}) were selected and are denoted by the Prev. Work column in table \ref{tbl:dataset_summary_stats}. In order to push further the evaluation, four additional datasets were included. These were deemed more complex owing to the larger action and state dimensions. Table \ref{tbl:dataset_summary_stats} provides an overview of the datasets used, including the state and action dimensions and number of observations.\\ 

\begin{table}[!h]
	\caption{The table displays summary statistics for the datasets used in the evaluation of the Kernel-WIS estimator. The values in brackets next to state dimensions of the soybean and kropt datasets define the dimensions after one-hot encoding categorical variables. The Prev. Work column defines whether the dataset was used in \cite{Kuzborskij2021}.}
	\label{tbl:dataset_summary_stats}
	\centering
		\begin{tabular}{cccccc}
		\toprule
		Dataset & \makecell{State\\dim.} & \makecell{Action\\dim.} & Size & Faulty actions & Prev. Work\\
		\midrule
		arrhythmia & 279 & 13 & 1,754 & \makecell{None; 0; 1; 0,1; 0,1,2;\\0,1,2,3;\\0,1,2,3,4; 0,1,2,3,5} & No \\
		\midrule
		soybean & 35 (132) & 19 & 683 & \makecell{None; 0; 1; 0,1; 0,1,2;\\0,1,2,4;\\0,1,2,4,6; 0,1,2,4,6,8} & No \\
		\midrule
		micro-mass & 1,300 & 20 & 571 & \makecell{None; 0; 1; 0,1; 0,1,2;\\0,1,2,4;\\ 0,1,2,4,6; 0,1,2,4,6,8} & No \\
		\midrule
		optdigitsDist & 64 & 10 & 5,620 & \makecell{None; 0; 1; 0,1; 0,1,2;\\0,1,2,3; 0,1,2,3,4} & No \\
		\midrule
		optdigits & 64 & 10 & 5,620 & \makecell{None; 0; 1; 0,1; 0,1,2;\\0,1,2,3; 0,1,2,3,4} & Yes \\
		\midrule
		yeast & 6 & 10 & 1,484 & \makecell{None; 0; 1; 0,1; 0,1,2;\\0,1,2,3; 0,1,2,3,4} & Yes \\
		\midrule
		page-blocks & 6 & 5 & 5,473 & \makecell{None; 0; 1; 0,1; 0,1,2} & Yes\\
		\midrule
		pendigits & 6 & 10 & 10,992 & \makecell{None; 0; 1; 0,1; 0,1,2;\\0,1,2,3; 0,1,2,3,4} & Yes\\
		\midrule
		letter & 6 & 26 & 20,000 & \makecell{None; 0; 1; 0,1; 0,1,2;\\0,1,2,4;\\0,1,2,4,6; 0,1,2,4,6,8} & Yes\\
		\midrule
		kropt & 6 (40) & 18 & 28,056 & \makecell{None; 0; 1; 0,1; 0,1,2;\\0,1,2,4;\\0,1,2,4,6; 0,1,2,4,6,8} & Yes\\
		\bottomrule
		\end{tabular}
\end{table}

\paragraph{Logging and behaviour policy definitions}\label{sec:logging_behaviour_policy_defs} The column Faulty actions in table \ref{tbl:dataset_summary_stats} defines the different faulty actions considered for the logging policy on each dataset. For example, under the arrhythmia dataset, an experiment was ran where the logging policy was optimal ($\textrm{fa} = \{\}$), and where the logging policy made errors on states with correct action equal to 0 i.e., $\textrm{fa} = \{0\}$ etc. Due to the computational complexity of the experiments (predominantly a result of the bandwidth selection for Kernel-WIS), only a subset of faulty action settings could be used i.e., the entire Cartesian product of potential experiments could not be considered. Since previous works (\cite{Kuzborskij2021}) considered \emph{only} almost optimal policies (i.e., only considering the faulty actions $\textrm{fa} = \{0\}$), this was used as a starting point but greatly expanded. For reference, running the proposed suite of experiments for the larger datasets/state-spaces took up to 24 hours using a 2021 16 inch MacBook Pro with 32GB ram and an Apple M1 Pro CPU. The logic to define the set of faulty actions per dataset was defined broadly speaking, as:
\begin{itemize}
	\item For datasets with action dimension 15 and under, faulty actions up to and including length 5 were considered (e.g., 0,1,2,3,5 for arrhythmia);
	\item For datasets with larger action dimensions, a faulty action of length 6 was additionally included.
\end{itemize}
Thus, for each dataset, \enquote{$\textrm{faulty action sets}\times\{0.3,0.5\}$} different logging policies were defined.\\

\paragraph{Dataset generation and reward noise} Under a given experimental setting i.e., logging policy temperature and faulty actions, evaluation policy temperature, dataset and behaviour policy specification, datasets were generated by randomly sampling the logging policy for each state in the dataset and inducing reward noise by the following process:
\begin{align*}
	\mathbbm{1}(U > 0.1)R(s_{i},\pi_{l}(s_{i})) + (1-\mathbbm{1}(U > 0.1))(1-R(s_{i},\pi_{l}(s_{i})))
\end{align*}
where $U$ is a uniform random variable between 0 and 1 and $R(\cdot,\cdot)$ is the deterministic reward function described at the beginning of section \ref{sec:dataset_desc_gen}. In words, $\mathbbm{1}(U > 0.1)$ is a Bernoulli random variable with probability 0.9. And thus, with probability 0.1, the reward is flipped.\\

\FloatBarrier

\subsection{Results}\label{sec:results}
The proceeding describes the results of a number of different experimental settings evaluating the pointwise predictive performance and tail behaviour of the Kernel-WIS estimator. To assess the relative performance of the Kernel-WIS estimator against the WIS estimator, a two-sided Wald t-test (with unequal variances) was used. Note, the t-tests were not conducted over repeated sampling of the datasets but over different experimental scenarios, defined in section \ref{sec:dataset_desc_gen}. Where $\textrm{AD(J)}$ is the absolute difference in performance under estimator $J$ (equation \ref{equ:ad_metric}), for the Kernel-WIS estimator the t-test assesses:
\begin{align*}
	H_{0}:& \mathbb{E}[\textrm{AD(WIS)}] = \mathbb{E}[\textrm{AD(Kernel-WIS)}]\\
	H_{1}:& \mathbb{E}[\textrm{AD(WIS)}] \neq \mathbb{E}[\textrm{AD(Kernel-WIS)}]
\end{align*}
As such, a p-value of less than or equal to 0.05 suggests that the performance of the WIS estimator is not equal to the Kernel-WIS estimator, at the 95\% level. The absolute difference in performance under estimator $J$ is defined as:
\begin{equation}\label{equ:ad_metric}
	\textrm{AD(J)} = \Bigg|\frac{1}{n}\sum_{i=1}^{n}\mathbb{E}[R(\pi_{e}(s_{i}),s_{i})|s_{i}] - J(\{s_{i},\pi_{\beta}(s_{i}),R(\pi_{\beta}(s_{i}),s_{i})\}_{i=1}^{n})\Bigg|.
\end{equation}

For the WIS and Kernel-WIS estimators, the coverage of boostrapped confidence intervals was assessed. Coverage was defined as an approximation to the true Frequentist statement:
\begin{align*}
	P\big(\mathbb{E}[R(\pi_{e}(S),S)] \in C\big) \approx \frac{\mathbbm{1}(\mathbb{E}[R(\pi_{e}(s_{i}),s_{i})|s_{i}] \in C)}{b},
\end{align*}
where $C$ defines the bootsrapped interval and $b$ defines the number of independent test sets used to approximate the probability, which was set to 10.\\

For all results reported, the best performing estimator is highlighted in bold:
\begin{itemize}
	\item For pointwise MSE results (tables \ref{tbl:overall_results}, \ref{tbl:k_wis_wis_oracle_results}, \ref{tbl:k_wis_wis_non_oracle_results} and \ref{tbl:k_wis_wis_cont_results}) the estimator(s) with the lowest value along with other results within 0.001 of the lowest value are highlighted;
	\item For coverage results (tables \ref{tbl:k_wis_wis_oracle_width} and \ref{tbl:k_wis_wis_non_oracle_width}) the estimator(s) with the highest value along with other results within 0.001 of the highest value are highlighted;
	\item For P-values (tables \ref{tbl:k_wis_wis_oracle_results} and \ref{tbl:k_wis_wis_non_oracle_results}) those values less than or equal to 0.05 are highlighted
\end{itemize}
A tolerance of 0.001 was selected since this represents 0.1\% of the theoretical maximum reward.\\

\subsubsection{Natural bias of the Kernel-WIS variants}\label{sec:kernel_wis_natural_bias}
The Kernel-WIS variants naturally bias towards predicting the observed return of the dataset. Consider for a large state-space (i.e., continuous), the probability of observing any identical state is 0. As such, as the bandwidth tends towards 0:
\begin{align*}
	\sum_{j=1}^{n}k_{h}(s_{i},s_{j}) \rightarrow & 1 \\
	\sum_{j=1}^{n}k_{h}(s_{i},s_{j})w_{j} \rightarrow & w_{i} \\
	\hat{J}_{\textrm{Kernel-WIS}} \rightarrow & \frac{1}{n}\sum_{i=1}^{n}r_{i}\frac{w_{i}}{w_{i}}.
\end{align*}
This is the price being paid for independence between observations but retaining empirical boundedness. Whilst not unreasonable, the observation does alter the interpretation of results i.e., it is both instructive and important with respect to bias (of the interpretation of results), to consider the performance of the Kernel-WIS variants at different divergences in evaluation/logging policy performance.\\

\subsubsection{Overall results}
Table \ref{tbl:overall_results} describes the overall performance of Kernel-WIS estimator against a number of baselines, split by dataset and experimental setting (reward structure and behaviour policy type). The \enquote{Single action} reward structure aligns to that described in section \ref{sec:dataset_desc_gen} and the definition of \enquote{Continuous} reward structure is described in the proceeding section \ref{sec:cont_reward_sensitivity}. The \enquote{Oracle} policy setting assumes that the propensity weights are generated under the same logging policy as that which created the dataset whilst the \enquote{Non-Oracle} setting requires the policy generating the weights to be different i.e., miss-specified. This is described in more detail in the proceeding section \ref{sec:non_oracle_beh_sensitivity}.\\

Based on the results in table \ref{tbl:overall_results}, the relative performance of the Kernel-WIS estimator is mixed. However, this will be discussed in further detail since the results in table \ref{tbl:overall_results} are aggregated over a huge array of different scenarios which warrant a more detailed examination. The inference that should be drawn from table \ref{tbl:overall_results} is the relative performance of the Kernel-WIS estimator against the VIS and truncated vanilla importance sampling (herein referred to as CLPD VIS (\cite{Ionides2008})) estimators. In a number of scenarios the VIS and CLPD VIS estimators are optimal however, it is widely understood that in many real-world scenarios, these estimators exhibit exceedingly large variance. Whilst this does not nullify the results, it does suggest that more challenging benchmarks are required. Table \ref{tbl:overall_results_mm} presents the mean 75th and 100th percentile prediction under each of the estimators. These results support the empirical $\alpha$-boundedness property of the Kernel-WIS estimator but also demonstrate that for scenarios were the vanilla estimator does have large variance (all of the non-oracle settings), the Kernel-WIS estimator performs optimally.\\ 

The results in table \ref{tbl:overall_results} additionally demonstrate the consistent sub-optimality of the State-WIS estimator. This is important to note given the observation in section \ref{sec:kernel_wis_natural_bias} i.e., that the Kernel-WIS estimator is naturally biased towards the performance of the logging policy. The sub-optimality of the State-WIS estimator demonstrates that the experimental scenarios used are sufficiently complex such that a naive and strict application of this bias (i.e., just use the logging policy return) is insufficient to obtain good performance and thus the bandwidth selection procedure is truly improving the resulting estimate.\\

\begin{table}[!h]
	\caption{The table describes the pointwise finite mean squared error of the Kernel-WIS estimator (with a shared bandwidth) against various baselines, split by dataset, whether an oracle behaviour policy was used and the reward function structure. Aggregations for each row are across logging policy temperature, logging policy faulty action and policy type. Values in bold highlight the best performing estimator and those within 0.001 of the best performing estimator.}
	\label{tbl:overall_results}
	\centering
	\resizebox{\linewidth}{!}{
	\begin{tabular}{cccccccc}
	\toprule
	\makecell{Reward\\structure} & \makecell{Behaviour\\policy} & Dataset & VIS & WIS & State-WIS & CLPD VIS & Kernel-WIS \\
	\midrule
	Single action & Oracle & yeast & 0.045 & 0.046 & 0.140 & \textbf{0.043} & \textbf{0.043} \\
	Single action & Oracle & soybean & 0.068 & 0.063 & 0.166 & \textbf{0.056} & 0.062 \\
	Single action & Oracle & page-blocks & 0.070 & \textbf{0.064} & 0.332 & 0.070 & \textbf{0.064} \\
	Single action & Oracle & pendigits & \textbf{0.038} & \textbf{0.037} & 0.201 & \textbf{0.038} & \textbf{0.038} \\
	Single action & Oracle & micro-mass & 0.061 & 0.061 & 0.156 & \textbf{0.055} & 0.062 \\
	Single action & Oracle & arrhythmia & 0.074 & 0.059 & 0.115 & 0.060 & \textbf{0.054} \\
	Single action & Oracle & optdigits & 0.041 & \textbf{0.040} & 0.199 & \textbf{0.041} & \textbf{0.040} \\
	Single action & Oracle & kropt & \textbf{0.038} & \textbf{0.037} & 0.146 & \textbf{0.038} & \textbf{0.033} \\
	Single action & Oracle & letter & \textbf{0.033} & \textbf{0.032} & 0.140 & \textbf{0.033} & \textbf{0.031} \\
	\midrule
	Single action & Non-Oracle & yeast & 0.413 & 0.142 & 0.148 & 0.384 & \textbf{0.117} \\
	Single action & Non-Oracle & soybean & 0.308 & 0.150 & 0.159 & 0.251 & \textbf{0.117} \\
	Single action & Non-Oracle & page-blocks & 1.638 & 0.168 & 0.312 & 1.638 & \textbf{0.157} \\
	Single action & Non-Oracle & pendigits & 0.443 & 0.134 & 0.172 & 0.443 & \textbf{0.114} \\
	Single action & Non-Oracle & arrhythmia & 0.515 & 0.139 & 0.119 & 0.367 & \textbf{0.105} \\
	Single action & Non-Oracle & optdigits & 0.432 & \textbf{0.137} & 0.174 & 0.432 & \textbf{0.138} \\
	Single action & Non-Oracle & micro-mass & 0.291 & 0.150 & 0.152 & 0.237 & \textbf{0.136} \\
	\midrule
	Continuous & Oracle & optdigitsDist & \textbf{0.056} & \textbf{0.057} & 0.110 & \textbf{0.056} & 0.065 \\
	\bottomrule
	\end{tabular}
	}
\end{table}

\begin{table}[!h]
	\caption{The table describes the 75th and 100th percentile predictions of the Kernel-WIS estimator (with a shared bandwidth) against various baselines, split by dataset, whether an oracle behaviour policy was used and the reward function structure. Aggregations for each row are across logging policy temperature, logging policy faulty action and policy type.}
	\label{tbl:overall_results_mm}
	\centering
	\resizebox{\linewidth}{!}{
		\begin{tabular}{ccccccccccccc}
		\toprule
		\makecell{Reward\\structure} & \makecell{Behaviour\\policy} & Dataset & \multicolumn{2}{c}{VIS} & \multicolumn{2}{c}{WIS} & \multicolumn{2}{c}{State-WIS} & \multicolumn{2}{c}{CLPD VIS} & \multicolumn{2}{c}{Kernel-WIS} \\
		& & & 75 & 100 & 75 & 100 & 75 & 100 & 75 & 100 & 75 & 100 \\
		\midrule
		Single action & Oracle & yeast & 0.52 & 0.73 & 0.51 & 0.72 & 0.53 & 0.70 & 0.52 & 0.73 & 0.51 & 0.72 \\
		Single action & Oracle & soybean & 0.49 & 0.59 & 0.49 & 0.59 & 0.53 & 0.59 & 0.49 & 0.59 & 0.47 & 0.59 \\
		Single action & Oracle & page-blocks & 0.73 & 0.81 & 0.75 & 0.81 & 0.62 & 0.80 & 0.73 & 0.81 & 0.75 & 0.81 \\
		Single action & Oracle & pendigits & 0.73 & 0.81 & 0.75 & 0.81 & 0.58 & 0.71 & 0.73 & 0.81 & 0.74 & 0.81 \\
		Single action & Oracle & micro-mass & 0.48 & 0.61 & 0.49 & 0.59 & 0.53 & 0.58 & 0.48 & 0.60 & 0.47 & 0.58 \\
		Single action & Oracle & arrhythmia & 0.36 & 0.70 & 0.34 & 0.68 & 0.35 & 0.68 & 0.31 & 0.70 & 0.32 & 0.68 \\
		Single action & Oracle & optdigits & 0.71 & 0.85 & 0.71 & 0.84 & 0.59 & 0.71 & 0.71 & 0.85 & 0.71 & 0.84 \\
		Single action & Oracle & kropt & 0.46 & 0.60 & 0.47 & 0.60 & 0.52 & 0.60 & 0.46 & 0.60 & 0.46 & 0.60 \\
		Single action & Oracle & letter & 0.48 & 0.52 & 0.48 & 0.52 & 0.48 & 0.53 & 0.48 & 0.52 & 0.47 & 0.52 \\
		\midrule
		Single action & Non-Oracle & yeast & 0.80 & 6.41 & 0.64 & 0.87 & 0.53 & 0.70 & 0.79 & 6.24 & 0.57 & 0.83 \\
		Single action & Non-Oracle & soybean & 0.82 & 2.85 & 0.57 & 0.85 & 0.54 & 0.60 & 0.68 & 2.00 & 0.51 & 0.80 \\
		Single action & Non-Oracle & page-blocks & 1.66 & 20.67 & 0.77 & 0.90 & 0.62 & 0.80 & 1.66 & 20.67 & 0.68 & 0.88 \\
		Single action & Non-Oracle & pendigits & 1.15 & 3.54 & 0.74 & 0.89 & 0.58 & 0.71 & 1.15 & 3.54 & 0.65 & 0.86 \\
		Single action & Non-Oracle & arrhythmia & 0.62 & 11.23 & 0.43 & 0.92 & 0.35 & 0.68 & 0.57 & 5.74 & 0.35 & 0.89 \\
		Single action & Non-Oracle & optdigits & 1.13 & 3.43 & 0.73 & 0.89 & 0.59 & 0.71 & 1.13 & 3.43 & 0.62 & 0.89 \\
		Single action & Non-Oracle & micro-mass & 0.79 & 2.40 & 0.56 & 0.83 & 0.53 & 0.58 & 0.67 & 1.84 & 0.52 & 0.78 \\
		\midrule
	Continuous & Oracle & optdigitsDist & 0.63 & 0.76 & 0.63 & 0.75 & 0.64 & 0.75 & 0.63 & 0.76 & 0.60 & 0.75 \\
		\bottomrule
		\end{tabular}
	}
\end{table}

\FloatBarrier

\subsubsection{Oracle behaviour policy results}\label{sec:oracle_beh_policy_results}
The results in table \ref{tbl:k_wis_wis_oracle_results} describe the performance of the Kernel-WIS estimator against the WIS estimator for oracle behaviour policies. Given the bias described in \ref{sec:kernel_wis_natural_bias}, the results are split by dataset, policy type and the deviation in performance under the evaluation and logging policies.\\

From the perspective of statistical significance, the results in table \ref{tbl:k_wis_wis_oracle_results} demonstrate that the WIS and Kernel-WIS estimators perform similarly. On only 11 occasions does the performance of the two estimators differ significantly and on all but one of these occasions the Kernel-WIS estimator outperforms the WIS estimator. Ignoring statistical significance, in terms of the number of scenarios (i.e., rows in the table) where one estimator outperforms the other, the WIS estimator is more optimal. However, in terms of the median performance across datasets (i.e., aggregating the rows for each dataset), the Kernel-WIS estimator is more optimal. These results are shown more clearly in figure \ref{fig:summary_performance_oracle}. Given the results, it is reasonable to conclude that the WIS estimator only marginally outperforms the Kernel-WIS estimator however, the Kernel-WIS estimator can drastically outperform the WIS estimator.\\

With respect to the proportion of instances that the Kernel-WIS estimator strictly outperforms the WIS estimator (figures \ref{fig:kwis_performance_no_buffer_by_n_samples_oracle} and \ref{fig:kwis_performance_no_buffer_by_action_dim_oracle}), the plots suggest either a: non-linear relationship to sample size or a (noisy) linear relationship to action dimension. It was unclear why these relationships emerged or whether the relative performance of Kernel-WIS and WIS was driven by some other latent factor. This is explored further in section \ref{sec:learnt_logging_sensitivity}.\\ 

\FloatBarrier

{\small\tabcolsep=3pt  
	\begin{longtable}{cccccccc}
	\caption{The table describes the pointwise finite mean squared error of the WIS estimator and Kernel-WIS estimator (with a shared bandwidth), split by dataset, policy type and a discretisation of the absolute difference in true performance under the evaluation and logging policy, for oracle behaviour policies. The P-Value column defines the p-value of a two-sided Wald t-test, comparing the finite mean squared error of the WIS and Kernel-WIS estimators. N observations defines the number of experiments aggregated over per row, where the aggregations are across logging policy temperature and logging policy faulty action. Values in bold in the WIS and Kernel-WIS columns highlight the best performing estimator. Neither are highlighted if the mean performance is within 0.001. Values highlighted in the P-Value column identify those which are less than or equal to 0.05.}
	\label{tbl:k_wis_wis_oracle_results} \\
	\toprule
	Dataset & Policy Type & \makecell{Eval/Log\\Abs. Diff.} & WIS & Kernel-WIS & N Obs. & P-Value\\
	\midrule
	\endfirsthead
	\toprule
	Dataset & Policy Type & \makecell{Eval/Log\\Abs. Diff.} & WIS & Kernel-WIS & N Obs. & P-Value\\
	\midrule
	\endhead            
	\midrule\multicolumn{7}{r}{{Continued on next page}} \\ 
	\endfoot     
	\endlastfoot
	arrhythmia & Gibbs & (-0.2, 0.0] & 0.037$\pm$0.016 & 0.037$\pm$0.016 & 16 & 1.00 \\
	arrhythmia & Gibbs & (0.0, 0.2] & 0.046$\pm$0.014 & \textbf{0.042$\pm$0.018} & 114 & 0.06 \\
	arrhythmia & Gibbs & (0.2, 0.4] & \textbf{0.055$\pm$0.020} & 0.062$\pm$0.017 & 18 & 0.27 \\
	arrhythmia & Gibbs & (0.4, 0.6] & \textbf{0.064$\pm$0.012} & 0.080$\pm$0.036 & 12 & 0.15 \\
	arrhythmia & IW & (0.0, 0.2] & 0.087$\pm$0.012 & \textbf{0.069$\pm$0.028} & 28 & \textbf{0.00} \\
	arrhythmia & IW & (0.2, 0.4] & 0.086$\pm$0.009 & \textbf{0.079$\pm$0.004} & 4 & 0.17 \\
	arrhythmia & SNIW & (0.0, 0.2] & 0.093$\pm$0.006 & \textbf{0.075$\pm$0.011} & 24 & \textbf{0.00} \\
	arrhythmia & SNIW & (0.2, 0.4] & 0.073$\pm$0.003 & 0.074$\pm$0.002 & 4 & 0.50 \\
	arrhythmia & SNIW & (0.4, 0.6] & 0.111$\pm$0.001 & 0.111$\pm$0.001 & 2 & 1.00 \\
	arrhythmia & SNIW & (0.6, 0.8] & 0.092$\pm$0.001 & 0.092$\pm$0.001 & 2 & 1.00 \\
	\midrule
	kropt & Gibbs & (-0.2, 0.0] & 0.032$\pm$0.019 & 0.032$\pm$0.019 & 16 & 1.00 \\
	kropt & Gibbs & (0.0, 0.2] & 0.033$\pm$0.018 & 0.033$\pm$0.017 & 80 & 0.89 \\
	kropt & Gibbs & (0.2, 0.4] & 0.029$\pm$0.017 & 0.030$\pm$0.017 & 56 & 0.74 \\
	kropt & IW & (0.0, 0.2] & 0.047$\pm$0.008 & \textbf{0.046$\pm$0.006} & 20 & 0.50 \\
	kropt & IW & (0.2, 0.4] & 0.049$\pm$0.003 & \textbf{0.030$\pm$0.011} & 12 & \textbf{0.00} \\
	kropt & SNIW & (0.0, 0.2] & 0.049$\pm$0.006 & \textbf{0.026$\pm$0.017} & 20 & \textbf{0.00} \\
	kropt & SNIW & (0.2, 0.4] & 0.053$\pm$0.002 & \textbf{0.049$\pm$0.004} & 10 & \textbf{0.00} \\
	kropt & SNIW & (0.4, 0.6] & 0.053$\pm$0.000 & \textbf{0.046$\pm$0.000} & 2 & \textbf{0.00} \\
	\midrule
	letter & Gibbs & (-0.2, 0.0] & 0.034$\pm$0.026 & 0.034$\pm$0.026 & 16 & 1.00 \\
	letter & Gibbs & (0.0, 0.2] & 0.033$\pm$0.025 & 0.033$\pm$0.025 & 60 & 0.97 \\
	letter & Gibbs & (0.2, 0.4] & 0.034$\pm$0.024 & 0.034$\pm$0.024 & 76 & 0.99 \\
	letter & IW & (0.0, 0.2] & 0.024$\pm$0.004 & 0.023$\pm$0.004 & 16 & 0.41 \\
	letter & IW & (0.2, 0.4] & 0.022$\pm$0.003 & 0.022$\pm$0.003 & 16 & 0.89 \\
	letter & SNIW & (0.0, 0.2] & 0.032$\pm$0.007 & \textbf{0.028$\pm$0.011} & 24 & 0.12 \\
	letter & SNIW & (0.2, 0.4] & 0.039$\pm$0.014 & 0.038$\pm$0.015 & 8 & 0.95 \\
	\midrule
	micro-mass & Gibbs & (-0.2, 0.0] & 0.045$\pm$0.020 & 0.045$\pm$0.020 & 16 & 1.00 \\
	micro-mass & Gibbs & (0.0, 0.2] & 0.052$\pm$0.015 & 0.051$\pm$0.015 & 68 & 0.68 \\
	micro-mass & Gibbs & (0.2, 0.4] & \textbf{0.050$\pm$0.019} & 0.057$\pm$0.018 & 68 & \textbf{0.04} \\
	micro-mass & IW & (0.0, 0.2] & 0.073$\pm$0.024 & \textbf{0.071$\pm$0.026} & 32 & 0.84 \\
	micro-mass & SNIW & (0.0, 0.2] & 0.098$\pm$0.004 & 0.097$\pm$0.004 & 14 & 0.86 \\
	micro-mass & SNIW & (0.2, 0.4] & 0.095$\pm$0.007 & 0.095$\pm$0.007 & 10 & 0.99 \\
	micro-mass & SNIW & (0.4, 0.6] & 0.100$\pm$0.007 & 0.100$\pm$0.007 & 8 & 1.00 \\
	\midrule
	optdigits & Gibbs & (-0.2, 0.0] & 0.027$\pm$0.013 & 0.027$\pm$0.013 & 14 & 1.00 \\
	optdigits & Gibbs & (0.0, 0.2] & 0.026$\pm$0.011 & 0.026$\pm$0.011 & 74 & 0.97 \\
	optdigits & Gibbs & (0.2, 0.4] & 0.027$\pm$0.011 & 0.027$\pm$0.011 & 44 & 0.99 \\
	optdigits & IW & (0.0, 0.2] & 0.076$\pm$0.011 & 0.077$\pm$0.010 & 5 & 0.94 \\
	optdigits & IW & (0.2, 0.4] & 0.068$\pm$0.006 & 0.068$\pm$0.005 & 9 & 0.92 \\
	optdigits & IW & (0.4, 0.6] & 0.075$\pm$0.004 & 0.075$\pm$0.004 & 12 & 0.90 \\
	optdigits & IW & (0.6, 0.8] & \textbf{0.075$\pm$0.000} & 0.081$\pm$0.002 & 2 & 0.10 \\
	optdigits & SNIW & (0.0, 0.2] & \textbf{0.066$\pm$0.007} & 0.069$\pm$0.005 & 8 & 0.42 \\
	optdigits & SNIW & (0.2, 0.4] & \textbf{0.076$\pm$0.006} & 0.080$\pm$0.006 & 5 & 0.32 \\
	optdigits & SNIW & (0.4, 0.6] & 0.071$\pm$0.006 & 0.071$\pm$0.006 & 15 & 0.94 \\
	\midrule
	page-blocks & Gibbs & (-0.2, 0.0] & 0.067$\pm$0.024 & 0.067$\pm$0.024 & 10 & 1.00 \\
	page-blocks & Gibbs & (0.0, 0.2] & 0.075$\pm$0.017 & 0.074$\pm$0.017 & 36 & 0.92 \\
	page-blocks & Gibbs & (0.2, 0.4] & 0.047$\pm$0.022 & 0.048$\pm$0.023 & 6 & 0.96 \\
	page-blocks & Gibbs & (0.4, 0.6] & 0.055$\pm$0.026 & 0.055$\pm$0.026 & 20 & 0.98 \\
	page-blocks & Gibbs & (0.6, 0.8] & 0.076$\pm$0.011 & 0.077$\pm$0.010 & 18 & 0.9 \\
	page-blocks & Gibbs & (0.8, 1.0] & \textbf{0.080$\pm$0.014} & 0.085$\pm$0.007 & 2 & 0.72 \\
	page-blocks & IW & (0.0, 0.2] & 0.048$\pm$0.013 & 0.048$\pm$0.014 & 6 & 0.98 \\
	page-blocks & IW & (0.2, 0.4] & 0.075$\pm$0.003 & 0.075$\pm$0.003 & 2 & 1.00 \\
	page-blocks & IW & (0.4, 0.6] & 0.049$\pm$nan & 0.049$\pm$nan & 1 & nan \\
	page-blocks & IW & (0.6, 0.8] & \textbf{0.066$\pm$0.012} & 0.068$\pm$0.014 & 10 & 0.63 \\
	page-blocks & IW & (0.8, 1.0] & 0.077$\pm$nan & 0.078$\pm$nan & 1 & nan \\
	page-blocks & SNIW & (0.0, 0.2] & 0.036$\pm$0.005 & 0.037$\pm$0.006 & 6 & 0.89 \\
	page-blocks & SNIW & (0.2, 0.4] & 0.031$\pm$0.002 & \textbf{0.023$\pm$0.001} & 2 & \textbf{0.05} \\
	page-blocks & SNIW & (0.4, 0.6] & \textbf{0.044$\pm$0.002} & 0.056$\pm$0.014 & 3 & 0.27 \\
	page-blocks & SNIW & (0.6, 0.8] & 0.054$\pm$0.014 & \textbf{0.052$\pm$0.013} & 9 & 0.78 \\
	\midrule
	pendigits & Gibbs & (-0.2, 0.0] & 0.028$\pm$0.014 & 0.028$\pm$0.014 & 14 & 1.00 \\
	pendigits & Gibbs & (0.0, 0.2] & 0.026$\pm$0.012 & 0.025$\pm$0.012 & 78 & 0.85 \\
	pendigits & Gibbs & (0.2, 0.4] & 0.027$\pm$0.012 & 0.027$\pm$0.012 & 40 & 0.98 \\
	pendigits & IW & (0.0, 0.2] & 0.071$\pm$0.003 & 0.071$\pm$0.003 & 6 & 0.81 \\
	pendigits & IW & (0.2, 0.4] & \textbf{0.062$\pm$0.004} & 0.065$\pm$0.002 & 6 & 0.10 \\
	pendigits & IW & (0.4, 0.6] & \textbf{0.065$\pm$0.005} & 0.067$\pm$0.004 & 14 & 0.28 \\
	pendigits & IW & (0.6, 0.8] & 0.062$\pm$0.002 & 0.062$\pm$0.002 & 2 & 1.00 \\
	pendigits & SNIW & (0.0, 0.2] & \textbf{0.050$\pm$0.005} & 0.057$\pm$0.009 & 8 & 0.10 \\
	pendigits & SNIW & (0.2, 0.4] & \textbf{0.045$\pm$0.005} & 0.068$\pm$0.022 & 5 & 0.08 \\
	pendigits & SNIW & (0.4, 0.6] & \textbf{0.064$\pm$0.003} & 0.069$\pm$0.009 & 14 & 0.06 \\
	pendigits & SNIW & (0.6, 0.8] & 0.064$\pm$nan & 0.064$\pm$nan & 1 & nan \\
	\midrule
	soybean & Gibbs & (-0.2, 0.0] & 0.044$\pm$0.020 & 0.044$\pm$0.020 & 16 & 1.00 \\
	soybean & Gibbs & (0.0, 0.2] & 0.049$\pm$0.016 & 0.048$\pm$0.019 & 68 & 0.72 \\
	soybean & Gibbs & (0.2, 0.4] & 0.050$\pm$0.019 & 0.051$\pm$0.019 & 68 & 0.77 \\
	soybean & IW & (0.0, 0.2] & 0.074$\pm$0.008 & \textbf{0.072$\pm$0.008} & 26 & 0.50 \\
	soybean & IW & (0.2, 0.4] & 0.072$\pm$0.008 & \textbf{0.068$\pm$0.007} & 6 & 0.35 \\
	soybean & SNIW & (0.0, 0.2] & 0.089$\pm$0.000 & \textbf{0.082$\pm$0.000} & 16 & \textbf{0.00} \\
	soybean & SNIW & (0.2, 0.4] & 0.152$\pm$0.001 & 0.152$\pm$0.001 & 4 & 1.00 \\
	soybean & SNIW & (0.4, 0.6] & 0.149$\pm$0.001 & 0.149$\pm$0.001 & 12 & 1.00 \\
	\midrule
	yeast & Gibbs & (-0.2, 0.0] & 0.035$\pm$0.019 & 0.035$\pm$0.019 & 14 & 1.00 \\
	yeast & Gibbs & (0.0, 0.2] & 0.035$\pm$0.016 & 0.034$\pm$0.016 & 80 & 0.82 \\
	yeast & Gibbs & (0.2, 0.4] & 0.038$\pm$0.015 & 0.039$\pm$0.015 & 28 & 0.89 \\
	yeast & Gibbs & (0.4, 0.6] & 0.048$\pm$0.006 & 0.047$\pm$0.006 & 10 & 0.92 \\
	yeast & IW & (0.0, 0.2] & 0.060$\pm$0.014 & \textbf{0.054$\pm$0.010} & 24 & 0.09 \\
	yeast & IW & (0.2, 0.4] & 0.071$\pm$0.006 & 0.071$\pm$0.006 & 4 & 1.00 \\
	yeast & SNIW & (0.0, 0.2] & 0.071$\pm$0.010 & \textbf{0.063$\pm$0.009} & 16 & \textbf{0.03} \\
	yeast & SNIW & (0.2, 0.4] & 0.084$\pm$0.004 & \textbf{0.064$\pm$0.004} & 8 & \textbf{0.00} \\
	yeast & SNIW & (0.4, 0.6] & 0.059$\pm$0.003 & 0.059$\pm$0.003 & 4 & 1.00 \\
		\bottomrule
	\end{longtable}
}

\FloatBarrier

\begin{figure}
	\centering
	\begin{subfigure}{0.49\linewidth}
		\centering
		\includegraphics[width=\linewidth]{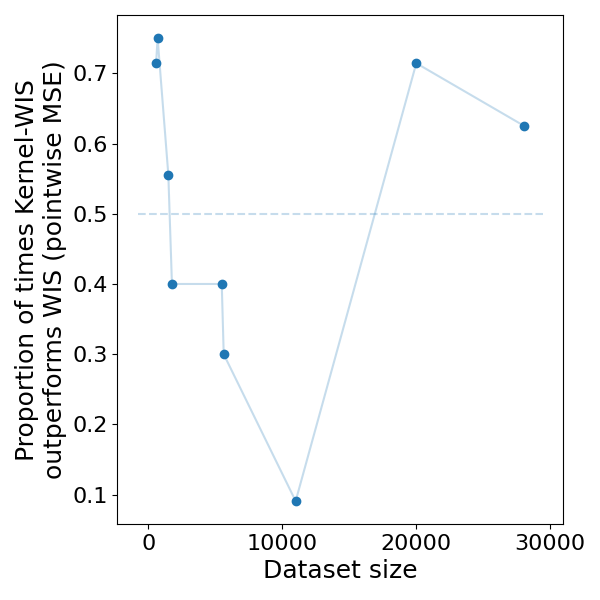}
		\caption{Proportion}
		\label{fig:kwis_performance_no_buffer_by_n_samples_oracle}
	\end{subfigure}
	\begin{subfigure}{0.49\linewidth}
		\centering
		\includegraphics[width=\linewidth]{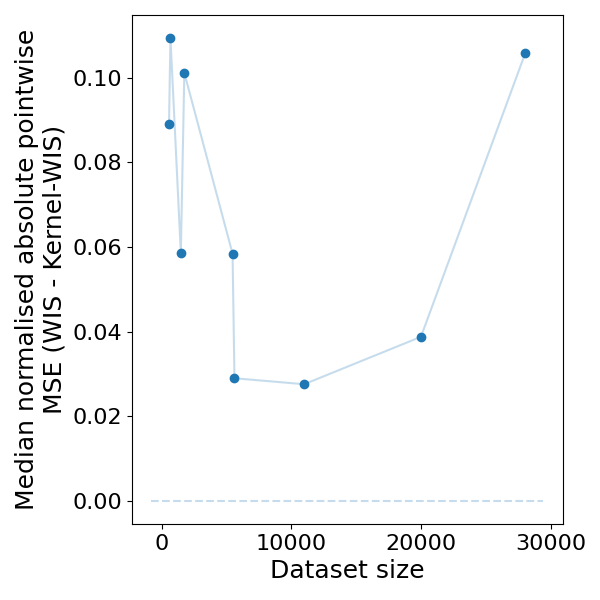}
		\caption{Median normalised difference}
		\label{fig:norm_performance_by_n_samples_oracle}
	\end{subfigure}
	\begin{subfigure}{0.49\linewidth}
		\centering
		\includegraphics[width=\linewidth]{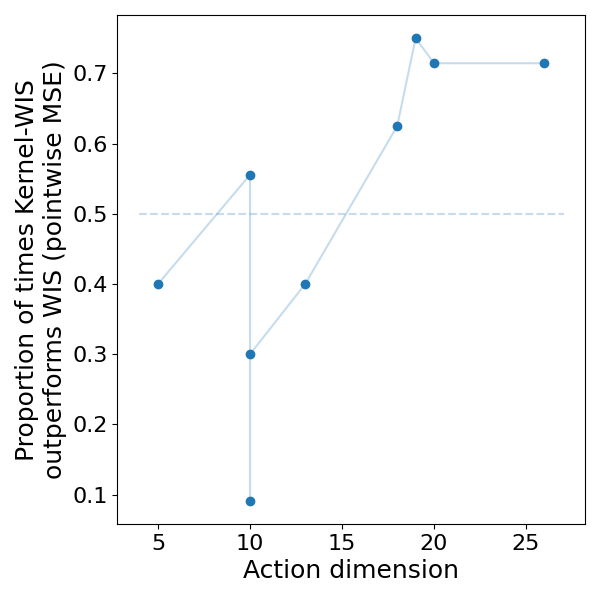}
		\caption{Proportion}
		\label{fig:kwis_performance_no_buffer_by_action_dim_oracle}
	\end{subfigure}
	\begin{subfigure}{0.49\linewidth}
		\centering
		\includegraphics[width=\linewidth]{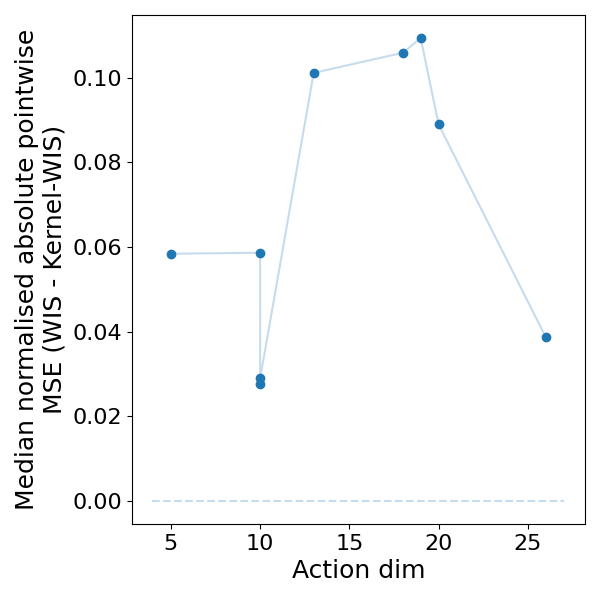}
		\caption{Median normalised difference}
		\label{fig:norm_performance_by_action_dim_oracle}
	\end{subfigure}
	\caption{The figures describe two views on the difference in pointwise mean squared error for the Kernel-WIS and WIS estimators against dataset size (top) and action dimension (bottom) on the x-axis. Data are shown for the single action reward, oracle behaviour policy setting. The left plots describe the proportion of scenarios (i.e., unique combinations of policy type and buckets differences in true evaluation and logging policy return) where the Kernel-WIS estimator strictly outperforms the WIS estimator. The right plots describe the median normalised difference in pointwise mean squared error ($\textrm{AD}(\textrm{WIS})-\textrm{AD}(\textrm{Kernel-WIS}/(\textrm{AD}(\textrm{WIS})+\textrm{AD}(\textrm{Kernel-WIS}))$).}
	\label{fig:summary_performance_oracle}
\end{figure}

Table \ref{tbl:k_wis_wis_oracle_width} again displays the relative performance of the WIS and Kernel-WIS estimators but includes the coverage of the bootstrapped confidence intervals and a summary of these results is displayed in figures \ref{fig:both_cov_prop_oracle} and \ref{fig:cov_diff_oracle}. Figure \ref{fig:both_cov_prop_oracle} describes the proportion of times that the coverage under each estimator is strictly larger than the other. Arguably, the two estimators perform similarly except for under the pendigits dataset where the coverage under the WIS estimator is noticeably improved. The results in figure \ref{fig:both_cov_prop_oracle} demonstrate however, that the magnitude of the difference in coverage values is small, especially when considered in the context of the theoretical range of the reward. The reward under each dataset is between 0 and 1 meaning that the coverage of the two estimators differs by at most 6\% of this range.\\

\FloatBarrier

{\small\tabcolsep=3pt  
	\begin{longtable}{*{8}{>{\footnotesize}c}}
	\caption{The table describes the pointwise mean squared error and the width of the bootstrapped confidence intervals of the WIS estimator and Kernel-WIS estimator (with a shared bandwidth), split by dataset, policy type and a discretisation of the absolute difference in true performance under the evaluation and logging policy, for oracle behaviour policies. Aggregations for each row are across logging policy temperature and logging policy faulty action. Values in bold in the WIS and Kernel-WIS columns highlight the best performing estimator. Neither are highlighted if the mean performance is within 0.001.}
	\label{tbl:k_wis_wis_oracle_width}\\
	\toprule
	& & & \multicolumn{2}{c}{Estimate} & \multicolumn{2}{c}{Coverage} \\ 
	Dataset & Policy Type & \makecell{Eval/Log\\Abs. Diff.} & WIS & Kernel-WIS & WIS & Kernel-WIS \\
	\midrule
	\endfirsthead            	
	\toprule
	& & & \multicolumn{2}{c}{Estimate} & \multicolumn{2}{c}{Coverage} \\ 
	Dataset & Policy Type & \makecell{Eval/Log\\Abs. Diff.} & WIS & Kernel-WIS & WIS & Kernel-WIS \\
	\midrule
	\endhead            
	\midrule\multicolumn{7}{r}{{Continued on next page}} \\ 
	\endfoot     
	\endlastfoot 
	arrhythmia & Gibbs & (-0.2, 0.0] & 0.037$\pm$0.016 & 0.037$\pm$0.016 & 0.618$\pm$0.356 & 0.618$\pm$0.356 \\
	arrhythmia & Gibbs & (0.0, 0.2] & 0.046$\pm$0.014 & \textbf{0.042$\pm$0.018} & \textbf{0.642$\pm$0.362} & 0.628$\pm$0.393 \\
	arrhythmia & Gibbs & (0.2, 0.4] & \textbf{0.055$\pm$0.020} & 0.062$\pm$0.017 & \textbf{0.802$\pm$0.234} & 0.772$\pm$0.230 \\
	arrhythmia & Gibbs & (0.4, 0.6] & \textbf{0.064$\pm$0.012} & 0.080$\pm$0.036 & \textbf{0.667$\pm$0.361} & 0.648$\pm$0.348 \\
	arrhythmia & IW & (0.0, 0.2] & 0.087$\pm$0.012 & \textbf{0.069$\pm$0.028} & 0.853$\pm$0.068 & \textbf{0.877$\pm$0.087} \\
	arrhythmia & IW & (0.2, 0.4] & 0.086$\pm$0.009 & \textbf{0.079$\pm$0.004} & 0.500$\pm$0.111 & \textbf{0.583$\pm$0.140} \\
	arrhythmia & SNIW & (0.0, 0.2] & 0.093$\pm$0.006 & \textbf{0.075$\pm$0.011} & 0.870$\pm$0.071 & \textbf{0.898$\pm$0.098} \\
	arrhythmia & SNIW & (0.2, 0.4] & 0.073$\pm$0.003 & 0.074$\pm$0.002 & 0.444$\pm$0.000 & 0.444$\pm$0.000 \\
	arrhythmia & SNIW & (0.4, 0.6] & 0.111$\pm$0.001 & 0.111$\pm$0.001 & 0.667$\pm$0.000 & 0.667$\pm$0.000 \\
	arrhythmia & SNIW & (0.6, 0.8] & 0.092$\pm$0.001 & 0.092$\pm$0.001 & 0.444$\pm$0.000 & 0.444$\pm$0.000 \\
	\midrule
	kropt & Gibbs & (-0.2, 0.0] & 0.032$\pm$0.019 & 0.032$\pm$0.019 & 0.125$\pm$0.284 & 0.125$\pm$0.284 \\
	kropt & Gibbs & (0.0, 0.2] & 0.033$\pm$0.018 & 0.033$\pm$0.017 & \textbf{0.210$\pm$0.350} & 0.167$\pm$0.315 \\
	kropt & Gibbs & (0.2, 0.4] & 0.029$\pm$0.017 & 0.030$\pm$0.017 & \textbf{0.288$\pm$0.391} & 0.254$\pm$0.373 \\
	kropt & IW & (0.0, 0.2] & 0.047$\pm$0.008 & 0.046$\pm$0.006 & 0.100$\pm$0.080 & 0.100$\pm$0.080 \\
	kropt & IW & (0.2, 0.4] & 0.049$\pm$0.003 & \textbf{0.030$\pm$0.011} & 0.037$\pm$0.055 & \textbf{0.204$\pm$0.200} \\
	kropt & SNIW & (0.0, 0.2] & 0.049$\pm$0.006 & \textbf{0.026$\pm$0.017} & 0.061$\pm$0.057 & \textbf{0.478$\pm$0.237} \\
	kropt & SNIW & (0.2, 0.4] & 0.053$\pm$0.002 & \textbf{0.049$\pm$0.004} & 0.000$\pm$0.000 & 0.000$\pm$0.000 \\
	kropt & SNIW & (0.4, 0.6] & 0.053$\pm$0.000 & \textbf{0.046$\pm$0.000} & 0.000$\pm$0.000 & 0.000$\pm$0.000 \\
	\midrule
	letter & Gibbs & (-0.2, 0.0] & 0.034$\pm$0.026 & 0.034$\pm$0.026 & 0.271$\pm$0.349 & 0.271$\pm$0.349 \\
	letter & Gibbs & (0.0, 0.2] & 0.033$\pm$0.025 & 0.033$\pm$0.025 & 0.359$\pm$0.409 & 0.359$\pm$0.409 \\
	letter & Gibbs & (0.2, 0.4] & 0.034$\pm$0.024 & 0.034$\pm$0.024 & 0.396$\pm$0.416 & 0.395$\pm$0.415 \\
	letter & IW & (0.0, 0.2] & 0.024$\pm$0.004 & 0.023$\pm$0.004 & 0.708$\pm$0.151 & 0.708$\pm$0.151 \\
	letter & IW & (0.2, 0.4] & 0.022$\pm$0.003 & 0.022$\pm$0.003 & \textbf{0.715$\pm$0.107} & 0.701$\pm$0.113 \\
	letter & SNIW & (0.0, 0.2] & 0.032$\pm$0.007 & \textbf{0.028$\pm$0.011} & 0.528$\pm$0.119 & \textbf{0.634$\pm$0.201} \\
	letter & SNIW & (0.2, 0.4] & 0.039$\pm$0.014 & 0.038$\pm$0.015 & 0.389$\pm$0.341 & \textbf{0.403$\pm$0.366} \\
	\midrule
	micro-mass & Gibbs & (-0.2, 0.0] & 0.045$\pm$0.020 & 0.045$\pm$0.020 & 0.528$\pm$0.375 & 0.528$\pm$0.375 \\
	micro-mass & Gibbs & (0.0, 0.2] & 0.052$\pm$0.015 & 0.051$\pm$0.015 & 0.520$\pm$0.316 & \textbf{0.526$\pm$0.316} \\
	micro-mass & Gibbs & (0.2, 0.4] & \textbf{0.050$\pm$0.019} & 0.057$\pm$0.018 & \textbf{0.554$\pm$0.426} & 0.516$\pm$0.397 \\
	micro-mass & IW & (0.0, 0.2] & 0.073$\pm$0.024 & \textbf{0.071$\pm$0.026} & 0.708$\pm$0.242 & \textbf{0.719$\pm$0.251} \\
	micro-mass & SNIW & (0.0, 0.2] & 0.098$\pm$0.004 & 0.097$\pm$0.004 & 0.476$\pm$0.166 & \textbf{0.484$\pm$0.161} \\
	micro-mass & SNIW & (0.2, 0.4] & 0.095$\pm$0.007 & 0.095$\pm$0.007 & 0.344$\pm$0.152 & 0.344$\pm$0.152 \\
	micro-mass & SNIW & (0.4, 0.6] & 0.100$\pm$0.007 & 0.100$\pm$0.007 & 0.167$\pm$0.103 & 0.167$\pm$0.103 \\
	\midrule
	optdigits & Gibbs & (-0.2, 0.0] & 0.027$\pm$0.013 & 0.027$\pm$0.013 & 0.238$\pm$0.337 & 0.238$\pm$0.337 \\
	optdigits & Gibbs & (0.0, 0.2] & 0.026$\pm$0.011 & 0.026$\pm$0.011 & 0.562$\pm$0.337 & \textbf{0.569$\pm$0.344} \\
	optdigits & Gibbs & (0.2, 0.4] & 0.027$\pm$0.011 & 0.027$\pm$0.011 & 0.472$\pm$0.399 & \textbf{0.477$\pm$0.398} \\
	optdigits & IW & (0.0, 0.2] & 0.076$\pm$0.011 & 0.077$\pm$0.010 & 0.178$\pm$0.186 & 0.178$\pm$0.186 \\
	optdigits & IW & (0.2, 0.4] & 0.068$\pm$0.006 & 0.068$\pm$0.005 & 0.605$\pm$0.273 & \textbf{0.617$\pm$0.243} \\
	optdigits & IW & (0.4, 0.6] & 0.075$\pm$0.004 & 0.075$\pm$0.004 & 0.231$\pm$0.146 & \textbf{0.241$\pm$0.133} \\
	optdigits & IW & (0.6, 0.8] & \textbf{0.075$\pm$0.000} & 0.081$\pm$0.002 & \textbf{0.556$\pm$0.000} & 0.389$\pm$0.079 \\
	optdigits & SNIW & (0.0, 0.2] & \textbf{0.066$\pm$0.007} & 0.069$\pm$0.005 & \textbf{0.514$\pm$0.258} & 0.444$\pm$0.279 \\
	optdigits & SNIW & (0.2, 0.4] & \textbf{0.076$\pm$0.006} & 0.080$\pm$0.006 & \textbf{0.756$\pm$0.050} & 0.733$\pm$0.061 \\
	optdigits & SNIW & (0.4, 0.6] & 0.071$\pm$0.006 & 0.071$\pm$0.006 & 0.393$\pm$0.258 & 0.393$\pm$0.258 \\
	\midrule
	page-blocks & Gibbs & (-0.2, 0.0] & 0.067$\pm$0.024 & 0.067$\pm$0.024 & 0.056$\pm$0.120 & 0.056$\pm$0.120 \\
	page-blocks & Gibbs & (0.0, 0.2] & 0.075$\pm$0.017 & 0.074$\pm$0.017 & \textbf{0.043$\pm$0.146} & 0.040$\pm$0.146 \\
	page-blocks & Gibbs & (0.2, 0.4] & 0.047$\pm$0.022 & 0.048$\pm$0.023 & 0.019$\pm$0.045 & 0.019$\pm$0.045 \\
	page-blocks & Gibbs & (0.4, 0.6] & 0.055$\pm$0.026 & 0.055$\pm$0.026 & 0.500$\pm$0.513 & 0.500$\pm$0.513 \\
	page-blocks & Gibbs & (0.6, 0.8] & 0.076$\pm$0.011 & 0.077$\pm$0.010 & 0.432$\pm$0.459 & 0.432$\pm$0.459 \\
	page-blocks & Gibbs & (0.8, 1.0] & \textbf{0.080$\pm$0.014} & 0.085$\pm$0.007 & 0.500$\pm$0.707 & 0.500$\pm$0.707 \\
	page-blocks & IW & (0.0, 0.2] & 0.048$\pm$0.013 & 0.048$\pm$0.014 & \textbf{0.315$\pm$0.333} & 0.296$\pm$0.327 \\
	page-blocks & IW & (0.2, 0.4] & 0.075$\pm$0.003 & 0.075$\pm$0.003 & 0.000$\pm$0.000 & 0.000$\pm$0.000 \\
	page-blocks & IW & (0.4, 0.6] & 0.049$\pm$nan & 0.049$\pm$nan & 0.889$\pm$nan & 0.889$\pm$nan \\
	page-blocks & IW & (0.6, 0.8] & \textbf{0.066$\pm$0.012} & 0.068$\pm$0.014 & 0.900$\pm$0.161 & 0.900$\pm$0.161 \\
	page-blocks & IW & (0.8, 1.0] & 0.077$\pm$nan & 0.078$\pm$nan & 1.000$\pm$nan & 1.000$\pm$nan \\
	page-blocks & SNIW & (0.0, 0.2] & 0.036$\pm$0.005 & 0.037$\pm$0.006 & \textbf{0.481$\pm$0.335} & 0.444$\pm$0.344 \\
	page-blocks & SNIW & (0.2, 0.4] & 0.031$\pm$0.002 & \textbf{0.023$\pm$0.001} & \textbf{0.667$\pm$0.000} & 0.556$\pm$0.000 \\
	page-blocks & SNIW & (0.4, 0.6] & \textbf{0.044$\pm$0.002} & 0.056$\pm$0.014 & 0.963$\pm$0.064 & 0.963$\pm$0.064 \\
	page-blocks & SNIW & (0.6, 0.8] & 0.054$\pm$0.014 & \textbf{0.052$\pm$0.013} & 0.901$\pm$0.103 & \textbf{0.926$\pm$0.111} \\
	\midrule
	pendigits & Gibbs & (-0.2, 0.0] & 0.028$\pm$0.014 & 0.028$\pm$0.014 & 0.143$\pm$0.297 & 0.143$\pm$0.297 \\
	pendigits & Gibbs & (0.0, 0.2] & 0.026$\pm$0.012 & 0.025$\pm$0.012 & 0.410$\pm$0.337 & \textbf{0.422$\pm$0.353} \\
	pendigits & Gibbs & (0.2, 0.4] & 0.027$\pm$0.012 & 0.027$\pm$0.012 & \textbf{0.328$\pm$0.345} & 0.319$\pm$0.348 \\
	pendigits & IW & (0.0, 0.2] & 0.071$\pm$0.003 & 0.071$\pm$0.003 & \textbf{0.037$\pm$0.057} & 0.000$\pm$0.000 \\
	pendigits & IW & (0.2, 0.4] & \textbf{0.062$\pm$0.004} & 0.065$\pm$0.002 & \textbf{0.574$\pm$0.204} & 0.500$\pm$0.153 \\
	pendigits & IW & (0.4, 0.6] & \textbf{0.065$\pm$0.005} & 0.067$\pm$0.004 & 0.190$\pm$0.284 & 0.190$\pm$0.284 \\
	pendigits & IW & (0.6, 0.8] & 0.062$\pm$0.002 & 0.062$\pm$0.002 & 0.556$\pm$0.000 & 0.556$\pm$0.000 \\
	pendigits & SNIW & (0.0, 0.2] & \textbf{0.050$\pm$0.005} & 0.057$\pm$0.009 & \textbf{0.458$\pm$0.288} & 0.319$\pm$0.338 \\
	pendigits & SNIW & (0.2, 0.4] & \textbf{0.045$\pm$0.005} & 0.068$\pm$0.022 & \textbf{0.911$\pm$0.093} & 0.600$\pm$0.268 \\
	pendigits & SNIW & (0.4, 0.6] & \textbf{0.064$\pm$0.003} & 0.069$\pm$0.009 & 0.190$\pm$0.252 & 0.190$\pm$0.232 \\
	pendigits & SNIW & (0.6, 0.8] & 0.064$\pm$nan & 0.064$\pm$nan & 0.444$\pm$nan & 0.444$\pm$nan \\
	\midrule
	soybean & Gibbs & (-0.2, 0.0] & 0.044$\pm$0.020 & 0.044$\pm$0.020 & 0.535$\pm$0.370 & 0.535$\pm$0.370 \\
	soybean & Gibbs & (0.0, 0.2] & 0.049$\pm$0.016 & 0.048$\pm$0.019 & \textbf{0.482$\pm$0.343} & 0.453$\pm$0.333 \\
	soybean & Gibbs & (0.2, 0.4] & 0.050$\pm$0.019 & 0.051$\pm$0.019 & \textbf{0.531$\pm$0.445} & 0.525$\pm$0.443 \\
	soybean & IW & (0.0, 0.2] & 0.074$\pm$0.008 & \textbf{0.072$\pm$0.008} & 0.667$\pm$0.109 & \textbf{0.671$\pm$0.115} \\
	soybean & IW & (0.2, 0.4] & 0.072$\pm$0.008 & \textbf{0.068$\pm$0.007} & 0.519$\pm$0.152 & \textbf{0.593$\pm$0.057} \\
	soybean & SNIW & (0.0, 0.2] & 0.089$\pm$0.000 & \textbf{0.082$\pm$0.000} & 0.500$\pm$0.057 & \textbf{0.556$\pm$0.000} \\
	soybean & SNIW & (0.2, 0.4] & 0.152$\pm$0.001 & 0.152$\pm$0.001 & 0.778$\pm$0.000 & 0.778$\pm$0.000 \\
	soybean & SNIW & (0.4, 0.6] & 0.149$\pm$0.001 & 0.149$\pm$0.001 & 0.500$\pm$0.058 & 0.500$\pm$0.058 \\
	\midrule
	yeast & Gibbs & (-0.2, 0.0] & 0.035$\pm$0.019 & 0.035$\pm$0.019 & 0.468$\pm$0.383 & 0.468$\pm$0.383 \\
	yeast & Gibbs & (0.0, 0.2] & 0.035$\pm$0.016 & \textbf{0.034$\pm$0.016} & 0.579$\pm$0.401 & \textbf{0.581$\pm$0.393} \\
	yeast & Gibbs & (0.2, 0.4] & \textbf{0.038$\pm$0.015} & 0.039$\pm$0.015 & \textbf{0.647$\pm$0.319} & 0.623$\pm$0.330 \\
	yeast & Gibbs & (0.4, 0.6] & 0.048$\pm$0.006 & 0.047$\pm$0.006 & 0.611$\pm$0.410 & \textbf{0.633$\pm$0.389} \\
	yeast & IW & (0.0, 0.2] & 0.060$\pm$0.014 & \textbf{0.054$\pm$0.010} & 0.819$\pm$0.097 & \textbf{0.852$\pm$0.102} \\
	yeast & IW & (0.2, 0.4] & 0.071$\pm$0.006 & 0.071$\pm$0.006 & 0.444$\pm$0.000 & 0.444$\pm$0.000 \\
	yeast & SNIW & (0.0, 0.2] & 0.071$\pm$0.010 & \textbf{0.063$\pm$0.009} & 0.521$\pm$0.113 & \textbf{0.583$\pm$0.111} \\
	yeast & SNIW & (0.2, 0.4] & 0.084$\pm$0.004 & \textbf{0.064$\pm$0.004} & \textbf{0.500$\pm$0.059} & 0.444$\pm$0.059 \\
	yeast & SNIW & (0.4, 0.6] & 0.059$\pm$0.003 & 0.059$\pm$0.003 & 0.556$\pm$0.000 & 0.556$\pm$0.000 \\
		\bottomrule
	\end{longtable}
}

\FloatBarrier

\begin{figure}[!h]
	\centering
	\begin{subfigure}{0.49\linewidth}
		\centering
		\includegraphics[width=\linewidth]{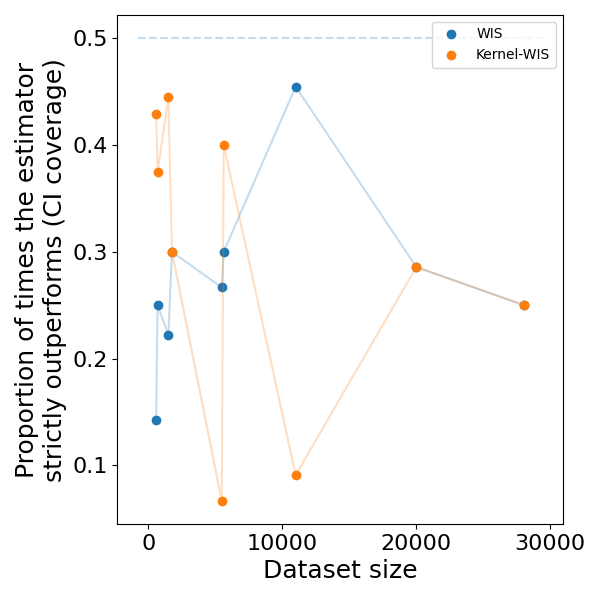}
		\caption{Proportion}
		\label{fig:both_cov_prop_oracle}
	\end{subfigure}
	\begin{subfigure}{0.49\linewidth}
		\centering
		\includegraphics[width=\linewidth]{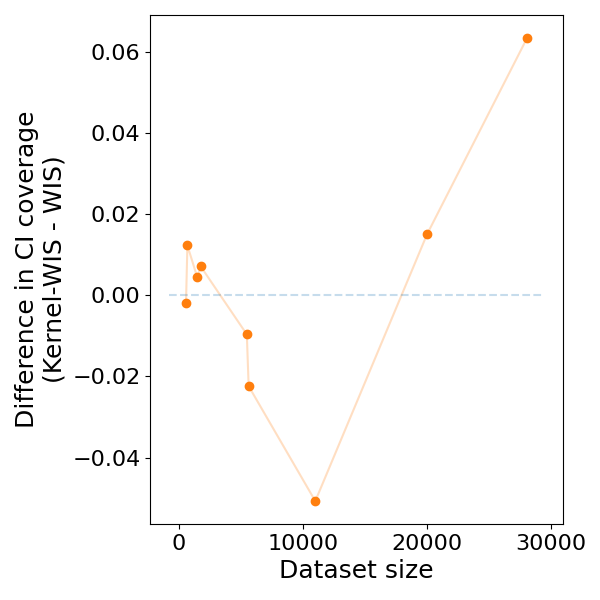}
		\caption{Difference}
		\label{fig:cov_diff_oracle}
	\end{subfigure}
	\caption{The figures describe two views of the difference in bootstrapped interval coverage results for the Kernel-WIS and WIS estimators against dataset size on the x-axis. Data are shown for the single action reward, oracle behaviour policy setting. The left plot describes the proportion of times the coverage under each estimator is strictly larger than the other. The right plot describes the difference in coverage values.}
\end{figure}

\FloatBarrier

\subsubsection{Non-oracle behaviour policy sensitivity}\label{sec:non_oracle_beh_sensitivity}

A sensitivity analysis of the relative performance of the WIS and Kernel-WIS estimators was performed, under miss-specification of the behaviour policy. The sensitivity analysis was performed for a subset of datasets due to computational constraints. For each logging policy, the set of miss-specified behaviour policies evaluated was based on the faulty actions of the logging policy. Specifically, the set of miss-specified behaviour policies included those that were within length 1 of the logging policy faulty actions. For example, under the arrythmia dataset, for logging policy with temperature 0.3 and $\textrm{fa} = 1$, the following behaviour policies were considered: 
\begin{itemize}
	\item Temperature = 0.5, $\textrm{fa} = \{1\}$;
	\item Temperature = 0.3, $\textrm{fa} = \{\}$;
	\item Temperature = 0.5, $\textrm{fa} = \{\}$;
	\item Temperature = 0.3, $\textrm{fa} = \{0\}$;
	\item Temperature = 0.5, $\textrm{fa} = \{0\}$;
	\item Temperature = 0.3, $\textrm{fa} = \{0,1\}$;
	\item Temperature = 0.5, $\textrm{fa} = \{0,1\}$.
\end{itemize}
Similarly, under the optdigits dataset, for logging policy with temperature equal to 0.5 and $\textrm{fa} = \{0,1,2,3,4\}$, the following behaviour policies were considered:
\begin{itemize}
	\item Temperature = 0.3, $\textrm{fa} = \{0,1,2,3,4\}$;
	\item Temperature = 0.3, $\textrm{fa} = \{0,1,2,3\}$;
	\item Temperature = 0.5, $\textrm{fa} = \{0,1,2,3\}$.
\end{itemize}

The relative performance of the Kernel-WIS estimator clearly improved, both in terms of the number of significant results and when not considering statistical significance. Similarly to the oracle behaviour setting (section \ref{sec:oracle_beh_policy_results}), figure \ref{fig:summary_performance_non_oracle} displays aggregated results, across the scenarios displayed in table \ref{tbl:k_wis_wis_non_oracle_results} (i.e., combinations of policy type and bucketed values of the difference between the true evaluation and logging policy performance). Examining the proportion of times the Kernel-WIS estimator strictly outperforms the WIS estimator (figure \ref{fig:kwis_performance_no_buffer_by_n_samples_non_oracle}), the results remain mixed however, examining the median performance (figure \ref{fig:norm_performance_by_n_samples_non_oracle}) suggests the Kernel-WIS estimator strongly outperforms the WIS estimator. Similarly to the oracle setting: 
\begin{itemize}
	\item The Kernel-WIS estimator drastically outperforms the WIS estimator whilst the WIS estimator only marginally outperforms the Kernel-WIS estimator and;
	\item The relative performance of the estimators is potentially dependent on sample size and/or action dimension (note that the two largest datasets: kropt and letter were excluded from the non-oracle analysis due to computational constraints).\\
\end{itemize}

\FloatBarrier

{\small\tabcolsep=3pt  
	\begin{longtable}{cccccccc}
	\caption{The table describes the pointwise finite mean squared error of the WIS estimator and Kernel-WIS estimator (with a shared bandwidth), split by dataset, policy type and a discretisation of the absolute difference in true performance under the evaluation and logging policy, for non-oracle behaviour policies. The P-Value column defines the p-value of a two-sided Wald t-test, comparing the finite mean squared error of the WIS and Kernel-WIS estimators. N observations defines the number of experiments aggregated over per row, where the aggregations are across logging policy temperature and logging policy faulty action. Values in bold in the WIS and Kernel-WIS columns highlight the best performing estimator. Neither are highlighted if the mean performance is within 0.001. Values highlighted in the P-Value column identify those which are less than or equal to 0.05.}
	\label{tbl:k_wis_wis_non_oracle_results} \\
	\toprule
	Dataset & Policy Type & \makecell{Eval/Log\\Abs. Diff.} & WIS & Kernel-WIS & N Obs. & P-Value\\
	\midrule
	\endfirsthead            
	\toprule
	Dataset & Policy Type & \makecell{Eval/Log\\Abs. Diff.} & WIS & Kernel-WIS & N Obs. & P-Value\\
	\midrule
	\endhead            
	\midrule\multicolumn{7}{r}{{Continued on next page}} \\ 
	\endfoot     
	\endlastfoot
	arrhythmia & Gibbs & (-0.2, 0.0] & 0.133$\pm$0.101 & \textbf{0.098$\pm$0.093} & 96 & \textbf{0.01} \\
	arrhythmia & Gibbs & (0.0, 0.2] & 0.097$\pm$0.074 & \textbf{0.074$\pm$0.056} & 684 & \textbf{0.00} \\
	arrhythmia & Gibbs & (0.2, 0.4] & 0.218$\pm$0.148 & \textbf{0.201$\pm$0.122} & 112 & 0.35 \\
	arrhythmia & Gibbs & (0.4, 0.6] & \textbf{0.221$\pm$0.150} & 0.249$\pm$0.139 & 76 & 0.23 \\
	arrhythmia & IW & (0.0, 0.2] & 0.155$\pm$0.115 & \textbf{0.095$\pm$0.084} & 181 & \textbf{0.00} \\
	arrhythmia & IW & (0.2, 0.4] & 0.284$\pm$0.117 & \textbf{0.221$\pm$0.109} & 11 & 0.21 \\
	arrhythmia & SNIW & (0.0, 0.2] & 0.157$\pm$0.110 & \textbf{0.092$\pm$0.050} & 153 & \textbf{0.00} \\
	arrhythmia & SNIW & (0.2, 0.4] & 0.082$\pm$0.045 & \textbf{0.078$\pm$0.046} & 15 & 0.81 \\
	arrhythmia & SNIW & (0.4, 0.6] & 0.455$\pm$0.210 & \textbf{0.257$\pm$0.121} & 24 & \textbf{0.00} \\
	\midrule
	micro-mass & Gibbs & (-0.2, 0.0] & 0.111$\pm$0.072 & \textbf{0.076$\pm$0.055} & 88 & \textbf{0.00} \\
	micro-mass & Gibbs & (0.0, 0.2] & 0.105$\pm$0.071 & \textbf{0.088$\pm$0.063} & 378 & \textbf{0.00} \\
	micro-mass & Gibbs & (0.2, 0.4] & \textbf{0.169$\pm$0.104} & 0.179$\pm$0.089 & 398 & 0.14 \\
	micro-mass & IW & (0.0, 0.2] & 0.151$\pm$0.085 & \textbf{0.114$\pm$0.076} & 176 & \textbf{0.00} \\
	micro-mass & SNIW & (0.0, 0.2] & 0.143$\pm$0.086 & \textbf{0.115$\pm$0.044} & 72 & \textbf{0.01} \\
	micro-mass & SNIW & (0.2, 0.4] & 0.267$\pm$0.138 & \textbf{0.241$\pm$0.119} & 81 & 0.21 \\
	micro-mass & SNIW & (0.4, 0.6] & 0.328$\pm$0.194 & \textbf{0.251$\pm$0.124} & 23 & 0.12 \\
	\midrule
	optdigits & Gibbs & (-0.2, 0.0] & 0.134$\pm$0.087 & \textbf{0.058$\pm$0.050} & 78 & \textbf{0.00} \\
	optdigits & Gibbs & (0.0, 0.2] & 0.125$\pm$0.085 & \textbf{0.087$\pm$0.059} & 420 & \textbf{0.00} \\
	optdigits & Gibbs & (0.2, 0.4] & \textbf{0.164$\pm$0.104} & 0.185$\pm$0.082 & 266 & \textbf{0.01} \\
	optdigits & IW & (0.0, 0.2] & 0.188$\pm$0.158 & \textbf{0.177$\pm$0.065} & 78 & 0.58 \\
	optdigits & IW & (0.2, 0.4] & \textbf{0.148$\pm$0.108} & 0.196$\pm$0.089 & 51 & \textbf{0.02} \\
	optdigits & IW & (0.4, 0.6] & \textbf{0.104$\pm$0.094} & 0.227$\pm$0.063 & 27 & \textbf{0.00} \\
	optdigits & SNIW & (0.0, 0.2] & 0.120$\pm$0.046 & \textbf{0.105$\pm$0.056} & 53 & 0.13 \\
	optdigits & SNIW & (0.2, 0.4] & \textbf{0.078$\pm$0.069} & 0.229$\pm$0.105 & 43 & \textbf{0.00} \\
	optdigits & SNIW & (0.4, 0.6] & \textbf{0.107$\pm$0.088} & 0.212$\pm$0.096 & 60 & \textbf{0.00} \\
	\midrule
	page-blocks & Gibbs & (-0.2, 0.0] & 0.114$\pm$0.117 & \textbf{0.085$\pm$0.073} & 58 & 0.11 \\
	page-blocks & Gibbs & (0.0, 0.2] & 0.083$\pm$0.078 & \textbf{0.078$\pm$0.059} & 206 & 0.48 \\
	page-blocks & Gibbs & (0.2, 0.4] & 0.152$\pm$0.118 & \textbf{0.132$\pm$0.094} & 36 & 0.43 \\
	page-blocks & Gibbs & (0.4, 0.6] & \textbf{0.238$\pm$0.199} & 0.247$\pm$0.210 & 124 & 0.74 \\
	page-blocks & Gibbs & (0.6, 0.8] & 0.336$\pm$0.248 & 0.366$\pm$0.251 & 112 & 0.38 \\
	page-blocks & Gibbs & (0.8, 1.0] & \textbf{0.310$\pm$0.259} & 0.313$\pm$0.263 & 12 & 0.97 \\
	page-blocks & IW & (0.0, 0.2] & 0.106$\pm$0.099 & \textbf{0.053$\pm$0.049} & 65 & \textbf{0.00} \\
	page-blocks & IW & (0.2, 0.4] & \textbf{0.021$\pm$0.009} & 0.067$\pm$0.014 & 7 & \textbf{0.00} \\
	page-blocks & IW & (0.4, 0.6] & 0.076$\pm$0.066 & \textbf{0.065$\pm$0.038} & 6 & 0.73 \\
	page-blocks & IW & (0.6, 0.8] & 0.125$\pm$0.121 & 0.124$\pm$0.082 & 38 & 0.96 \\
	page-blocks & SNIW & (0.0, 0.2] & 0.189$\pm$0.083 & \textbf{0.066$\pm$0.085} & 39 & \textbf{0.00} \\
	page-blocks & SNIW & (0.2, 0.4] & 0.202$\pm$0.050 & \textbf{0.027$\pm$0.021} & 9 & \textbf{0.00} \\
	page-blocks & SNIW & (0.4, 0.6] & \textbf{0.179$\pm$0.149} & 0.198$\pm$0.118 & 39 & 0.53 \\
	page-blocks & SNIW & (0.6, 0.8] & 0.071$\pm$0.076 & \textbf{0.062$\pm$0.054} & 29 & 0.60 \\
	\midrule
	pendigits & Gibbs & (-0.2, 0.0] & 0.134$\pm$0.087 & \textbf{0.067$\pm$0.046} & 78 & \textbf{0.00} \\
	pendigits & Gibbs & (0.0, 0.2] & 0.127$\pm$0.087 & \textbf{0.084$\pm$0.059} & 444 & \textbf{0.00} \\
	pendigits & Gibbs & (0.2, 0.4] & \textbf{0.159$\pm$0.104} & 0.164$\pm$0.088 & 242 & 0.53 \\
	pendigits & IW & (0.0, 0.2] & 0.169$\pm$0.129 & \textbf{0.102$\pm$0.039} & 94 & \textbf{0.00} \\
	pendigits & IW & (0.2, 0.4] & 0.156$\pm$0.140 & \textbf{0.136$\pm$0.054} & 38 & 0.41 \\
	pendigits & IW & (0.4, 0.6] & \textbf{0.113$\pm$0.121} & 0.149$\pm$0.026 & 24 & 0.17 \\
	pendigits & SNIW & (0.0, 0.2] & 0.097$\pm$0.035 & 0.098$\pm$0.046 & 53 & 0.90 \\
	pendigits & SNIW & (0.2, 0.4] & \textbf{0.101$\pm$0.084} & 0.150$\pm$0.055 & 38 & \textbf{0.00} \\
	pendigits & SNIW & (0.4, 0.6] & \textbf{0.088$\pm$0.069} & 0.165$\pm$0.044 & 63 & \textbf{0.00} \\
	pendigits & SNIW & (0.6, 0.8] & \textbf{0.075$\pm$0.002} & 0.112$\pm$0.001 & 2 & \textbf{0.00} \\
	\midrule
	soybean & Gibbs & (-0.2, 0.0] & 0.102$\pm$0.063 & \textbf{0.072$\pm$0.050} & 88 & \textbf{0.00} \\
	soybean & Gibbs & (0.0, 0.2] & 0.104$\pm$0.070 & \textbf{0.086$\pm$0.060} & 384 & \textbf{0.00} \\
	soybean & Gibbs & (0.2, 0.4] & \textbf{0.165$\pm$0.104} & 0.178$\pm$0.092 & 392 & 0.07 \\
	soybean & IW & (0.0, 0.2] & 0.139$\pm$0.077 & \textbf{0.120$\pm$0.067} & 149 & \textbf{0.02} \\
	soybean & IW & (0.2, 0.4] & 0.241$\pm$0.175 & \textbf{0.066$\pm$0.063} & 21 & \textbf{0.00} \\
	soybean & IW & (0.4, 0.6] & 0.493$\pm$0.007 & \textbf{0.009$\pm$0.002} & 6 & \textbf{0.00} \\
	soybean & SNIW & (0.0, 0.2] & 0.141$\pm$0.088 & \textbf{0.078$\pm$0.027} & 96 & \textbf{0.00} \\
	soybean & SNIW & (0.2, 0.4] & 0.320$\pm$0.059 & \textbf{0.039$\pm$0.034} & 36 & \textbf{0.00} \\
	soybean & SNIW & (0.4, 0.6] & 0.336$\pm$0.151 & \textbf{0.116$\pm$0.055} & 44 & \textbf{0.00} \\
	\midrule
	yeast & Gibbs & (-0.2, 0.0] & 0.112$\pm$0.076 & \textbf{0.089$\pm$0.069} & 78 & 0.06 \\
	yeast & Gibbs & (0.0, 0.2] & 0.116$\pm$0.080 & \textbf{0.108$\pm$0.078} & 456 & 0.10 \\
	yeast & Gibbs & (0.2, 0.4] & 0.143$\pm$0.109 & \textbf{0.142$\pm$0.088} & 166 & 0.95 \\
	yeast & Gibbs & (0.4, 0.6] & \textbf{0.211$\pm$0.123} & 0.225$\pm$0.116 & 64 & 0.51 \\
	yeast & IW & (0.0, 0.2] & 0.147$\pm$0.118 & \textbf{0.099$\pm$0.085} & 119 & \textbf{0.00} \\
	yeast & IW & (0.2, 0.4] & 0.188$\pm$0.157 & \textbf{0.084$\pm$0.062} & 37 & \textbf{0.00} \\
	yeast & SNIW & (0.0, 0.2] & 0.148$\pm$0.124 & \textbf{0.094$\pm$0.070} & 98 & \textbf{0.00} \\
	yeast & SNIW & (0.2, 0.4] & 0.220$\pm$0.140 & \textbf{0.110$\pm$0.067} & 44 & \textbf{0.00} \\
	yeast & SNIW & (0.4, 0.6] & 0.378$\pm$0.211 & \textbf{0.160$\pm$0.071} & 14 & \textbf{0.00} \\
	\bottomrule
	\end{longtable}
}

\FloatBarrier

\begin{figure}
	\centering
	\begin{subfigure}{0.49\linewidth}
		\centering
		\includegraphics[width=\linewidth]{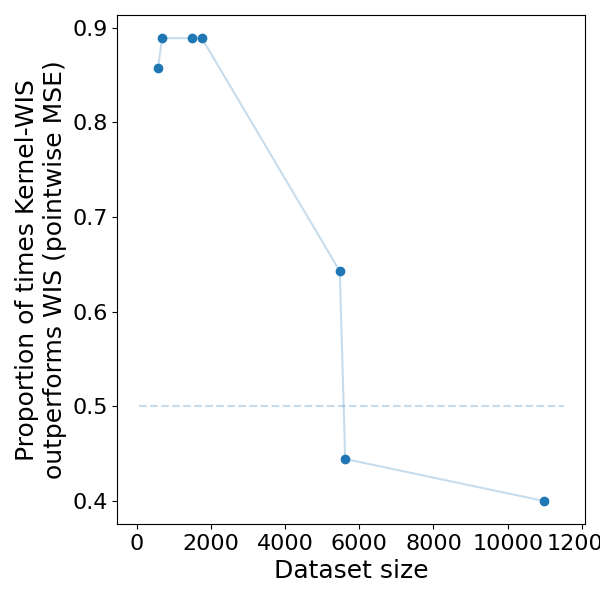}
		\caption{Proportion}
		\label{fig:kwis_performance_no_buffer_by_n_samples_non_oracle}
	\end{subfigure}
	\begin{subfigure}{0.49\linewidth}
		\centering
		\includegraphics[width=\linewidth]{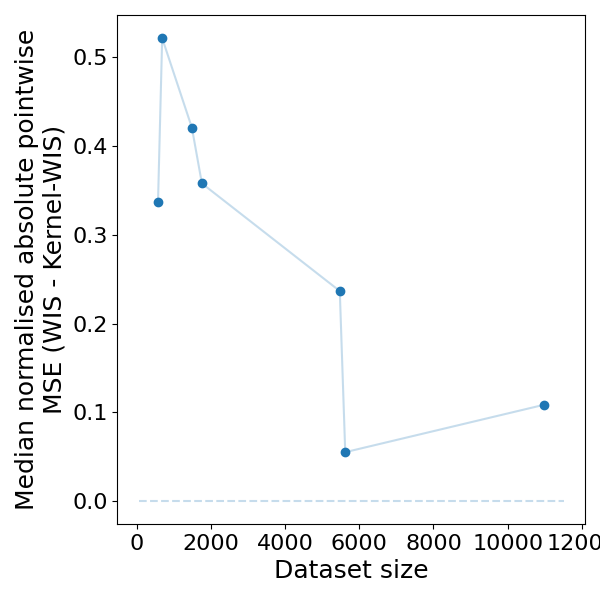}
		\caption{Median normalised difference}
		\label{fig:norm_performance_by_n_samples_non_oracle}
	\end{subfigure}
	\begin{subfigure}{0.49\linewidth}
		\centering
		\includegraphics[width=\linewidth]{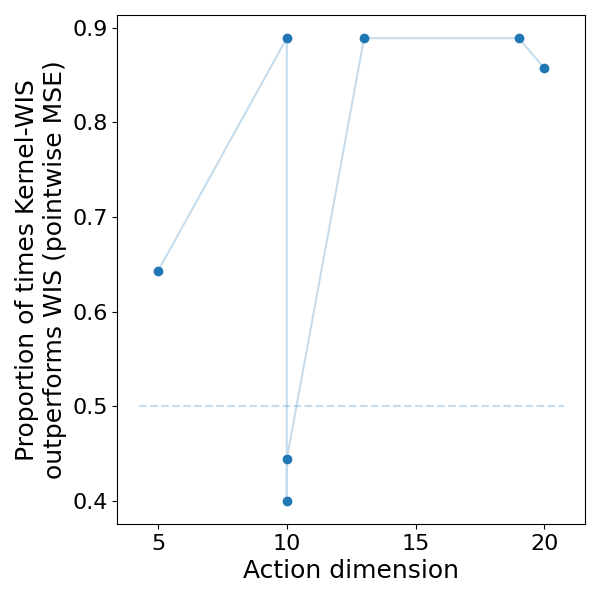}
		\caption{Proportion}
		\label{fig:kwis_performance_no_buffer_by_action_dim_non_oracle}
	\end{subfigure}
	\begin{subfigure}{0.49\linewidth}
		\centering
		\includegraphics[width=\linewidth]{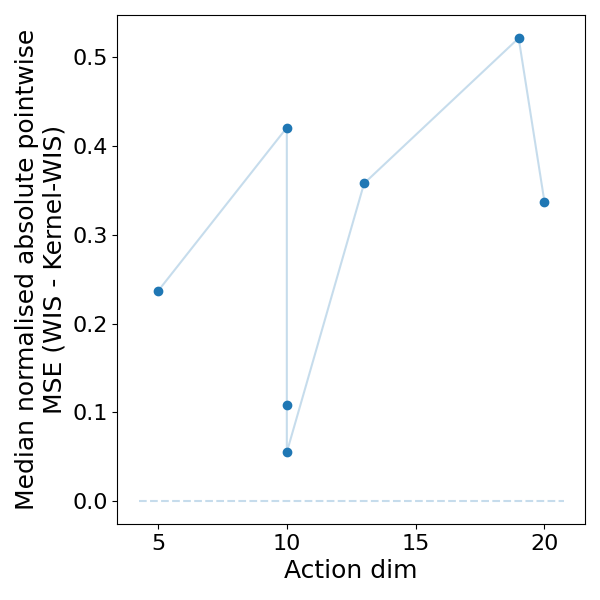}
		\caption{Median normalised difference}
		\label{fig:norm_performance_by_action_dim_non_oracle}
	\end{subfigure}
	\caption{The figures describe two views on the difference in pointwise mean squared error for the Kernel-WIS and WIS estimators against dataset size (top) and action dimension (bottom) on the x-axis. Data are shown for the single action reward, non-oracle behaviour policy setting. The left plots describe the proportion of scenarios (i.e., unique combinations of policy type and buckets differences in true evaluation and logging policy return) where the Kernel-WIS estimator strictly outperforms the WIS estimator. The right plots describe the median normalised difference in pointwise mean squared error ($\textrm{AD}(\textrm{WIS})-\textrm{AD}(\textrm{Kernel-WIS}/(\textrm{AD}(\textrm{WIS})+\textrm{AD}(\textrm{Kernel-WIS}))$).}
	\label{fig:summary_performance_non_oracle}
\end{figure}

Table \ref{tbl:k_wis_wis_non_oracle_width} displays the same pointwise MSE performance as displayed in table \ref{tbl:k_wis_wis_non_oracle_results}, but similarly to oracle setting, the coverage of the bootstrapped confidence intervals is also included. These results are summarised in figures \ref{fig:both_cov_prop_non_oracle} and \ref{fig:cov_diff_non_oracle}. These figures demonstrate that the coverage of the boostrapped intervals is strongly related to the pointwise MSE performance (i.e., figure \ref{fig:kwis_performance_no_buffer_by_n_samples_non_oracle}), which is understandable given the confidence intervals themselves are biased. Arguably, the Kernel-WIS estimator trades off some coverage for the increased accuracy since, by comparing figure \ref{fig:kwis_performance_no_buffer_by_n_samples_non_oracle} with figure \ref{fig:both_cov_prop_oracle}, the coverage performance of the Kernel-WIS estimator degrades more rapidly than the pointwise performance.\\

\FloatBarrier

{\small\tabcolsep=3pt  
	\begin{longtable}{*{7}{>{\footnotesize}c}}
	\caption{The table describes the pointwise mean squared error and the width of the bootstrapped confidence intervals of the WIS estimator and Kernel-WIS estimator (with a shared bandwidth), split by dataset, policy type and a discretisation of the absolute difference in true performance under the evaluation and logging policy, for non-oracle behaviour policies. Aggregations for each row are across logging policy temperature and logging policy faulty action. Values in bold in the WIS and Kernel-WIS columns highlight the best performing estimator. Neither are highlighted if the mean performance is within 0.001.}
	\label{tbl:k_wis_wis_non_oracle_width} \\
	\toprule
	& & & \multicolumn{2}{c}{Estimate} & \multicolumn{2}{c}{Coverage} \\ 
	Dataset & Policy Type & \makecell{Eval/Log\\Abs. Diff.} & WIS & Kernel-WIS & WIS & Kernel-WIS \\
	\midrule
	\endfirsthead
	\toprule
	& & & \multicolumn{2}{c}{Estimate} & \multicolumn{2}{c}{Coverage} \\ 
	Dataset & Policy Type & \makecell{Eval/Log\\Abs. Diff.} & WIS & Kernel-WIS & WIS & Kernel-WIS \\
	\midrule
	\endhead            
	\midrule\multicolumn{7}{r}{{Continued on next page}} \\ 
	\endfoot     
	\endlastfoot 
	arrhythmia & Gibbs & (-0.2, 0.0] & 0.133$\pm$0.101 & \textbf{0.098$\pm$0.093} & 0.306$\pm$0.366 & \textbf{0.433$\pm$0.405} \\
	arrhythmia & Gibbs & (0.0, 0.2] & 0.097$\pm$0.074 & \textbf{0.074$\pm$0.056} & 0.417$\pm$0.400 & \textbf{0.463$\pm$0.416} \\
	arrhythmia & Gibbs & (0.2, 0.4] & 0.218$\pm$0.148 & \textbf{0.201$\pm$0.122} & \textbf{0.342$\pm$0.427} & 0.317$\pm$0.414 \\
	arrhythmia & Gibbs & (0.4, 0.6] & \textbf{0.221$\pm$0.150} & 0.249$\pm$0.139 & \textbf{0.408$\pm$0.466} & 0.327$\pm$0.439 \\
	arrhythmia & IW & (0.0, 0.2] & 0.155$\pm$0.115 & \textbf{0.095$\pm$0.084} & 0.559$\pm$0.368 & \textbf{0.681$\pm$0.346} \\
	arrhythmia & IW & (0.2, 0.4] & 0.284$\pm$0.117 & \textbf{0.221$\pm$0.109} & 0.131$\pm$0.293 & \textbf{0.202$\pm$0.325} \\
	arrhythmia & SNIW & (0.0, 0.2] & 0.157$\pm$0.110 & \textbf{0.092$\pm$0.050} & 0.622$\pm$0.312 & \textbf{0.718$\pm$0.284} \\
	arrhythmia & SNIW & (0.2, 0.4] & 0.082$\pm$0.045 & \textbf{0.078$\pm$0.046} & 0.400$\pm$0.110 & \textbf{0.422$\pm$0.134} \\
	arrhythmia & SNIW & (0.4, 0.6] & 0.455$\pm$0.210 & \textbf{0.257$\pm$0.121} & 0.074$\pm$0.251 & 0.074$\pm$0.251 \\
	\midrule
	micro-mass & Gibbs & (-0.2, 0.0] & 0.111$\pm$0.072 & \textbf{0.076$\pm$0.055} & 0.381$\pm$0.380 & \textbf{0.442$\pm$0.378} \\
	micro-mass & Gibbs & (0.0, 0.2] & 0.105$\pm$0.071 & \textbf{0.088$\pm$0.063} & 0.367$\pm$0.365 & \textbf{0.397$\pm$0.360} \\
	micro-mass & Gibbs & (0.2, 0.4] & \textbf{0.169$\pm$0.104} & 0.179$\pm$0.089 & \textbf{0.362$\pm$0.414} & 0.239$\pm$0.354 \\
	micro-mass & IW & (0.0, 0.2] & 0.151$\pm$0.085 & \textbf{0.114$\pm$0.076} & 0.494$\pm$0.370 & \textbf{0.615$\pm$0.386} \\
	micro-mass & SNIW & (0.0, 0.2] & 0.143$\pm$0.086 & \textbf{0.115$\pm$0.044} & 0.420$\pm$0.259 & \textbf{0.423$\pm$0.262} \\
	micro-mass & SNIW & (0.2, 0.4] & 0.267$\pm$0.138 & \textbf{0.241$\pm$0.119} & \textbf{0.140$\pm$0.238} & 0.133$\pm$0.222 \\
	micro-mass & SNIW & (0.4, 0.6] & 0.328$\pm$0.194 & \textbf{0.251$\pm$0.124} & 0.068$\pm$0.115 & 0.068$\pm$0.115 \\
	\midrule
	optdigits & Gibbs & (-0.2, 0.0] & 0.134$\pm$0.087 & \textbf{0.058$\pm$0.050} & 0.066$\pm$0.210 & \textbf{0.239$\pm$0.378} \\
	optdigits & Gibbs & (0.0, 0.2] & 0.125$\pm$0.085 & \textbf{0.087$\pm$0.059} & 0.089$\pm$0.241 & \textbf{0.142$\pm$0.305} \\
	optdigits & Gibbs & (0.2, 0.4] & \textbf{0.164$\pm$0.104} & 0.185$\pm$0.082 & \textbf{0.125$\pm$0.308} & 0.024$\pm$0.140 \\
	optdigits & IW & (0.0, 0.2] & 0.188$\pm$0.158 & \textbf{0.177$\pm$0.065} & \textbf{0.256$\pm$0.387} & 0.057$\pm$0.208 \\
	optdigits & IW & (0.2, 0.4] & \textbf{0.148$\pm$0.108} & 0.196$\pm$0.089 & \textbf{0.261$\pm$0.348} & 0.115$\pm$0.265 \\
	optdigits & IW & (0.4, 0.6] & \textbf{0.104$\pm$0.094} & 0.227$\pm$0.063 & \textbf{0.593$\pm$0.441} & 0.049$\pm$0.181 \\
	optdigits & SNIW & (0.0, 0.2] & 0.120$\pm$0.046 & \textbf{0.105$\pm$0.056} & 0.059$\pm$0.139 & \textbf{0.119$\pm$0.272} \\
	optdigits & SNIW & (0.2, 0.4] & \textbf{0.078$\pm$0.069} & 0.229$\pm$0.105 & \textbf{0.545$\pm$0.423} & 0.080$\pm$0.191 \\
	optdigits & SNIW & (0.4, 0.6] & \textbf{0.107$\pm$0.088} & 0.212$\pm$0.096 & \textbf{0.444$\pm$0.444} & 0.059$\pm$0.142 \\
	\midrule
	page-blocks & Gibbs & (-0.2, 0.0] & 0.114$\pm$0.117 & \textbf{0.085$\pm$0.073} & 0.159$\pm$0.353 & \textbf{0.195$\pm$0.382} \\
	page-blocks & Gibbs & (0.0, 0.2] & 0.083$\pm$0.078 & \textbf{0.078$\pm$0.059} & \textbf{0.161$\pm$0.347} & 0.099$\pm$0.276 \\
	page-blocks & Gibbs & (0.2, 0.4] & 0.152$\pm$0.118 & \textbf{0.132$\pm$0.094} & \textbf{0.244$\pm$0.417} & 0.170$\pm$0.340 \\
	page-blocks & Gibbs & (0.4, 0.6] & \textbf{0.238$\pm$0.199} & 0.247$\pm$0.210 & 0.259$\pm$0.401 & 0.259$\pm$0.396 \\
	page-blocks & Gibbs & (0.6, 0.8] & 0.336$\pm$0.248 & 0.366$\pm$0.251 & \textbf{0.231$\pm$0.404} & 0.184$\pm$0.364 \\
	page-blocks & Gibbs & (0.8, 1.0] & \textbf{0.310$\pm$0.259} & 0.313$\pm$0.263 & 0.324$\pm$0.422 & \textbf{0.343$\pm$0.438} \\
	page-blocks & IW & (0.0, 0.2] & 0.106$\pm$0.099 & \textbf{0.053$\pm$0.049} & 0.137$\pm$0.323 & \textbf{0.232$\pm$0.374} \\
	page-blocks & IW & (0.2, 0.4] & \textbf{0.021$\pm$0.009} & 0.067$\pm$0.014 & \textbf{0.921$\pm$0.124} & 0.000$\pm$0.000 \\
	page-blocks & IW & (0.4, 0.6] & 0.076$\pm$0.066 & \textbf{0.065$\pm$0.038} & \textbf{0.648$\pm$0.504} & 0.630$\pm$0.414 \\
	page-blocks & IW & (0.6, 0.8] & 0.125$\pm$0.121 & 0.124$\pm$0.082 & \textbf{0.594$\pm$0.435} & 0.570$\pm$0.421 \\
	page-blocks & SNIW & (0.0, 0.2] & 0.189$\pm$0.083 & \textbf{0.066$\pm$0.085} & 0.083$\pm$0.213 & \textbf{0.368$\pm$0.406} \\
	page-blocks & SNIW & (0.2, 0.4] & 0.202$\pm$0.050 & \textbf{0.027$\pm$0.021} & 0.000$\pm$0.000 & \textbf{0.469$\pm$0.471} \\
	page-blocks & SNIW & (0.4, 0.6] & \textbf{0.179$\pm$0.149} & 0.198$\pm$0.118 & \textbf{0.385$\pm$0.442} & 0.236$\pm$0.384 \\
	page-blocks & SNIW & (0.6, 0.8] & 0.071$\pm$0.076 & \textbf{0.062$\pm$0.054} & 0.866$\pm$0.258 & \textbf{0.874$\pm$0.273} \\
	\midrule
	pendigits & Gibbs & (-0.2, 0.0] & 0.134$\pm$0.087 & \textbf{0.067$\pm$0.046} & 0.044$\pm$0.175 & \textbf{0.078$\pm$0.236} \\
	pendigits & Gibbs & (0.0, 0.2] & 0.127$\pm$0.087 & \textbf{0.084$\pm$0.059} & 0.050$\pm$0.182 & \textbf{0.123$\pm$0.278} \\
	pendigits & Gibbs & (0.2, 0.4] & \textbf{0.159$\pm$0.104} & 0.164$\pm$0.088 & \textbf{0.110$\pm$0.295} & 0.022$\pm$0.139 \\
	pendigits & IW & (0.0, 0.2] & 0.169$\pm$0.129 & \textbf{0.102$\pm$0.039} & \textbf{0.092$\pm$0.283} & 0.008$\pm$0.034 \\
	pendigits & IW & (0.2, 0.4] & 0.156$\pm$0.140 & \textbf{0.136$\pm$0.054} & \textbf{0.269$\pm$0.420} & 0.000$\pm$0.000 \\
	pendigits & IW & (0.4, 0.6] & \textbf{0.113$\pm$0.121} & 0.149$\pm$0.026 & \textbf{0.597$\pm$0.456} & 0.000$\pm$0.000 \\
	pendigits & SNIW & (0.0, 0.2] & \textbf{0.097$\pm$0.035} & 0.098$\pm$0.046 & \textbf{0.040$\pm$0.093} & 0.000$\pm$0.000 \\
	pendigits & SNIW & (0.2, 0.4] & \textbf{0.101$\pm$0.084} & 0.150$\pm$0.055 & \textbf{0.316$\pm$0.382} & 0.023$\pm$0.090 \\
	pendigits & SNIW & (0.4, 0.6] & \textbf{0.088$\pm$0.069} & 0.165$\pm$0.044 & \textbf{0.460$\pm$0.452} & 0.000$\pm$0.000 \\
	pendigits & SNIW & (0.6, 0.8] & \textbf{0.075$\pm$0.002} & 0.112$\pm$0.001 & \textbf{0.278$\pm$0.079} & 0.000$\pm$0.000 \\
	\midrule
	soybean & Gibbs & (-0.2, 0.0] & 0.102$\pm$0.063 & \textbf{0.072$\pm$0.050} & \textbf{0.362$\pm$0.408} & 0.359$\pm$0.401 \\
	soybean & Gibbs & (0.0, 0.2] & 0.104$\pm$0.070 & \textbf{0.086$\pm$0.060} & \textbf{0.326$\pm$0.363} & 0.297$\pm$0.345 \\
	soybean & Gibbs & (0.2, 0.4] & \textbf{0.165$\pm$0.104} & 0.178$\pm$0.092 & \textbf{0.364$\pm$0.421} & 0.187$\pm$0.336 \\
	soybean & IW & (0.0, 0.2] & 0.139$\pm$0.077 & \textbf{0.120$\pm$0.067} & 0.477$\pm$0.353 & \textbf{0.509$\pm$0.343} \\
	soybean & IW & (0.2, 0.4] & 0.241$\pm$0.175 & \textbf{0.066$\pm$0.063} & 0.302$\pm$0.382 & \textbf{0.725$\pm$0.319} \\
	soybean & IW & (0.4, 0.6] & 0.493$\pm$0.007 & \textbf{0.009$\pm$0.002} & 0.000$\pm$0.000 & \textbf{1.000$\pm$0.000} \\
	soybean & SNIW & (0.0, 0.2] & 0.141$\pm$0.088 & \textbf{0.078$\pm$0.027} & 0.451$\pm$0.263 & \textbf{0.644$\pm$0.160} \\
	soybean & SNIW & (0.2, 0.4] & 0.320$\pm$0.059 & \textbf{0.039$\pm$0.034} & 0.043$\pm$0.181 & \textbf{0.849$\pm$0.219} \\
	soybean & SNIW & (0.4, 0.6] & 0.336$\pm$0.151 & \textbf{0.116$\pm$0.055} & 0.146$\pm$0.221 & \textbf{0.184$\pm$0.210} \\
	\midrule
	yeast & Gibbs & (-0.2, 0.0] & 0.112$\pm$0.076 & \textbf{0.089$\pm$0.069} & 0.147$\pm$0.297 & \textbf{0.215$\pm$0.361} \\
	yeast & Gibbs & (0.0, 0.2] & 0.116$\pm$0.080 & \textbf{0.108$\pm$0.078} & \textbf{0.220$\pm$0.349} & 0.213$\pm$0.350 \\
	yeast & Gibbs & (0.2, 0.4] & 0.143$\pm$0.109 & 0.142$\pm$0.088 & \textbf{0.379$\pm$0.434} & 0.261$\pm$0.356 \\
	yeast & Gibbs & (0.4, 0.6] & \textbf{0.211$\pm$0.123} & 0.225$\pm$0.116 & \textbf{0.306$\pm$0.412} & 0.224$\pm$0.349 \\
	yeast & IW & (0.0, 0.2] & 0.147$\pm$0.118 & \textbf{0.099$\pm$0.085} & 0.492$\pm$0.460 & \textbf{0.556$\pm$0.444} \\
	yeast & IW & (0.2, 0.4] & 0.188$\pm$0.157 & \textbf{0.084$\pm$0.062} & 0.342$\pm$0.452 & \textbf{0.348$\pm$0.439} \\
	yeast & SNIW & (0.0, 0.2] & 0.148$\pm$0.124 & \textbf{0.094$\pm$0.070} & 0.471$\pm$0.399 & \textbf{0.512$\pm$0.410} \\
	yeast & SNIW & (0.2, 0.4] & 0.220$\pm$0.140 & \textbf{0.110$\pm$0.067} & 0.316$\pm$0.407 & \textbf{0.351$\pm$0.382} \\
	yeast & SNIW & (0.4, 0.6] & 0.378$\pm$0.211 & \textbf{0.160$\pm$0.071} & \textbf{0.151$\pm$0.249} & 0.143$\pm$0.240 \\
	\bottomrule
	\end{longtable}
}

\FloatBarrier

\begin{figure}[!h]
	\centering
	\begin{subfigure}{0.49\linewidth}
		\centering
		\includegraphics[width=\linewidth]{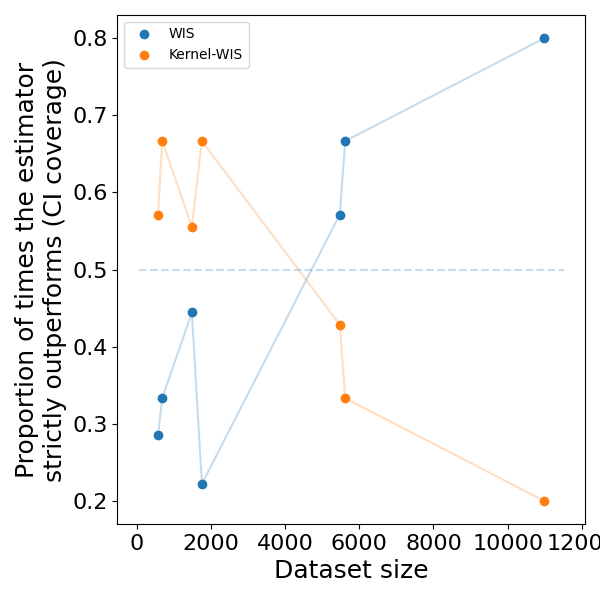}
		\caption{Proportion}
		\label{fig:both_cov_prop_non_oracle}
	\end{subfigure}
	\begin{subfigure}{0.49\linewidth}
		\centering
		\includegraphics[width=\linewidth]{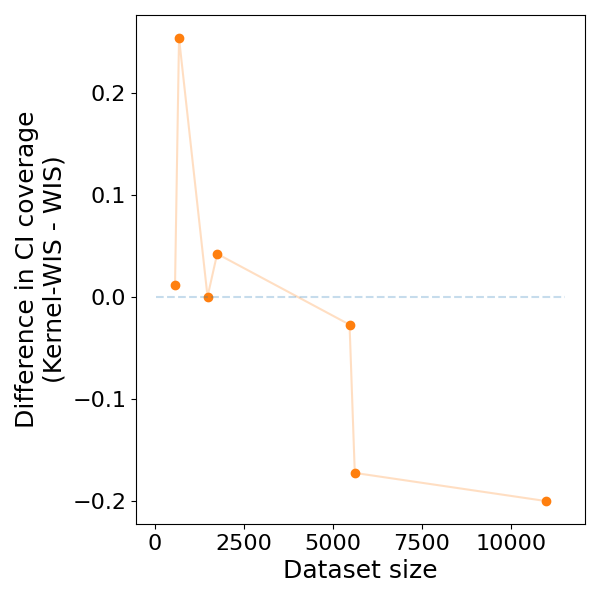}
		\caption{Difference}
		\label{fig:cov_diff_non_oracle}
	\end{subfigure}
	\caption{The figures describe two views of the difference in bootstrapped interval coverage results for the Kernel-WIS and WIS estimators against dataset size on the x-axis. Data are shown for the single action reward, non-oracle behaviour policy setting. The left plot describes the proportion of times the coverage under each estimator is strictly larger than the other. The right plot describes the difference in coverage values.}
\end{figure}

\FloatBarrier

\subsubsection{Continuous reward sensitivity}\label{sec:cont_reward_sensitivity}
All of the results described in sections \ref{sec:oracle_beh_policy_results} and \ref{sec:non_oracle_beh_sensitivity} pertained to contextual bandit definitions where all non-optimal actions achieved a reward of $0$ i.e., $\forall s \in \mathcal{S}, \exists !a \in \mathcal{A}: R(a,s) = 1, \forall a' \neq a, R(a,s) = 0$. Clearly, this is not representative of all contextual bandit definitions. To try and assess the stability of the results in section \ref{sec:oracle_beh_policy_results} (i.e., under oracle behaviour policies only), a continuous reward analogue of the experiments performed on the optdigits dataset was created. To define the continuous rewards, a variational autoencoder (VAE) (\cite{kingma_auto-encoding_2022}) was trained on the dataset and the reward of suboptimal actions was defined by the distance to the mean of the suboptimal action under the latent space of the VAE. Concretely, consider a VAE, defined as $g\cdot f(s)$. For a given state $s_{i}$, there exists an optimal action $a_{\textrm{max}} = \argmax_{a'}R(a',s_{i})$, the associated reward of which is 1, where $R$ is the reward function under the single action reward setup. For all other actions, $a'\neq a_{\textrm{max}}$, define the unnormalised reward as: 
\begin{align*}
	R_{\textrm{Cont-U}}(a',s_{i}) = \Bigg|f(s_{i}) - \frac{1}{n_{\mathcal{H}}}\sum_{j \in \mathcal{H}}f(s_{j})\Bigg|,	
\end{align*}
where $\mathcal{H}$ is the set of indices such that $a'$ is optimal under the single action reward setting and $n_{\mathcal{H}} = \sum_{i=1}^{n}\mathbbm{1}(i \in \mathcal{H})$. The reward function used for the continuous setting, $R_{\textrm{Cont}}$, is derived from $R_{\textrm{Cont-U}}(a',s_{i})$ by normalising as:
\begin{align*}
	R_{\textrm{Cont}}(a',s_{i}) = \begin{cases}
		1 & \textrm{if } a' = \argmax_{a'}R(a',s_{i}) \\
		R_{\textrm{Cont-U}}(a',s_{i}) / \max_{a''}R_{\textrm{Cont-U}}(a'',s_{i}) & \textrm{otherwise.}
	\end{cases}
\end{align*}
This approach was inspired by \cite{kingma_auto-encoding_2022} who originally proposed the VAE architecture and demonstrated the ability to naturally interpolate between digits. Evidence of convergence of the model is provided in appendix section \ref{sec:optdigits_gen_appendix} along with other training details. Figure \ref{fig:mean_dists_cont_reward} presents a heatmap of the rewards for non-optimal actions under the optdigitDist dataset, which displays some intuitive results. For example, predicting the number 9 when the target is 3 receives a relatively high reward whilst predicting the number 6 receives a low reward. Figure \ref{fig:optdigit_data_examps} displays two examples of a digit 3, 6 and 9 chosen randomly from the optdigits dataset. Given the similar orientation of 3 and 9 in comparison to 3 and 6, the proposed reward structure is not unreasonable.\\

\begin{figure}[!h]
	\centering
	\includegraphics[width=\linewidth]{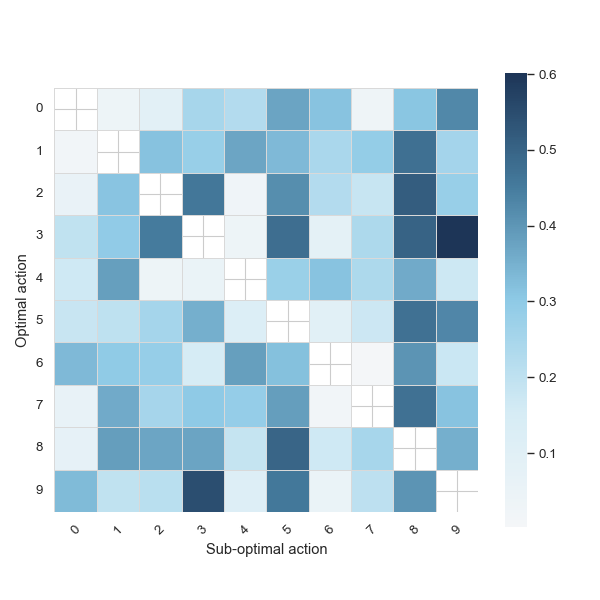}
	\caption{The figure displays a heatmap of the mean reward under sub-optimal actions. The y-axis defines the optimal action and the x-axis defines the sub-optimal action selected as such, the figure should be read from left to right, not up and down.}
	\label{fig:mean_dists_cont_reward}
\end{figure}

\begin{figure}[h!]
	\centering
	\begin{subfigure}{0.49\linewidth}
		\includegraphics[width=\linewidth]{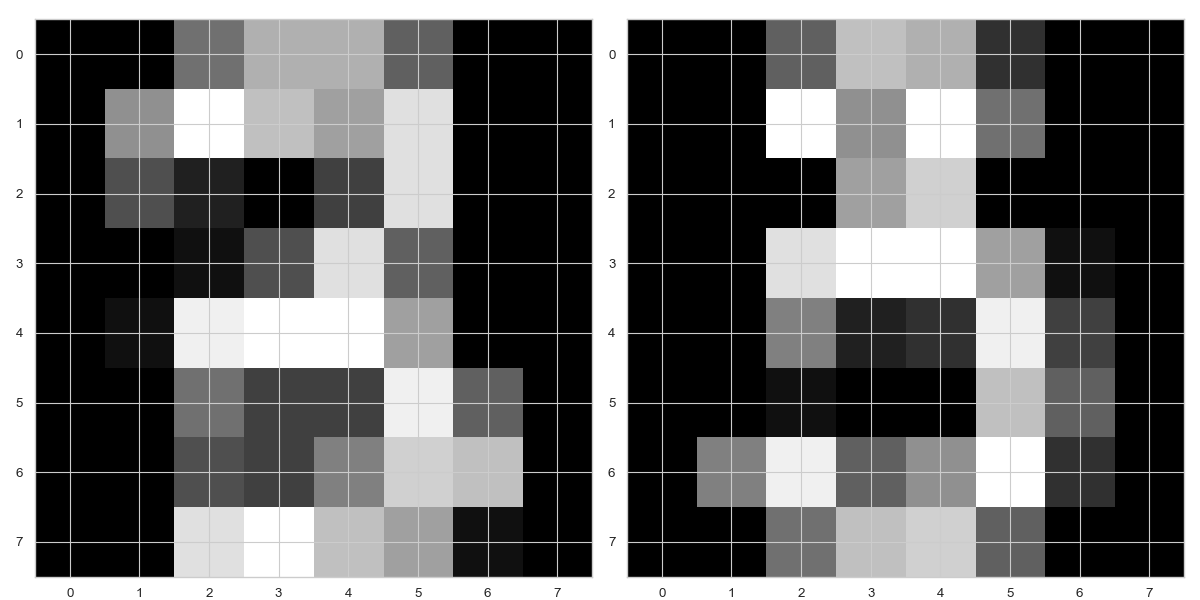}
		\caption{Random observations depicting 3}
	\end{subfigure}
	\begin{subfigure}{0.49\linewidth}
		\includegraphics[width=\linewidth]{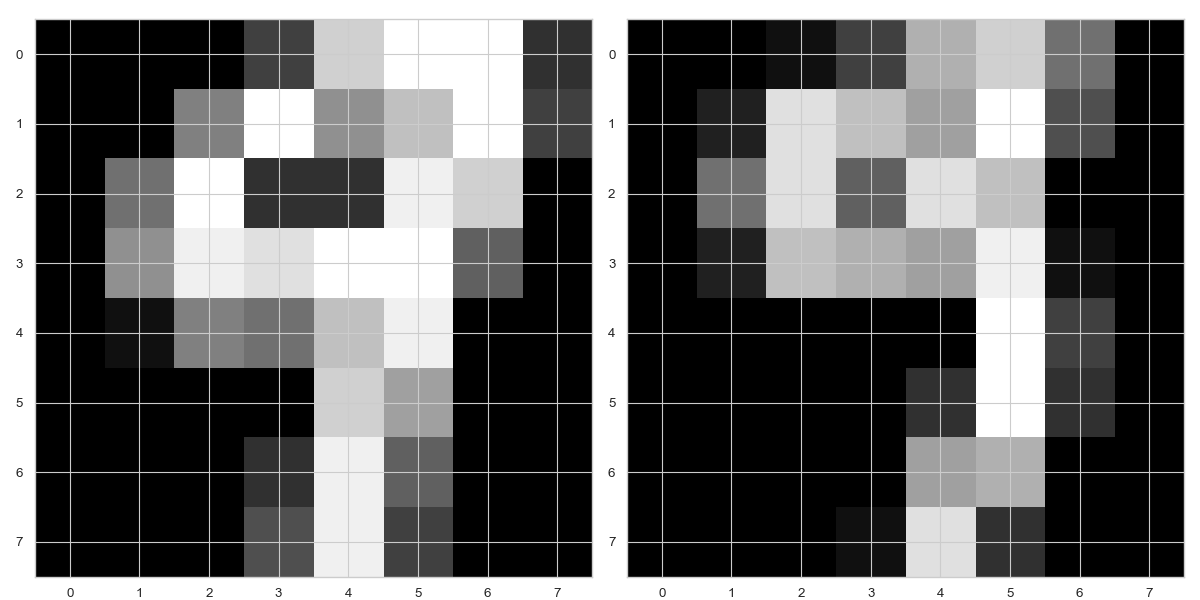}
		\caption{Random observations depicting 9}
	\end{subfigure}
	\begin{subfigure}{0.49\linewidth}
		\includegraphics[width=\linewidth]{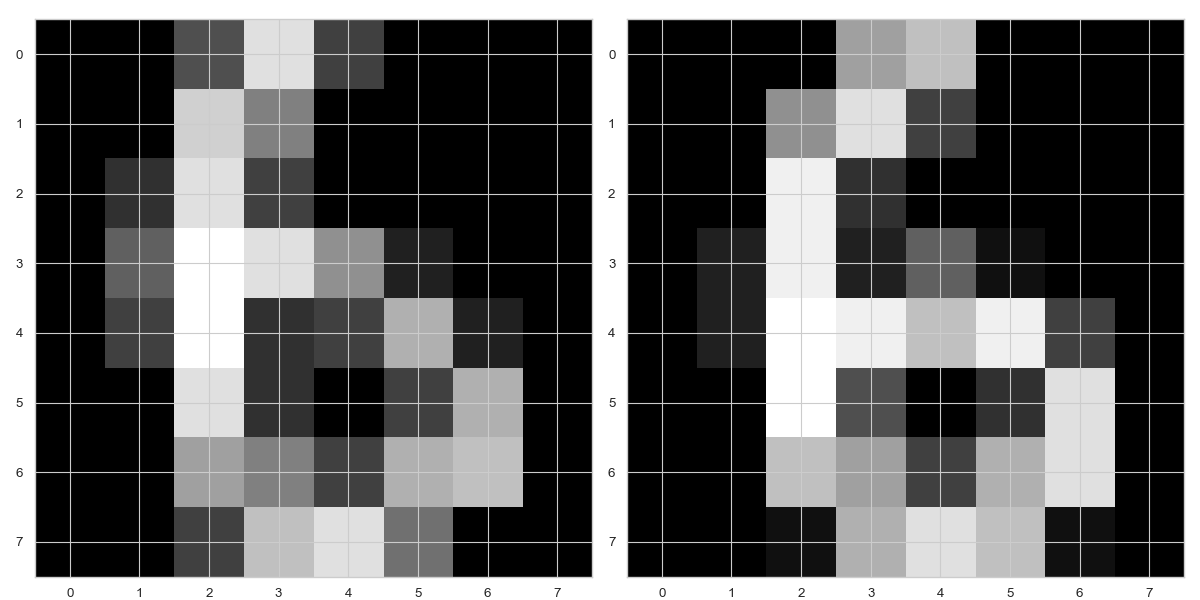}
		\caption{Random observations depicting 6}
	\end{subfigure}
	\caption{The figure depicts randomly selected observations from the optdigits dataset.}
	\label{fig:optdigit_data_examps}
\end{figure}

Clearly, under the continuous reward setting, the relative performance of the Kernel-WIS estimator significantly degraded. However, understanding the precise reasons for this were left for future work.\\

\begin{table}
	\caption{The table compares the performance of the WIS and Kernel-WIS estimators under a single action reward (optdigits) and continuous reward setting (optdigitsDist), using the optdigits dataset. The table describes the pointwise finite mean squared error of the WIS estimator and Kernel-WIS estimator (with a shared bandwidth), split by dataset, policy type and a discretisation of the absolute difference in true performance under the evaluation and logging policy, for oracle behaviour policies. The P-Value column defines the p-value of a two-sided Wald t-test, comparing the finite mean squared error of the WIS and Kernel-WIS estimators. N observations defines the number of experiments aggregated over per row, where the aggregations are across logging policy temperature and logging policy faulty action. Values in bold in the WIS and Kernel-WIS columns highlight the best performing estimator. Neither are highlighted if the mean performance is within 0.001. Values highlighted in the P-Value column identify those which are less than or equal to 0.05.}
	\label{tbl:k_wis_wis_cont_results}
	\resizebox{\linewidth}{!}{
		\begin{tabular}{ccccccc}
		\toprule
		Dataset & Policy Type & Eval/Log Abs. Diff. & WIS & Kernel-WIS & P-Value & N Observations \\
		\midrule
		optdigits & Gibbs & (-0.2, 0.0] & 0.027$\pm$0.013 & 0.027$\pm$0.013 & 14 & 1.00 \\
		optdigits & Gibbs & (0.0, 0.2] & 0.026$\pm$0.011 & 0.026$\pm$0.011 & 74 & 0.97 \\
		optdigits & Gibbs & (0.2, 0.4] & 0.027$\pm$0.011 & 0.027$\pm$0.011 & 44 & 0.99 \\
		optdigits & IW & (0.0, 0.2] & 0.076$\pm$0.011 & 0.077$\pm$0.010 & 5 & 0.94 \\
		optdigits & IW & (0.2, 0.4] & 0.068$\pm$0.006 & 0.068$\pm$0.005 & 9 & 0.92 \\
		optdigits & IW & (0.4, 0.6] & 0.075$\pm$0.004 & 0.075$\pm$0.004 & 12 & 0.90 \\
		optdigits & IW & (0.6, 0.8] & \textbf{0.075$\pm$0.000} & 0.081$\pm$0.002 & 2 & 0.10 \\
		optdigits & SNIW & (0.0, 0.2] & \textbf{0.066$\pm$0.007} & 0.069$\pm$0.005 & 8 & 0.42 \\
		optdigits & SNIW & (0.2, 0.4] & \textbf{0.076$\pm$0.006} & 0.080$\pm$0.006 & 5 & 0.32 \\
		optdigits & SNIW & (0.4, 0.6] & 0.071$\pm$0.006 & 0.071$\pm$0.006 & 15 & 0.94 \\
		\midrule
		optdigitsDist & Gibbs & (-0.2, 0.0] & 0.056$\pm$0.011 & 0.056$\pm$0.011 & 14 & 1.00 \\
		optdigitsDist & Gibbs & (0.0, 0.2] & \textbf{0.054$\pm$0.009} & 0.056$\pm$0.011 & 96 & 0.14 \\
		optdigitsDist & Gibbs & (0.2, 0.4] & \textbf{0.059$\pm$0.015} & 0.064$\pm$0.017 & 22 & 0.38 \\
		optdigitsDist & IW & (0.0, 0.2] & \textbf{0.061$\pm$0.007} & 0.079$\pm$0.025 & 28 & \textbf{0.00} \\
		optdigitsDist & SNIW & (0.0, 0.2] & \textbf{0.062$\pm$0.003} & 0.084$\pm$0.025 & 26 & \textbf{0.00} \\
		optdigitsDist & SNIW & (0.2, 0.4] & \textbf{0.070$\pm$0.001} & 0.133$\pm$0.002 & 2 & \textbf{0.01} \\
		\bottomrule
		\end{tabular}
	}
\end{table}

\FloatBarrier

\subsubsection{Bandwidth sensitivity}\label{sec:bandwidth_sensitivity}
For the results presented in the previous sections, the Kernel-WIS estimator was optimised using a single bandwidth value for every dimension of the feature space i.e., $h$ in equation \ref{equ:kwis_estimator} is a real number and not a vector, herein referred to as Kernel-WIS (Shared). Every experiment was also run by optimising unique bandwidth values for each dimension of the state, herein referred to as Kernel-WIS (Unique). However, this produced worse (or equivalent) results in all but one instance. The aggregate results are provided in table \ref{tbl:overall_results_wis_optim}.\\

\begin{table}[!h]
	\caption{The table describes the pointwise finite mean squared error of the Kernel-WIS estimator with a shared and unique bandwidth (denoted Kernel-WIS (Shared) and Kernel-WIS (Unique)) against the WIS estimator. The results are split by dataset, whether an oracle behaviour policy was used and the reward function structure. Aggregations for each row are across logging policy temperature, logging policy faulty action and policy type.}
	\label{tbl:overall_results_wis_optim}
	\centering
		\begin{tabular}{cccccc}
		\toprule
		\makecell{Reward\\Structure} & \makecell{Behaviour\\Policy} & Dataset & WIS & \makecell{Kernel-WIS\\(Shared)} & \makecell{Kernel-WIS\\(Unique)} \\
		\midrule
		Single action & Oracle & yeast & 0.046$\pm$0.021 & 0.043$\pm$0.018 & 0.043$\pm$0.032 \\
		Single action & Oracle & soybean & 0.063$\pm$0.032 & 0.062$\pm$0.032 & 0.100$\pm$0.066 \\
		Single action & Oracle & page-blocks & 0.064$\pm$0.021 & 0.064$\pm$0.021 & 0.070$\pm$0.024 \\
		Single action & Oracle & pendigits & 0.037$\pm$0.019 & 0.038$\pm$0.022 & 0.057$\pm$0.045 \\
		Single action & Oracle & micro-mass & 0.061$\pm$0.025 & 0.062$\pm$0.024 & 0.126$\pm$0.100 \\
		Single action & Oracle & arrhythmia & 0.059$\pm$0.024 & 0.054$\pm$0.026 & 0.096$\pm$0.113 \\
		Single action & Oracle & optdigits & 0.040$\pm$0.023 & 0.040$\pm$0.024 & 0.130$\pm$0.119 \\
		Single action & Oracle & kropt & 0.037$\pm$0.017 & 0.033$\pm$0.017 & 0.034$\pm$0.017 \\
		Single action & Oracle & letter & 0.032$\pm$0.021 & 0.031$\pm$0.021 & 0.031$\pm$0.021 \\
		\midrule
		Single action & Non-Oracle & yeast & 0.142$\pm$0.112 & 0.117$\pm$0.087 & 0.113$\pm$0.086 \\
		Single action & Non-Oracle & soybean & 0.150$\pm$0.109 & 0.117$\pm$0.083 & 0.128$\pm$0.094 \\
		Single action & Non-Oracle & page-blocks & 0.168$\pm$0.175 & 0.157$\pm$0.180 & 0.157$\pm$0.179 \\
		Single action & Non-Oracle & pendigits & 0.134$\pm$0.098 & 0.114$\pm$0.072 & 0.119$\pm$0.075 \\
		Single action & Non-Oracle & arrhythmia & 0.139$\pm$0.120 & 0.105$\pm$0.095 & 0.107$\pm$0.116 \\
		Single action & Non-Oracle & optdigits & 0.137$\pm$0.100 & 0.138$\pm$0.091 & 0.166$\pm$0.117 \\
		Single action & Non-Oracle & micro-mass & 0.150$\pm$0.106 & 0.136$\pm$0.094 & 0.136$\pm$0.104 \\
		\midrule
		Continuous & Oracle & optdigitsDist & 0.057$\pm$0.010 & 0.065$\pm$0.021 & 0.143$\pm$0.108 \\
		\bottomrule
		\end{tabular}
\end{table}

Figure \ref{fig:kwis_optim_unique_bw_mean_ovr_std_oracle} displays boxplots of the standard deviation of bandwidth values over the mean bandwidth value per experiment, calculated for the Kernel-WIS (Unique) estimator. As can be seen on the plot, that variation of bandwidth values under the Kernel-WIS (Unique) estimator was considerable and it is hypothesised that this lead to a kind of \enquote{overfitting} behaviour. Figure \ref{fig:mean_mean_dist_by_ds_oracle} displays the average kernel output i.e., the average value of $k_{h}(s,s_{j})$ of the  Kernel-WIS (Shared) and Kernel-WIS (Unique) estimators side by side. Clearly, under the Kernel-WIS (Unique) estimator, the distribution of average distances is smaller. This is a result of the overfitting behaviour causing a single or subset of bandwidth values to be extremely large and all others to be small (demonstrated by ratio of standard deviation to mean on figure \ref{fig:kwis_optim_unique_bw_mean_ovr_std_oracle}). In particular, for high dimensional states (i.e., not yeast and page-blocks), the resulting distance was very small, causing the Kernel-WIS (Unique) estimator to behave more like the State-WIS estimator. Both figures \ref{fig:kwis_optim_unique_bw_mean_ovr_std_oracle} and \ref{fig:mean_mean_dist_by_ds_oracle} are displayed for the single action reward, oracle behaviour policies however, the results transfer.\\

\begin{figure}[!h]
	\centering
	\includegraphics[width=\linewidth]{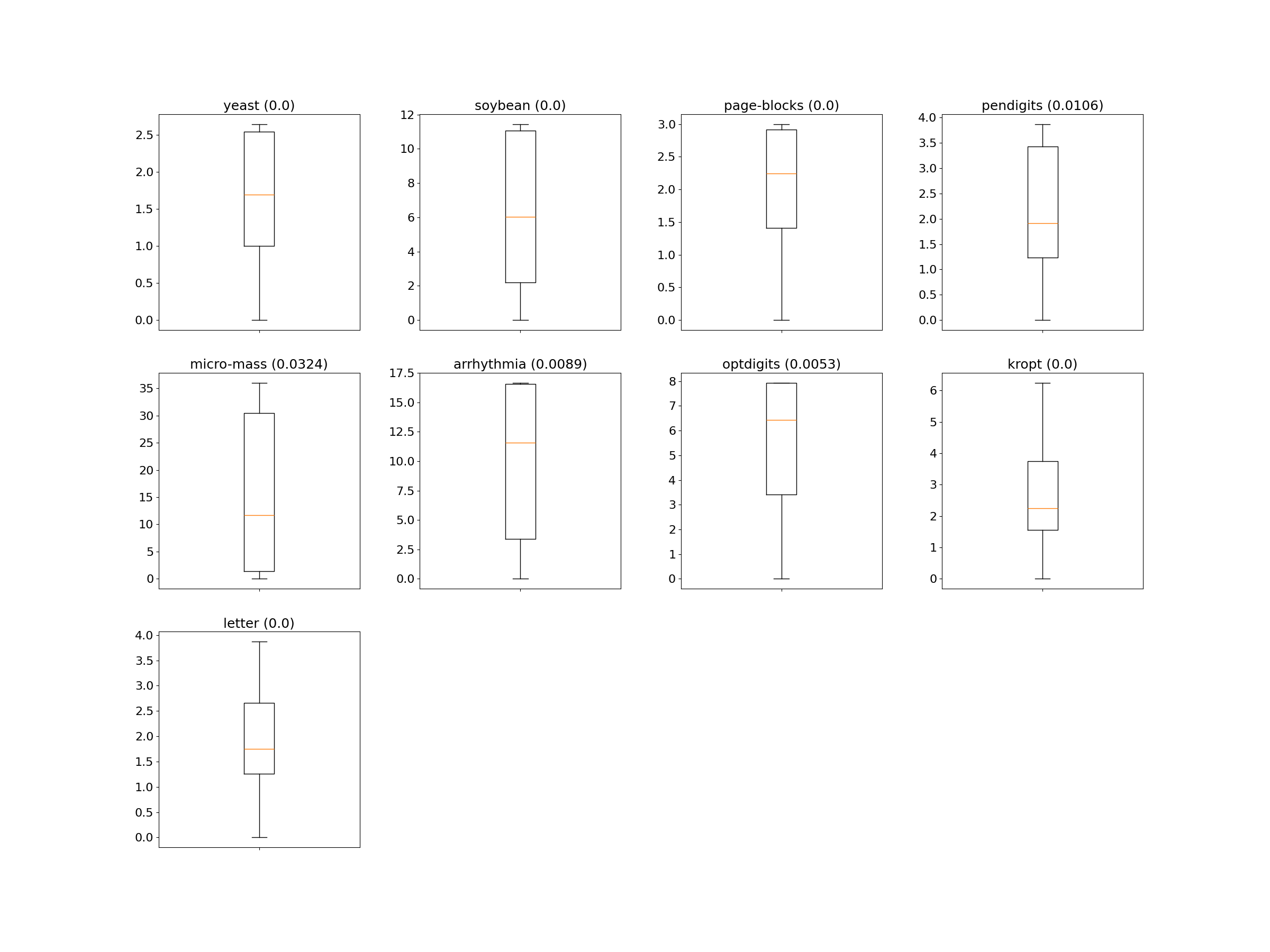}
	\caption{The figure depicts boxplots of the mean bandwidth value over the standard deviation of bandwidth values for the Kernel-WIS (Unique) estimator, split by datasets and calculated under the single action reward, oracle behaviour policy setting. The values in brackets next to the titles define the proportion of standard deviation value that overflowed (due to already extremely large bandwidth values).}
	\label{fig:kwis_optim_unique_bw_mean_ovr_std_oracle}
\end{figure}

\begin{figure}[!h]
	\centering
	\includegraphics[width=\linewidth]{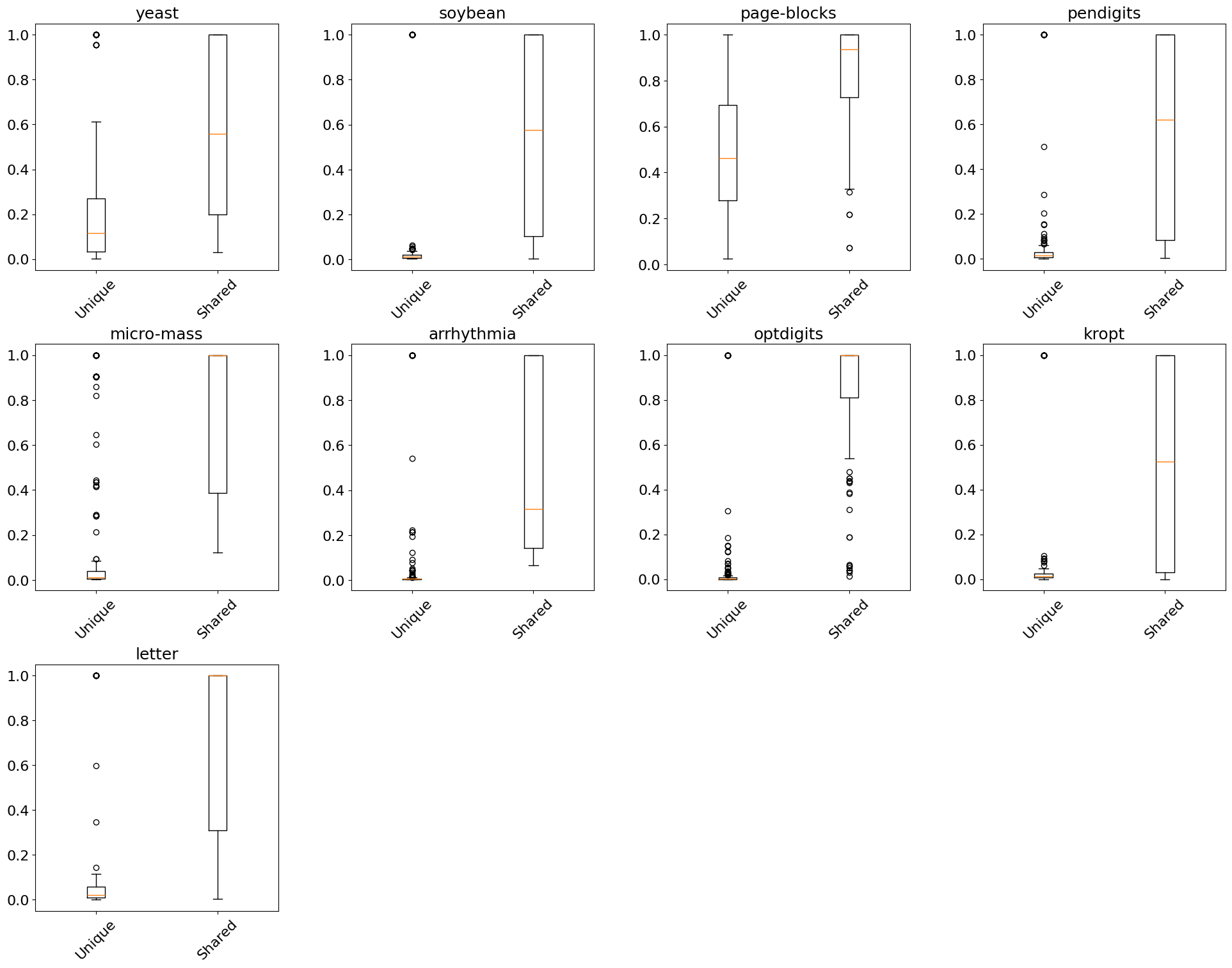}
	\caption[The figure depicts boxplots of the average kernel value, split by Kernel-WIS (Shared) and Kernel-WIS (Unique) and dataset and calculated under the single action reward, oracle behaviour policy setting.]{The figure depicts boxplots of the average kernel value (i.e., the average value of $k_{h}(s,s_{j})$), split by Kernel-WIS (Shared) and Kernel-WIS (Unique) and dataset and calculated under the single action reward, oracle behaviour policy setting.}
	\label{fig:mean_mean_dist_by_ds_oracle}
\end{figure}

An interesting direction of future research would be to consider whether unique bandwidths could be used with an appropriate regularisation e.g., penalising for a large standard deviation. However, given the bandwidth selection procedure is already tangental to the ultimate estimand of interest, this was not explored any further.\\

\FloatBarrier

\subsubsection{Learnt logging policy sensitivity}\label{sec:learnt_logging_sensitivity}
The results in sections \ref{sec:oracle_beh_policy_results} and \ref{sec:non_oracle_beh_sensitivity} suggested the presence of a relationship between the relative performance of Kernel-WIS and WIS and sample-size or action dimension. It was hypothesised that the dependence was on action dimension and this was an artefact of the experimental setup. The construction of the logging and behaviour policies with respect to faulty actions (described in section \ref{sec:logging_behaviour_policy_defs}) was hypothesised to have lead to irregular (with respect to smoothness) policy functions. Policies with different levels of performance were defined by forcing errors across an action, across all states. For datasets with small action spaces (where Kernel-WIS performed worse), this irregularity was thus hypothesised to have a greater effect.\\

To assess this hypothesis, logging and behaviour policies were trained on each of the datasets. Different performance levels were achieved by selecting different epochs in the training process. Figure \ref{fig:summary_performance_cont_policy} describes the median number of times the Kernel-WIS estimator strictly outperforms the WIS estimator and the median difference in mean squared error values under oracle and non-oracle behaviour policies against action dimension (i.e., the same views as displayed in figures \ref{fig:summary_performance_oracle} and \ref{fig:summary_performance_non_oracle}). Arguably, in figure \ref{fig:kwis_performance_no_buffer_by_action_dim_oracle_learnt_policies}, there is a noisy linear relationship between performance and action dimension however, it is certainly not clear. The relative performance of the two estimators is thus not necessarily dependent on action dimension and this needs to be investigated further. It is clear however, that using smoother policies improved the relative performance of the Kernel-WIS estimator. Demonstrated by figures \ref{fig:kwis_performance_no_buffer_by_action_dim_non_oracle_learnt_policies} and \ref{fig:norm_performance_by_action_dim_non_oracle_learnt_policies}, the Kernel-WIS estimator (in aggregation), uniformly outperforms the WIS estimator under non-oracle behaviour policies. Furthermore, under the oracle setting, the worst case and best case Kernel-WIS performance is higher in comparison to the experiments in section \ref{sec:oracle_beh_policy_results}. The results suggest that policy smoothness may affect the performance of the Kernel-WIS estimator, an observation that would be useful to capture within a learning bound (discussed in section \ref{sec:conc}). The itemised results along with discussion on model training and epoch selection are provided in appendix section \ref{sec:appendix_learnt_logging_sensitivity}.\\

\begin{figure}[!h]
	\centering
	\begin{subfigure}{0.49\linewidth}
		\includegraphics[width=\linewidth]{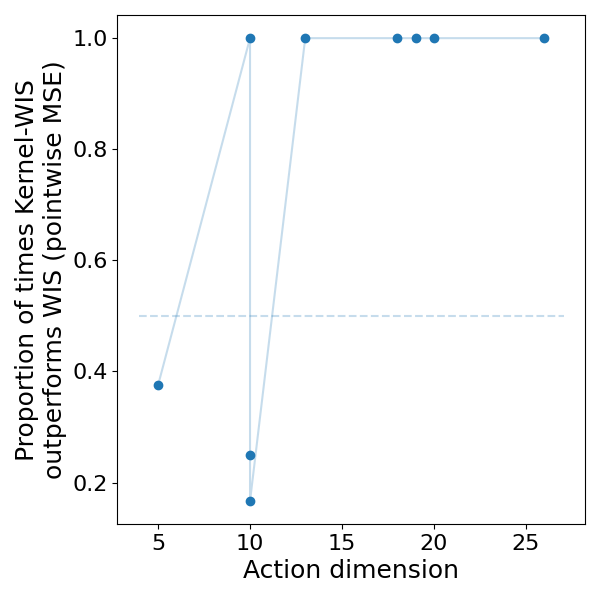}
		\caption{Proportion}
		\label{fig:kwis_performance_no_buffer_by_action_dim_oracle_learnt_policies}
	\end{subfigure}
	\begin{subfigure}{0.49\linewidth}
		\includegraphics[width=\linewidth]{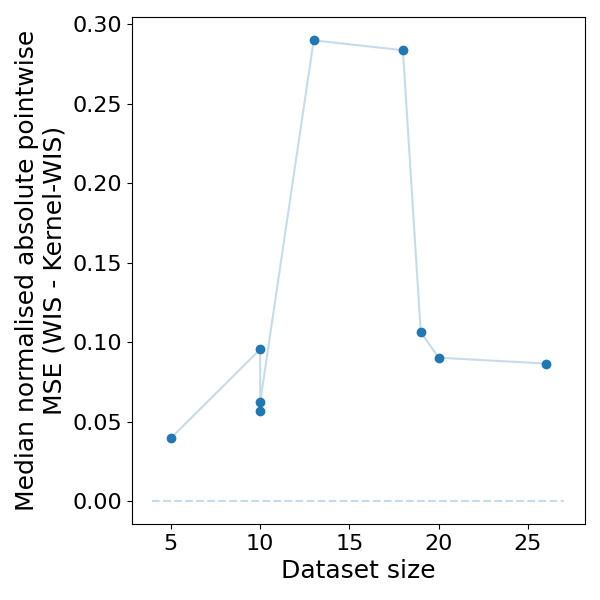}
		\caption{Median normalised difference}
		\label{fig:norm_performance_by_action_dim_oracle_learnt_policies}
	\end{subfigure}
	\begin{subfigure}{0.49\linewidth}
		\includegraphics[width=\linewidth]{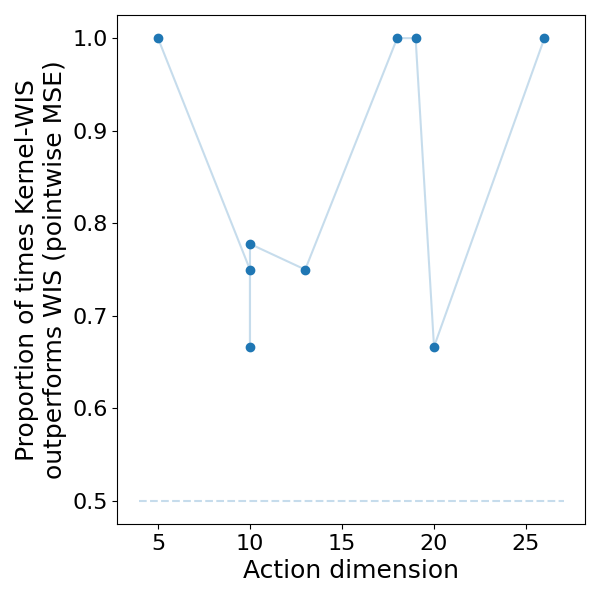}
		\caption{Proportion}
		\label{fig:kwis_performance_no_buffer_by_action_dim_non_oracle_learnt_policies}
	\end{subfigure}
	\begin{subfigure}{0.49\linewidth}
		\includegraphics[width=\linewidth]{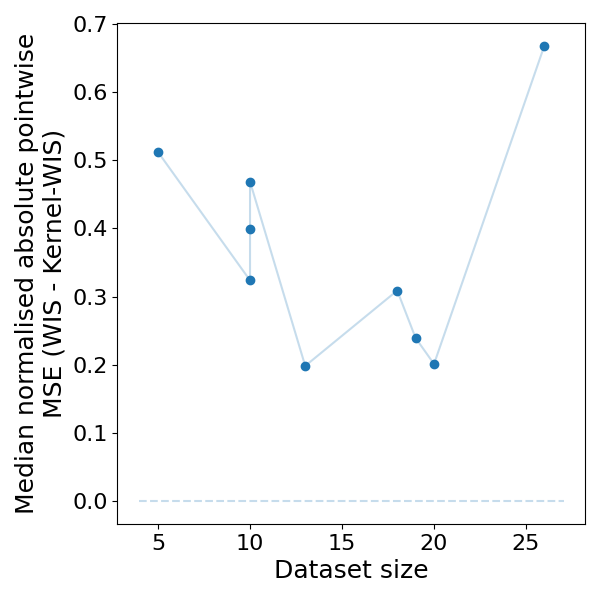}
		\caption{Median normalised difference}
		\label{fig:norm_performance_by_action_dim_non_oracle_learnt_policies}
	\end{subfigure}
	\caption{The figures describe two views on the difference in pointwise mean squared error for the Kernel-WIS and WIS estimators against action dimensionon the x-axis, for oracle (top) and non-oracle (bottom) behaviour policies. Data are shown for the single action reward, learnt logging/behaviour policy setting. The left plots describe the proportion of scenarios (i.e., unique combinations of policy type and buckets differences in true evaluation and logging policy return) where the Kernel-WIS estimator strictly outperforms the WIS estimator. The right plots describe the median normalised difference in pointwise mean squared error ($\textrm{AD}(\textrm{WIS})-\textrm{AD}(\textrm{Kernel-WIS}/(\textrm{AD}(\textrm{WIS})+\textrm{AD}(\textrm{Kernel-WIS}))$).}
	\label{fig:summary_performance_cont_policy}
\end{figure}

\FloatBarrier

\section{Limitations}\label{sec:limitations}

\subsection{Non-uniformity of performance}
The Kernel-WIS estimator does not uniformly outperform the WIS estimator which leaves the analyst with the difficult task of performing estimator selection. The analysis provided some relatively clear boundaries: under discrete rewards, the Kernel-WIS estimator can be used to trade-off variance reduction for a (statistically insignificant) drop in performance under oracle behaviour policies and for a (statistically significant) boost in performance under non-oracle behaviour policies. However, again this is not entirely uniform, demonstrated by the results on the Optdigits dataset in particular. Since it is reasonable to assume that behaviour policies are never oracle, always choosing the Kernel-WIS estimator might be a reasonable approach. Improving performance via stronger bandwidth selection algorithms is anticipated to strengthen this recommendation.\\

\subsection{Reduced coverage}
The experiments using boostrapped confidence intervals suggest coverage under the Kernel-WIS estimator is potentially reduced. The poor coverage results under the Kernel-WIS estimator are most pronounced under non-oracle behaviour policies, summarised in figure \ref{fig:both_cov_prop_non_oracle}. Whilst the coverage under Kernel-WIS does become extremely low, this behaviour does not appear drastically different to the WIS estimator. Indeed, when the WIS estimator performs poorly in pointwise estimation, the coverage of the boostrapped intervals is similarly low. The correlation between pointwise performance and interval converage is expected since the intervals themselves are biased. The evidence presented does suggest a potential loss in coverage under the Kernel-WIS estimator however, an analysis considering unbiased intervals (e.g., PAC or asymptotically normal) is required before strong conclusions can be drawn.\\

\subsection{Increased computation}
The main practical limitation with Kernel-WIS is the challenge in optimising the bandwidth. Even with improved bandwidth selection procedures, it is unlikely that the resulting optimisation landscape will be convex, meaning selecting the bandwidth would still require repeated evaluations of the expensive kernel function. This was largely mitigated through the use of analytic derivatives for the cross-validation approach defined in section \ref{sec:bandwidth_selection} however, this manual derivation limits the generalisability of Kernel-WIS since gradients must be re-derived for any change in the kernel function. Since no research was conducted into the effect of kernel selection, it is unclear how extensive the lack of generalisability will be.\\

\subsection{Theoretical analysis}
The theoretical analysis presented is limited as it assumed fixed evaluation and behaviour policies, and assumes a bandwidth that tends to 0 as $n\rightarrow \infty$. This assumption is not unreasonable: practically, the assumption translates to the size of the test set, $d_{\textrm{test}}$, tending to infinity (assuming the behaviour policy is trained on the training set) and assumes, given a finite dataset of any size, that the bandwidth will be manually forced to 0 as $n \rightarrow \infty$, beginning from the bandwidth selected by the cross validation procedure. It would be beneficial to establish consistency results for a changing behaviour and evaluation policy as well as demonstrate that the sequence of cross validation procedures converges to a bandwidth equal to 0. Intuitively, it is expected that in the limit of infinite data, cross validation will enable an accurate approximation of the random variable $\mathbb{E}[W|S]$ and thus, it is not unrealistic for such a proof to exist. However, this should not be a priority of future work since, as discussed in section \ref{sec:conc}, it is anticipated that better bandwidth selection procedures exist.\\

\subsection{Experimental scenarios}
Finally, whilst the experimental scenarios were non-trivial and expansive in terms of parameters, it would be beneficial to assess the performance of the Kernel-WIS estimator with evaluation and behaviour policy parameterised by common function approximaters such as neural networks. Additionally, it would be beneficial to assess Kernel-WIS considering deterministic, rather than stochastic evaluation policies, since these are common place in off-policy evaluation.\\

\section{Conclusion}
\label{sec:conc}

The analysis presented describes a novel estimator for performing off-policy evaluation in time-independent settings. The estimator was derived under the hypothesis that, whilst the WIS estimator benefits from the $\alpha$-bounded property, the strong interactions between random variables result in the tail behaviour of the estimator being difficult to control. Unfortunately, there was not sufficient scope to theoretically analyse the tail behaviour of the proposed Kernel-WIS estimator however, the analyses provided has laid the foundations for the future in this direction:
\begin{itemize}
	\item The Kernel-WIS was demonstrated to be asymptotically consistent, almost a necessity for any (frequentist) estimator;
	\item With bandwidth selected by cross-validation, the Kernel-WIS estimator demonstrated almost (statistically) identical performance to the WIS estimator under an oracle behaviour policy and single action reward for a wide number of scenarios however;
	\item The Kernel-WIS estimator demonstrated for the majority of settings assessed, (statistically) superior performance to the WIS estimator under mild miss-specification of the behaviour policy and single action reward;
	\item Generally speaking, under the single action reward setting, the WIS estimator only ever marginally outperformed the Kernel-WIS estimator, demonstrated by figures \ref{fig:summary_performance_oracle} and \ref{fig:summary_performance_non_oracle}, however, the Kernel-WIS estimator drastically outperformed the WIS estimator (demonstrated by the same figures and the tests of statistical significance);
	\item Boostrapped confidence intervals under the Kernel-WIS estimator did demonstrate some loss of coverage however, this was minimal under correctly specified behaviour policies. Under miss-specified behaviour policies, the coverage was (understandably) strongly dependent on pointwise performance due to the intervals themselves being biased (discussed further below). Finally;
	\item Under a continuous reward setting, the performance of the Kernel-WIS estimator degraded significantly.
\end{itemize}

Overall, the analysis demonstrated that there is utility in considering the Kernel-WIS estimator over and above the WIS estimator. Arguably, the Kernel-WIS estimator should always be preferred due to the magnitude of performance improvement that can be attained and the superior performance under miss-specified behaviour policies. It is anticipated that more research into bandwidth selection would provide further performance improvements since the proposed approach via cross-validation is somewhat tangental to the actual goal of the estimator: the estimand is $\mathbb{E}_{p_{\pi_{e}}}[R]$ and not $\mathbb{E}_{p_{\pi_{\beta}}}[W|S]$.\\

It is anticipated that the most impactful next step for the Kernel-WIS estimator would be to establish a PAC-Bayes result. This would enable optimisation of the bandwidth hyperparameter to be directly associated with the tail performance of the estimator, rather than tangentially as in the process described in section \ref{sec:bandwidth_selection}. The drop in performance of the Kernel-WIS estimator on some of the datasets (e.g., optdigits) could not be explained however, a learning bound might provide some insight. Additional impactful next steps might be to extend the estimator to the full, time dependent, Markov decision process setting and establish asymptotic normality results. Future work should additionally explore the reasons behind the poor performance of the Kernel-WIS estimator under continuous reward settings.\\


\acks{Thank you to Professor Benjamin Guedj for discussions regarding the general approach to the consistency proof. Joshua Spear is funded by Great Ormond Street Hospital Charity.
}


\newpage

\appendix

\section{Theoretical results}\label{sec:ther_results_appendix}
The proceeding section provides supporting analysis for theorem \ref{ther:asympt_concis_bound_wsa}. Table \ref{tbl:assump_justi} describes each of the assumptions required to derive the result and their practical relevance. As can be seen, none of the assumptions are unusual or restrictive (outside of the common causal inference assumptions such as correct identification of the propensity score model). Theorem \ref{ther:asympt_concis_bound_wsa} is also restated and a full proof is given.\\

{\small\tabcolsep=3pt  
	\begin{longtable}{>{\centering\arraybackslash}p{0.05\linewidth}|>{\centering\arraybackslash}p{0.45\linewidth}|>{\centering\arraybackslash}p{0.45\linewidth}}
		\caption{The table describes the assumptions required for the theoretical analysis of the Kernel-WIS estimator, along with the justification for the applicability of the assumption to standard statistical analysis pipelines.}
		\label{tbl:assump_justi}\\
		\toprule
		\# & Assumption & Implication \\
		\midrule
		\endhead            
		\midrule\multicolumn{3}{r}{{Continued on next page}} \\ 
		\endfoot     
		\endlastfoot
			1 & $\|k_{h}(x)\|<\infty$ & Satisfied by many standard kernels e.g., RBF.\\
			\midrule
			2 & For some $C>0$ and $\nu>0$, $N(\epsilon,\mathcal{K})\leq C\epsilon^{-\nu}$ where $\mathcal{K} = \{k_{h}((x-\cdot)/h^{-\frac{1}{d}}):h>0,x\in \mathbb{R}^{d}\}$ & Satisfied by many standard kernels e.g., RBF. Refer to \cite{einmahl_uniform_2005} for a further explanation.\\
			\midrule
			3 & $\mathcal{K}$ is pointwise measurable & Standard measurability assumption. Pervasive in almost all statistical analyses.\\
			\midrule
			4 & The support of $k_{h}$ is $[-0.5,0.5]^{d}$ & This can be achieved through a simple rescaling of the input data, a process which is common in many statistical analysis pipelines.\\
			\midrule
			5 & $\int k_{h}(x)d(x) = 1$ & Specifies the kinds of kernels that are applicable and notably covers common exponential types e.g., RBF and Matern. Due to the self-normalising property of the estimator, this does not prohibit the use of kernels in their un-normalised form.\\
			\midrule
			6 & Assume $f$ defines the marginal density of $S$ and, $f$ is uniformly Lipschitz continuous and strictly positive on $I^{\epsilon}$, where $I$ is a compact set in $\mathbb{R}^{d}$ and $I^{\epsilon} = \{\max_{1\leq i \leq d}|s_{i}| \leq \epsilon\}$ & Impossible to validate but not particularly strict since no specific bound is placed in the Lipschitz constant.\\
			\midrule
			7 & There exists $M$ such that almost surely, $|W|\mathbbm{1}(S\in I^{\epsilon}) \leq M$ and $W\geq w_{\textrm{min}} >0$ & Bounded weights are implied by the positivity assumption of causal inference so the assumption is uncontroversial.\\
			\midrule 
			8 & $R \in [0,1]$ & Standard assumption for contextual bandits \\
			\midrule
			9 & Correct identification of the propensity model i.e., $\hat{p}_{\pi_{l}} = p_{\pi_{l}}$ & \\
		\bottomrule
	\end{longtable}
}
\newpage

\begin{theorem}[Asymptotic consistency of Kernel-WIS (Theorem \ref{ther:asympt_concis_bound_wsa} repeated)] Under assumptions 1 to 9, for $0<a_{n}<h<b_{n}<1$, $b_{n} \rightarrow 0$ and $\frac{na_{n}}{\log n}\rightarrow\infty$, the Kernel-WIS estimator is asymptotically consistent:
\begin{align*}
	\hat{J}_{\textrm{Kernel-WIS}} \xrightarrow{a.s.} \mathbb{E}[R]
\end{align*}
where: 
\begin{align*}
	\hat{J}_{\textrm{Kernel-WIS}} = \frac{1}{n}\sum_{i=1}r_{i}w_{i}\Bigg(\frac{\sum_{j=1}^{n}k_{h}(s_{i},s_{j})w_{i}}{\sum_{j=1}^{n}k_{h}(s_{i},s_{j})}\Bigg)^{-1}.
\end{align*}
\end{theorem}

\begin{theorem}[Theorem 2, \cite{einmahl_uniform_2005}]\label{ther:theorem_2_einmahl} Assume conditions 1 to 8 are satisfied. Then, for large enough $c>0$ and any $b_{n}\downarrow 0$:
	\begin{align*}
		\limsup_{n\rightarrow \infty} \sup_{\frac{c\log n}{n\leq h \leq b_{n}}} \frac{\sqrt{nh}\|\hat{m}_{n,h}-\frac{\bar{q}(\cdot,h)}{\bar{f}(\cdot,h)}\|_{I}}{\sqrt{\log(h^{-1}) \vee \log\log n}} < \infty
	\end{align*}
	where:
	\begin{align*}
		\hat{m}_{n,h}(s) =& \frac{\sum_{i=1}^{n}W_{i}k_{h}((s-S_{i})h^{-\frac{1}{d}})}{\sum_{i=1}^{n}k_{h}((s-S_{i})h^{-\frac{1}{d}})} \\
		\bar{q}(\cdot,h) =& \mathbb{E}[Wk_{h}((s-S)h^{-\frac{1}{d}})]/h \\
		\bar{f}(\cdot,h) =& \mathbb{E}[k_{h}((s-S)h^{-\frac{1}{d}})]/h 
	\end{align*} 
\end{theorem}
As the authors state, a clear corollary to this is that, almost surely: 
\begin{align*}
	\limsup_{n\rightarrow \infty} \sup_{\frac{c\log n}{n\leq h \leq b_{n}}} \Bigg\|\hat{m}_{n,h}-\frac{\bar{q}(\cdot,h)}{\bar{f}(\cdot,h)}\Bigg\|_{I} = 0.
\end{align*}

\begin{corollary}[Corollary 2, \cite{einmahl_uniform_2005}]\label{corol:corol_2_einmahl}
	Under assumption 1 to 8 and assuming $p(w,s)$ is the joint density of $(W,S)$ with the marginal density given by:
	\begin{align*}
		f(s) = \int p(w,s)dw
	\end{align*}
	then, for $0<a_{n}<b_{n}<1$, $b_{n} \rightarrow 0$ and $\frac{na_{n}}{\log n}\rightarrow\infty$:
	\begin{align*}
	\lim_{n\rightarrow \infty}\sup_{a_{n}\leq h \leq b_{n}}\|\hat{m}_{n,h} - \mathbb{E}[W|S=s]\|_{I} \xrightarrow{a.s.} 0.
	\end{align*}
\end{corollary}

\begin{proof} Let $f:\mathcal{S}\times\mathcal{W}\times\mathcal{R}\rightarrow\mathbb{R} \in \mathcal{F}$ define the space of functions of the form:
\begin{align*}
	f_{k,n} = wr\Bigg(\frac{\sum_{i=1}^{n}k_{h}(s_{i},s)w_{i}}{\sum_{i=1}^{n}k_{h}(s_{i},s)}\Bigg)^{-1}, 
\end{align*}

Throughout, the observation that $\frac{\sum_{i=1}^{n}k_{h}(s_{i},s)w_{i}}{\sum_{i=1}^{n}k_{h}(s_{i},s)}$, $f_{n}$ and $\hat{J}_{\textrm{Kernel-WIS}}$ are all bounded is frequently used:
\begin{align*}
	w_{\textrm{min}} \leq \frac{\sum_{i=1}^{n}k_{h}(s_{i},s)w_{i}}{\sum_{i=1}^{n}k_{h}(s_{i},s)} \leq w_{\textrm{max}}
\end{align*}
\begin{align*}
	0 \leq \frac{wr}{w_{\textrm{max}}} \leq f_{n} \leq \frac{wr}{w_{\textrm{min}}} \leq \frac{w_{\textrm{max}}}{w_{\textrm{min}}}
\end{align*}
\begin{align*}
	0 \leq \hat{J}_{\textrm{Kernel-WIS}} \leq \frac{w_{\textrm{max}}}{w_{\textrm{min}}}.
\end{align*}

Observe that:
\begin{align}
	\lim_{n\rightarrow\infty}&\sup_{a_{n}\leq h \leq b_{n}}\|\hat{J}_{\textrm{Kernel-WIS}} - \mathbb{E}_{p_{\pi_{e}}}[R]\|_{I\times \mathcal{W} \times \mathcal{R}} \nonumber \\ 
	\leq & \lim_{n\rightarrow\infty}\sup_{a_{n}\leq h \leq b_{n}}\Big\|\hat{J}_{\textrm{Kernel-WIS}} - \frac{1}{n}\sum_{i=1}^{n}w_{i}r_{i}\mathbb{E}[W|S=s_{i}]^{-1}\Big\|_{I\times \mathcal{W} \times \mathcal{R}} \nonumber \\
	& + \lim_{n\rightarrow\infty}\sup_{a_{n}\leq h \leq b_{n}}\Big\|\frac{1}{n}\sum_{i=1}^{n}w_{i}r_{i}\mathbb{E}[W|S=s_{i}]^{-1} - \mathbb{E}_{p_{\pi_{e}}}[R]\Big\|_{I\times \mathcal{W} \times \mathcal{R}}. \label{equ:consist_decomp}
\end{align}

Beginning with:
\begin{align*}
	\lim_{n\rightarrow\infty}\sup_{a_{n}\leq h \leq b_{n}}\Big\|\hat{J}_{\textrm{Kernel-WIS}} - \frac{1}{n}\sum_{i=1}^{n}w_{i}r_{i}\mathbb{E}[W|S=s_{i}]^{-1}\Big\|_{I\times \mathcal{W} \times \mathcal{R}},
\end{align*}
from corollary \ref{corol:corol_2_einmahl}: 
\begin{align*}
	\lim_{n\rightarrow \infty}\sup_{a_{n}\leq h \leq b_{n}} \Bigg\|\frac{\sum_{i=1}^{n}k_{h}(s_{i},s)w_{i}}{\sum_{i=1}^{n}k_{h}(s_{i},s)} - \mathbb{E}[W|S=s]\Bigg\|_{I} \xrightarrow{a.s.} 0,
\end{align*}
then, for $a_{n}\leq h \leq b_{n}$, uniformly in $s\in I$: 
\begin{align*}
	\frac{\sum_{i=1}^{n}k_{h}(s_{i},s)w_{i}}{\sum_{i=1}^{n}k_{h}(s_{i},s)} \xrightarrow{a.s.} \mathbb{E}[W|S=s].
\end{align*}

Consider the transformation: 
\begin{align*}
	f_{n}(w,r,s;\omega) = wr\Bigg(\frac{\sum_{i=1}^{n}k_{h}(s_{i},s)w_{i}}{\sum_{i=1}^{n}k_{h}(s_{i},s)}\Bigg)^{-1}.
\end{align*}
Since $0<w_{\textrm{min}} < w_{\textrm{max}} < \infty$, then the function, $f_{n}(w,r,s;\omega)$, is continuous and thus, by the continuous mapping theorem (\cite{vaart_asymptotic_1998}), subject to $a_{n}\leq h \leq b_{n}$ and uniformly in $s \in I$:
\begin{align*}
	f_{n}(w,r,s;\omega) \xrightarrow{a.s.} wr\mathbb{E}[W|S=s]^{-1}.
\end{align*}
To obtain convergence uniformly in $w$ and $r$, observe that, given both $W$ and $R$ are defined over bounded domains, subject to $a_{n}\leq h \leq b_{n}$:
\begin{align*}
	|f_{n}(w,r,s;\omega) - wr\mathbb{E}[W|S=s]^{-1}| \leq  \sup_{w,r}\|f_{n}(w,r,s;\omega) - wr\mathbb{E}[W|S=s]^{-1}\|_{I}& \\
	\leq \sup_{w,r}\Bigg\|wr \Bigg(\Bigg(\frac{\sum_{i=1}^{n}k_{h}(s_{i},s)w_{i}}{\sum_{i=1}^{n}k_{h}(s_{i},s)}\Bigg)^{-1} - \mathbb{E}[W|S=s]^{-1}\Bigg)\Bigg\|_{I} &. 
\end{align*}
As such:
\begin{align}
	\sup_{a_{n}\leq h \leq b_{n}}\|f_{n}(w,r,s;\omega) - wr\mathbb{E}[W|S=s]^{-1}\|_{I\times \mathcal{W}\times \mathcal{R}} \xrightarrow{a.s.} 0. \label{equ:fkn_as_conv}
\end{align}
Thus, almost surely: 
\begin{align}
	\lim_{n\rightarrow\infty}&\sup_{a_{n}\leq h \leq b_{n}}\Big\|\hat{J}_{\textrm{Kernel-WIS}} - \frac{1}{n}\sum_{i=1}^{n}w_{i}r_{i}\mathbb{E}[W|S=s_{i}]^{-1}\Big\|_{I\times \mathcal{W} \times \mathcal{R}} \nonumber\\
	= &\lim_{n\rightarrow\infty}\sup_{a_{n}\leq h \leq b_{n}}\Big\|\frac{1}{n}\sum_{i=1}^{n}\Big(f_{n}(w_{i},r_{i},s_{i};\omega) - w_{i}r_{i}\mathbb{E}[W|S=s_{i}]^{-1}\Big)\Big\|_{I\times \mathcal{W} \times \mathcal{R}} \nonumber\\
	\leq &\lim_{n\rightarrow\infty}\frac{1}{n}\sum_{i=1}^{n}\sup_{a_{n}\leq h \leq b_{n}}\Big\|f_{n}(w_{i},r_{i},s_{i};\omega) - w_{i}r_{i}\mathbb{E}[W|S=s_{i}]^{-1}\Big\|_{I\times \mathcal{W} \times \mathcal{R}} = 0, \label{equ:lhs_as_conv}
\end{align}
where equation \ref{equ:lhs_as_conv} follows from equation \ref{equ:fkn_as_conv}.\\

For:
\begin{align}
	\lim_{n\rightarrow\infty}\sup_{a_{n}\leq h \leq b_{n}}\Big\|\frac{1}{n}\sum_{i=1}^{n}w_{i}r_{i}\mathbb{E}[W|S=s_{i}]^{-1} - \mathbb{E}_{p_{\pi_{e}}}[R]\Big\|_{I\times \mathcal{W} \times \mathcal{R}},
\end{align}

observe that:
\begin{align*}
	\mathbb{E}[W|S=s] = \int \frac{p_{\pi_{e}}(A|S))}{p_{\pi_{\beta}}(A|S)}p_{\pi_{\beta}}(A|S)da = 1.
\end{align*}
As such by the strong law of large numbers and under the correct identification of the propensity score:
\begin{align*}
	\limsup_{n\rightarrow\infty}\sup_{a_{n}\leq h \leq b_{n}}\Big\|\frac{1}{n}\sum_{i=1}^{n}w_{i}r_{i}\mathbb{E}[W|S=s]^{-1} - \mathbb{E}_{p_{\pi_{e}}}[R]\Big\|_{I\times \mathcal{W} \times \mathcal{R}} = 0	.
\end{align*}
These two steps, under the decomposition in equation \ref{equ:consist_decomp}, imply:
\begin{align*}
	P\Big(\lim_{n\rightarrow \infty}\sup_{a_{n}\leq h \leq b_{n}}\|\hat{J}_{\textrm{Kernel-WIS}} - \mathbb{E}_{p_{\pi_{e}}}[R]\|_{I\times \mathcal{W} \times \mathcal{R}} > \epsilon\Big) = 0,
\end{align*}
and thus, under $a_{n}\leq h \leq b_{n}$:
\begin{align*}
	\hat{J}_{\textrm{Kernel-WIS}} - \mathbb{E}_{p_{\pi_{e}}}[R] \xrightarrow{a.s.} 0.
\end{align*}
	
\end{proof}

\FloatBarrier

\section{Derivative of Nadaraya–Watson kernel regressor with RBF kernel}\label{sec:appendix_nw_kernel_grad}

The proceeding describes the analytic gradient for the Nadaraya–Watson kernel regressor with RBF kernel function, with respect to the bandwidth, assuming it is a one-dimensional real value (section \ref{sec:oned_grad}) and a real valued vector (section \ref{sec:vec_val_grad}). In both cases, the gradient of interest is with respect to $\log h$, since the optimisation is performed over $\log h$ to circumvent the need for explicit constraints. 

\subsection{One-dimensional gradient}\label{sec:oned_grad}
To begin, observe the following decomposition by repeated application of the chain rule:
\begin{align}
	\frac{\partial \textrm{MSE}(h)}{\partial \log h} =& \frac{\partial \textrm{MSE}(h)}{\partial h}\Bigg(\frac{\partial \log h}{\partial h}\Bigg)^{-1} \nonumber \\
	=& \frac{\partial \textrm{MSE}(h)}{\partial \hat{y}(s_{i},h)}\frac{\partial \hat{y}(s_{i},h)}{\partial h}\Bigg(\frac{\partial \log h}{\partial h}\Bigg)^{-1} \nonumber  \\
	=& \frac{\partial \textrm{MSE}(h)}{\partial u(w_{i},s_{i},\hat{y})}\frac{\partial u(w_{i},s_{i},\hat{y})}{\partial \hat{y}(s_{i},h)}\frac{\partial \hat{y}(s_{i},h)}{\partial h}\Bigg(\frac{\partial \log h}{\partial h}\Bigg)^{-1} \label{equ:mse_deriv_decomp}
\end{align}
for: 
\begin{align*}
	\textrm{MSE}(h) =& \frac{1}{n}\sum (y_{i}-\hat{y}_{i})^{2} \\
	u(w_{i},s_{i},\hat{y}) =& (w_{i}-\hat{y}_{i}(s_{i})) \\
	\hat{y}(s_{i},h) =& \sum_{j=1}^{n}k(h,s_{i},s_{j})w_{j}\Big(\sum_{j=1}^{n}k(h,s_{i},s_{j})\Big)^{-1} \\
	k(h,s_{i},s_{j}) =& \exp(-0.5h^{-2}||s_{i}-s_{j}||^{2}) \\
	||s_{i}-s_{j}||^{2} =& \sum_{k} (s_{i,k}-s_{j,k})^{2}.
\end{align*}
Of the terms in equation \ref{equ:mse_deriv_decomp}, $\frac{\partial \textrm{MSE}(h)}{\partial u(w_{i},s_{i},\hat{y})}$, $\frac{\partial u(w_{i},s_{i},\hat{y})}{\partial \hat{y}(s_{i},h)}$ and $(\frac{\partial \log h}{\partial h})^{-1}$ are trivial to define and are listed below: 
\begin{align*}
	\frac{\partial \textrm{MSE}(h)}{\partial u(w_{i},s_{i},\hat{y})} =& \frac{2}{n}\sum_{i=1}^{n}u(w_{i},s_{i},\hat{y})  = \frac{2}{n}\sum_{i=1}^{n}(w_{i}-\hat{y}_{i}(s_{i}))\\
	\frac{\partial u(w_{i},s_{i},\hat{y})}{\partial \hat{y}(s_{i},h)} =& -1 \\
	\Bigg(\frac{\partial \log h}{\partial h}\Bigg)^{-1} =& h.
\end{align*}
$\frac{\partial \hat{y}(s_{i},h)}{\partial h}$, however, requires a little more care. First let:
\begin{align*}
	f(s_{i},h) =& \sum_{j=1}^{n}k(h,s_{i},s_{j})w_{j} \\
	g(s_{i},h) =& \sum_{j=1}^{n}k(h,s_{i},s_{j}),
\end{align*}
and thus:
\begin{align*}
	\frac{\partial f(s_{i},h)}{\partial h} = \sum_{j=1}^{n}h^{-3}||s_{i,l}-s_{j,l}||^{2}k(h,s_{i},s_{j})w_{j} \\
	\frac{\partial g(s_{i},h)}{\partial h} = \sum_{j=1}^{n}h^{-3}||s_{i,l}-s_{j,l}||^{2}k(h,s_{i},s_{j}).
\end{align*}
By the quotient rule:
\begin{align*}
	\frac{\partial \hat{y}(s_{i},h)}{\partial h} =& \frac{\frac{\partial f(s_{i},h)}{\partial h}g(s_{i},h) - \frac{\partial g(s_{i},h)}{\partial h}f(s_{i},h)}{g(s_{i},h)^{2}} \\
	=& \frac{\sum_{j=1}^{n}h^{-3}||s_{i,l}-s_{j,l}||^{2}k(h,s_{i},s_{j})w_{j}}{g(s_{i},h)} \\
	& - \frac{\sum_{j=1}^{n}h^{-3}||s_{i,l}-s_{j,l}||^{2}k(h,s_{i},s_{j})}{g(s_{i},h)}\hat{y}(s_{i},h) \\
	= & \sum_{j=1}^{n}g(s_{i},h)^{-1}h^{-3}||s_{i,l}-s_{j,l}||^{2}k(h,s_{i},s_{j})(w_{j}-\hat{y}(s_{i},h)).
\end{align*}
As such: 
\begin{align}
	&\frac{\partial \textrm{MSE}(h)}{\partial \log h} \nonumber\\
	& = \frac{2}{n}\sum_{i=1}^{n}(w_{i}-\hat{y}_{i}(s_{i}))(-1)h\sum_{j=1}^{n}g(s_{i},h)^{-1}h^{-3}||s_{i,l}-s_{j,l}||^{2}k(h,s_{i},s_{j})(w_{j}-\hat{y}(s_{i},h)) \nonumber \\
	& = \frac{-2}{n}\sum_{i=1}^{n}(w_{i}-\hat{y}_{i}(s_{i}))g(s_{i},h)^{-1}\sum_{j=1}^{n}h^{-2}||s_{i,l}-s_{j,l}||^{2}k(h,s_{i},s_{j})(w_{j}-\hat{y}(s_{i},h)), \label{equ:mse_deriv}
\end{align}
where $||x-x'||^{2} = \sum_{k}(x_{k}-x_{k}')^{2}$.
%

\subsection{Vector valued gradient}\label{sec:vec_val_grad}
To begin, observe the following decomposition by repeated application of the chain rule:
\begin{align}
	\frac{\partial \textrm{MSE}(h)}{\partial \log h_{l}} =& \frac{\partial \textrm{MSE}(h)}{\partial h_{l}}\Bigg(\frac{\partial \log h_{l}}{\partial h_{l}}\Bigg)^{-1} \nonumber \\
	=& \frac{\partial \textrm{MSE}(h)}{\partial \hat{y}(s_{i},h)}\frac{\partial \hat{y}(s_{i},h)}{\partial h_{l}}\Bigg(\frac{\partial \log h_{l}}{\partial h_{l}}\Bigg)^{-1} \nonumber  \\
	=& \frac{\partial \textrm{MSE}(h)}{\partial u(w_{i},s_{i},\hat{y})}\frac{\partial u(w_{i},s_{i},\hat{y})}{\partial \hat{y}(s_{i},h)}\frac{\partial \hat{y}(s_{i},h)}{\partial h_{l}}\Bigg(\frac{\partial \log h_{l}}{\partial h_{l}}\Bigg)^{-1}, \label{equ:mse_deriv_decomp_vec}
\end{align}
for: 
\begin{align*}
	\textrm{MSE}(h) =& \frac{1}{n}\sum (y_{i}-\hat{y}_{i})^{2} \\
	u(w_{i},s_{i},\hat{y}) =& (w_{i}-\hat{y}_{i}(s_{i})) \\
	\hat{y}(s_{i},h) =& \sum_{j=1}^{n}k(h,s_{i},s_{j})w_{j}\Big(\sum_{j=1}^{n}k(h,s_{i},s_{j})\Big)^{-1} \\
	k(h,s_{i},s_{j}) =& \exp(-0.5h^{-2}||s_{i}-s_{j}||^{2}) \\
	||s_{i}-s_{j}||^{2} =& \sum_{k} (s_{i,k}-s_{j,k})^{2}.
\end{align*}
Of the terms in equation \ref{equ:mse_deriv_decomp_vec}, $\frac{\partial \textrm{MSE}(h)}{\partial u(w_{i},s_{i},\hat{y})}$, $\frac{\partial u(w_{i},s_{i},\hat{y})}{\partial \hat{y}(s_{i},h)}$ and $(\frac{\partial \log h}{\partial h})^{-1}$ are trivial to define and are listed below: 
\begin{align*}
	\frac{\partial \textrm{MSE}(h)}{\partial u(w_{i},s_{i},\hat{y})} =& \frac{2}{n}\sum_{i=1}^{n}u(w_{i},s_{i},\hat{y})  = \frac{2}{n}\sum_{i=1}^{n}(w_{i}-\hat{y}_{i}(s_{i}))\\
	\frac{\partial u(w_{i},s_{i},\hat{y})}{\partial \hat{y}(s_{i},h)} =& -1 \\
	\Bigg(\frac{\partial \log h_{l}}{\partial h_{l}}\Bigg)^{-1} =& h_{l}.
\end{align*}
$\frac{\partial \hat{y}(s_{i},h)}{\partial h}$, however, requires a little more care. First let:
\begin{align*}
	f(s_{i},h) =& \sum_{j=1}^{n}k(h,s_{i},s_{j})w_{j} \\
	g(s_{i},h) =& \sum_{j=1}^{n}k(h,s_{i},s_{j}),
\end{align*}
and thus:
\begin{align*}
	\frac{\partial f(s_{i},h)}{\partial h_{l}} = \sum_{j=1}^{n}h^{-3}(s_{i,l}-s_{j,l})^{2}k(h,s_{i},s_{j})w_{j} \\
	\frac{\partial g(s_{i},h)}{\partial h_{l}} = \sum_{j=1}^{n}h^{-3}(s_{i,l}-s_{j,l})^{2}k(h,s_{i},s_{j}).
\end{align*}
By the quotient rule:
\begin{align*}
	\frac{\partial \hat{y}(s_{i},h)}{\partial h_{l}} =& \frac{\frac{\partial f(s_{i},h)}{\partial h_{l}}g(s_{i},h) - \frac{\partial g(s_{i},h)}{\partial h_{l}}f(s_{i},h)}{g(s_{i},h)^{2}} \\
	=& \frac{\sum_{j=1}^{n}h^{-3}(s_{i,l}-s_{j,l})^{2}k(h,s_{i},s_{j})w_{j}}{g(s_{i},h)} \\
	& - \frac{\sum_{j=1}^{n}h^{-3}(s_{i,l}-s_{j,l})^{2}k(h,s_{i},s_{j})}{g(s_{i},h)}\hat{y}(s_{i},h) \\
	= & \sum_{j=1}^{n}g(s_{i},h)^{-1}h^{-3}(s_{i,l}-s_{j,l})^{2}k(h,s_{i},s_{j})(w_{j}-\hat{y}(s_{i},h)).
\end{align*}
As such: 
\begin{align*}
	&\frac{\partial \textrm{MSE}(h)}{\partial \log h_{l}} \\
	& = \frac{2}{n}\sum_{i=1}^{n}(w_{i}-\hat{y}_{i}(s_{i}))(-1)h\sum_{j=1}^{n}g(s_{i},h)^{-1}h^{-3}(s_{i,l}-s_{j,l})^{2}k(h,s_{i},s_{j})(w_{j}-\hat{y}(s_{i},h)) \\
	& = \frac{-2}{n}\sum_{i=1}^{n}(w_{i}-\hat{y}_{i}(s_{i}))g(s_{i},h)^{-1}\sum_{j=1}^{n}h^{-2}(s_{i,l}-s_{j,l})^{2}k(h,s_{i},s_{j})(w_{j}-\hat{y}(s_{i},h)) 
\end{align*}

\section{OptdigitsDist dataset generation}\label{sec:optdigits_gen_appendix}

The code used to construct and train the variational autoencoder (VAE) for generation of the optdigitsDist dataset was directly copied from \cite{francis_variational-autoencoder-for-mnist_2023} except for some inconsequential changes to the metric logging.\\

Figure \ref{fig:vae_training_metrics} displays the reconstruction error and Kullback–Leibler (KL) divergence for the trained VAE. The VAE was concluded to reasonably converge due to both elements of the loss function stabilising away from 0.\\

\begin{figure}[!h]
	\centering
	\begin{subfigure}{0.49\linewidth}
		\includegraphics[width=\linewidth]{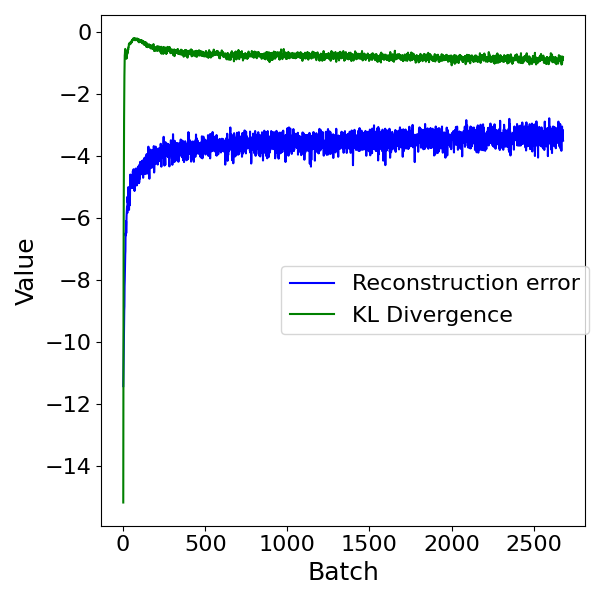}
		\caption{Training dataset}
	\end{subfigure}
	\begin{subfigure}{0.49\linewidth}
		\includegraphics[width=\linewidth]{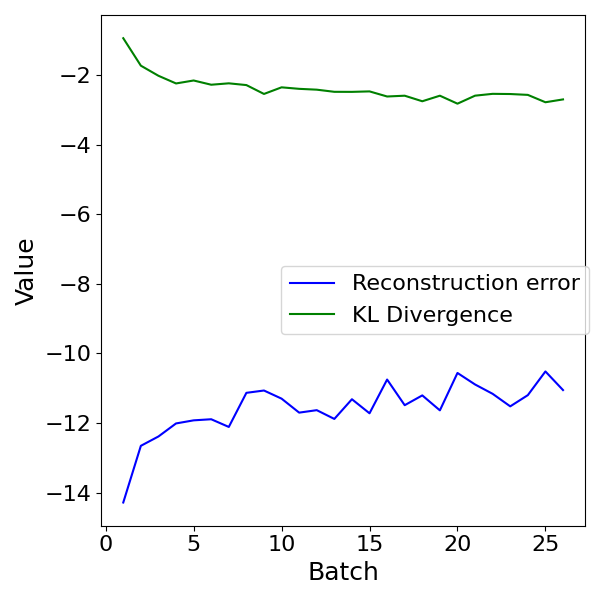}
		\caption{Testing dataset}
	\end{subfigure}
	\caption{This figure displays the reconstruction error and Kullback–Leibler divergence for the variational autoencoder used to derive the continuous reward optdigits setting (optdigitsDist). The metrics are displayed for both the training set and the testing set.}
	\label{fig:vae_training_metrics}
\end{figure}

\FloatBarrier

\section{Learnt logging policy sensitivity}\label{sec:appendix_learnt_logging_sensitivity}
The proceeding describes the training process and itemised results for the sensitivity analysis performed under learnt logging and behaviour policies, described in section \ref{sec:learnt_logging_sensitivity}.\\

\subsection{Policy training}
Two model architectures were considered for the evaluation: a linear model (herein referred to as linear) and a fully connected multi-layer perceptron with a single hidden layer (herein referred to as one-layer). For all datasets, both models were trained to convergence using a maximum of 150 epochs. Additionally, a batch size of 32 was used except for on the micro-mass and soybean datasets where a batch size of 12 was used, owing to the smaller sample size. For the one-layer networks, a hidden layer of dimension 64 was used, except for on the page-blocks dataset where 32 was used to slow down the model convergence, enabling policies of different performance levels to be selected.\\

Epochs were selected for deriving the experiment policies according accuracy i.e.,:
\begin{align*}
	\frac{\textrm{TP}+\textrm{TN}}{\textrm{TP}+\textrm{TN}+\textrm{FP}+\textrm{FN}}
\end{align*}
where $\textrm{TP}$,$\textrm{TN}$,$\textrm{FP}$ and $\textrm{FN}$ denote true positives, true negatives, false positives and false negatives, respectively. The epochs selected according the validation accuracy followed:
\begin{itemize}
	\item The earliest epoch associated with the maximum validation score, rounded to the nearest 0.05 and;
	\item The earliest epoch associated with the maximum validation score (rounded to the nearest 0.05) minus 0.4.
\end{itemize}
For example, under the optdigits dataset with a linear model, the maximum score achieved was 0.97. This was rounded down to 0.95 and the earliest epoch achieving at least this score was epoch 46. $0.95-0.4=0.55$ and the earliest epoch achieving this score was 6. Thus, epochs 6 and 46 were used. Where the maximum score achieved was less than 0.5, a minimum validation score of 0.1 was used.\\

\subsection{Results}
The aggregated results are described in figure \ref{sec:learnt_logging_sensitivity}. Tables \ref{tbl:k_wis_wis_oracle_results_learnt_policies} and \ref{tbl:k_wis_wis_non_oracle_results_learnt_policies} describe the performance of the WIS and Kernel-WIS estimators, split by evaluation and logging policy divergence, respectively, identically to the results described in sections \ref{sec:oracle_beh_policy_results} and \ref{sec:non_oracle_beh_sensitivity}.\\

{\small\tabcolsep=3pt  
	\begin{longtable}{cccccccc}
	\caption{The table describes the pointwise finite mean squared error of the WIS estimator and Kernel-WIS estimator (with a shared bandwidth), split by dataset, policy type and a discretisation of the absolute difference in true performance under the evaluation and logging policy, for oracle behaviour policies with learnt logging and behaviour policies. The P-Value column defines the p-value of a two-sided Wald t-test, comparing the finite mean squared error of the WIS and Kernel-WIS estimators. N observations defines the number of experiments aggregated over per row, where the aggregations are across logging/behaviour model architecture and epoch, and evaluation policy. Values in bold in the WIS and Kernel-WIS columns highlight the best performing estimator. Neither are highlighted if the mean performance is within 0.001. Values highlighted in the P-Value column identify those which are less than or equal to 0.05.}
	\label{tbl:k_wis_wis_oracle_results_learnt_policies} \\
	\toprule
	Dataset & Policy Type & \makecell{Eval/Log\\Abs. Diff.} & WIS & Kernel-WIS & N Obs. & P-Value\\
	\midrule
	\endfirsthead
	\toprule
	Dataset & Policy Type & \makecell{Eval/Log\\Abs. Diff.} & WIS & Kernel-WIS & N Obs. & P-Value\\
	\midrule
	\endhead            
	\midrule\multicolumn{7}{r}{{Continued on next page}} \\ 
	\endfoot     
	\endlastfoot
	arrhythmia & IW & (0.0, 0.2] & 0.101$\pm$0.032 & \textbf{0.081$\pm$0.028} & 6 & 0.28 \\
	arrhythmia & IW & (0.2, 0.4] & 0.075$\pm$0.005 & \textbf{0.055$\pm$0.002} & 2 & 0.07 \\
	arrhythmia & SNIW & (0.0, 0.2] & 0.091$\pm$0.026 & \textbf{0.070$\pm$0.028} & 6 & 0.21 \\
	arrhythmia & SNIW & (0.2, 0.4] & 0.078$\pm$0.001 & \textbf{0.073$\pm$0.001} & 2 & 0.05 \\
	\midrule
	kropt & IW & (0.0, 0.2] & 0.052$\pm$0.006 & \textbf{0.049$\pm$0.008} & 8 & 0.41 \\
	kropt & SNIW & (0.0, 0.2] & 0.044$\pm$0.008 & \textbf{0.033$\pm$0.016} & 4 & 0.29 \\
	kropt & SNIW & (0.2, 0.4] & 0.110$\pm$0.003 & \textbf{0.059$\pm$0.054} & 4 & 0.15 \\
	\midrule
	letter & IW & (0.0, 0.2] & 0.040$\pm$0.017 & \textbf{0.038$\pm$0.019} & 4 & 0.89 \\
	letter & IW & (0.2, 0.4] & 0.031$\pm$0.012 & 0.030$\pm$0.012 & 4 & 0.84 \\
	letter & SNIW & (0.2, 0.4] & 0.066$\pm$0.061 & 0.065$\pm$0.061 & 6 & 0.97 \\
	letter & SNIW & (0.6, 0.8] & 0.093$\pm$0.005 & \textbf{0.062$\pm$0.003} & 2 & \textbf{0.03} \\
	\midrule
	micro-mass & IW & (0.0, 0.2] & 0.064$\pm$0.016 & \textbf{0.054$\pm$0.020} & 8 & 0.29 \\
	micro-mass & SNIW & (0.0, 0.2] & 0.086$\pm$0.020 & 0.086$\pm$0.020 & 5 & 0.99 \\
	micro-mass & SNIW & (0.2, 0.4] & 0.090$\pm$0.023 & 0.090$\pm$0.023 & 3 & 1.0 \\
	\midrule
	optdigits & IW & (0.0, 0.2] & 0.080$\pm$0.013 & \textbf{0.078$\pm$0.011} & 4 & 0.86 \\
	optdigits & IW & (0.6, 0.8] & 0.061$\pm$0.008 & 0.061$\pm$0.009 & 4 & 0.93 \\
	optdigits & SNIW & (0.0, 0.2] & \textbf{0.084$\pm$0.005} & 0.087$\pm$0.005 & 2 & 0.57 \\
	optdigits & SNIW & (0.6, 0.8] & 0.062$\pm$0.011 & 0.062$\pm$0.011 & 6 & 1.0 \\
	\midrule
	page-blocks & IW & (0.0, 0.2] & 0.052$\pm$0.010 & 0.053$\pm$0.010 & 3 & 0.87 \\
	page-blocks & IW & (0.2, 0.4] & 0.072$\pm$nan & 0.072$\pm$nan & 1 & nan \\
	page-blocks & IW & (0.4, 0.6] & 0.052$\pm$0.003 & \textbf{0.050$\pm$0.003} & 2 & 0.63 \\
	page-blocks & IW & (0.6, 0.8] & 0.078$\pm$0.002 & 0.077$\pm$0.001 & 2 & 0.63 \\
	page-blocks & SNIW & (0.0, 0.2] & 0.039$\pm$nan & 0.039$\pm$nan & 1 & nan \\
	page-blocks & SNIW & (0.2, 0.4] & \textbf{0.043$\pm$0.019} & 0.056$\pm$0.009 & 3 & 0.38 \\
	page-blocks & SNIW & (0.4, 0.6] & \textbf{0.047$\pm$0.004} & 0.049$\pm$0.000 & 2 & 0.63 \\
	page-blocks & SNIW & (0.6, 0.8] & \textbf{0.074$\pm$0.002} & 0.077$\pm$0.002 & 2 & 0.32 \\
	\midrule
	pendigits & IW & (0.0, 0.2] & \textbf{0.072$\pm$0.002} & 0.073$\pm$0.003 & 2 & 0.69 \\
	pendigits & IW & (0.2, 0.4] & 0.043$\pm$0.001 & 0.043$\pm$0.001 & 2 & 1.0 \\
	pendigits & IW & (0.6, 0.8] & 0.074$\pm$0.001 & 0.075$\pm$0.000 & 4 & 0.43 \\
	pendigits & SNIW & (0.0, 0.2] & 0.025$\pm$0.005 & 0.025$\pm$0.005 & 2 & 1.0 \\
	pendigits & SNIW & (0.4, 0.6] & 0.208$\pm$0.008 & \textbf{0.161$\pm$0.006} & 2 & \textbf{0.03} \\
	pendigits & SNIW & (0.6, 0.8] & 0.073$\pm$0.001 & 0.073$\pm$0.001 & 4 & 1.0 \\
	\midrule
	soybean & IW & (0.0, 0.2] & 0.085$\pm$0.032 & \textbf{0.083$\pm$0.034} & 7 & 0.91 \\
	soybean & IW & (0.2, 0.4] & 0.105$\pm$nan & 0.105$\pm$nan & 1 & nan \\
	soybean & SNIW & (0.0, 0.2] & 0.115$\pm$0.002 & \textbf{0.111$\pm$0.003} & 4 & 0.05 \\
	soybean & SNIW & (0.2, 0.4] & 0.089$\pm$0.001 & 0.089$\pm$0.001 & 2 & 1.0 \\
	soybean & SNIW & (0.4, 0.6] & 0.349$\pm$0.004 & \textbf{0.034$\pm$0.001} & 2 & \textbf{0.0} \\
	\midrule
	yeast & IW & (0.0, 0.2] & 0.057$\pm$0.007 & \textbf{0.052$\pm$0.007} & 5 & 0.27 \\
	yeast & IW & (0.2, 0.4] & 0.061$\pm$0.004 & 0.060$\pm$0.004 & 3 & 0.89 \\
	yeast & SNIW & (0.2, 0.4] & 0.055$\pm$0.019 & \textbf{0.053$\pm$0.017} & 8 & 0.8 \\
	\bottomrule
	\end{longtable}
}

{\small\tabcolsep=3pt  
	\begin{longtable}{cccccccc}
	\caption{The table describes the pointwise finite mean squared error of the WIS estimator and Kernel-WIS estimator (with a shared bandwidth), split by dataset, policy type and a discretisation of the absolute difference in true performance under the evaluation and logging policy, for non-oracle behaviour policies with learnt logging and behaviour policies. The P-Value column defines the p-value of a two-sided Wald t-test, comparing the finite mean squared error of the WIS and Kernel-WIS estimators. N observations defines the number of experiments aggregated over per row, where the aggregations are across logging/behaviour model architecture and epoch, and evaluation policy. Values in bold in the WIS and Kernel-WIS columns highlight the best performing estimator. Neither are highlighted if the mean performance is within 0.001. Values highlighted in the P-Value column identify those which are less than or equal to 0.05.}
	\label{tbl:k_wis_wis_non_oracle_results_learnt_policies} \\
	\toprule
	Dataset & Policy Type & \makecell{Eval/Log\\Abs. Diff.} & WIS & Kernel-WIS & N Obs. & P-Value\\
	\midrule
	\endfirsthead
	\toprule
	Dataset & Policy Type & \makecell{Eval/Log\\Abs. Diff.} & WIS & Kernel-WIS & N Obs. & P-Value\\
	\midrule
	\endhead            
	\midrule\multicolumn{7}{r}{{Continued on next page}} \\ 
	\endfoot     
	\endlastfoot
	arrhythmia & IW & (0.0, 0.2] & 0.121$\pm$0.065 & \textbf{0.074$\pm$0.032} & 17 & \textbf{0.01} \\
	arrhythmia & IW & (0.2, 0.4] & 0.159$\pm$0.078 & \textbf{0.156$\pm$0.076} & 7 & 0.95 \\
	arrhythmia & SNIW & (0.0, 0.2] & 0.154$\pm$0.096 & \textbf{0.142$\pm$0.091} & 23 & 0.65 \\
	arrhythmia & SNIW & (0.2, 0.4] & \textbf{0.046$\pm$nan} & 0.116$\pm$nan & 1 & nan \\
	\midrule
	kropt & IW & (0.0, 0.2] & 0.112$\pm$0.047 & \textbf{0.078$\pm$0.072} & 23 & 0.07 \\
	kropt & IW & (0.2, 0.4] & 0.240$\pm$nan & 0.240$\pm$nan & 1 & nan \\
	kropt & SNIW & (0.0, 0.2] & 0.076$\pm$0.044 & \textbf{0.067$\pm$0.056} & 13 & 0.68 \\
	kropt & SNIW & (0.2, 0.4] & 0.400$\pm$0.382 & \textbf{0.067$\pm$0.074} & 11 & \textbf{0.02} \\
	\midrule
	letter & IW & (0.0, 0.2] & 0.083$\pm$0.038 & \textbf{0.064$\pm$0.041} & 16 & 0.21 \\
	letter & IW & (0.2, 0.4] & 0.149$\pm$0.191 & \textbf{0.058$\pm$0.049} & 8 & 0.23 \\
	letter & SNIW & (0.0, 0.2] & 0.095$\pm$0.059 & \textbf{0.092$\pm$0.063} & 7 & 0.94 \\
	letter & SNIW & (0.2, 0.4] & 0.341$\pm$0.250 & \textbf{0.147$\pm$0.092} & 11 & \textbf{0.03} \\
	letter & SNIW & (0.6, 0.8] & 0.637$\pm$0.137 & \textbf{0.136$\pm$0.102} & 6 & \textbf{0.0} \\
	\midrule
	micro-mass & IW & (0.0, 0.2] & 0.111$\pm$0.050 & \textbf{0.090$\pm$0.058} & 24 & 0.2 \\
	micro-mass & SNIW & (0.0, 0.2] & 0.172$\pm$0.150 & \textbf{0.166$\pm$0.159} & 18 & 0.91 \\
	micro-mass & SNIW & (0.2, 0.4] & \textbf{0.178$\pm$0.128} & 0.187$\pm$0.158 & 6 & 0.92 \\
	\midrule
	optdigits & IW & (0.0, 0.2] & \textbf{0.091$\pm$0.105} & 0.100$\pm$0.100 & 16 & 0.81 \\
	optdigits & IW & (0.2, 0.4] & \textbf{0.059$\pm$0.005} & 0.085$\pm$0.025 & 4 & 0.13 \\
	optdigits & IW & (0.6, 0.8] & 0.115$\pm$0.051 & \textbf{0.099$\pm$0.044} & 4 & 0.65 \\
	optdigits & SNIW & (0.2, 0.4] & 0.484$\pm$0.019 & \textbf{0.142$\pm$0.023} & 2 & \textbf{0.0} \\
	optdigits & SNIW & (0.4, 0.6] & 0.427$\pm$0.253 & \textbf{0.309$\pm$0.236} & 7 & 0.39 \\
	optdigits & SNIW & (0.6, 0.8] & 0.437$\pm$0.287 & \textbf{0.342$\pm$0.250} & 15 & 0.34 \\
	\midrule
	page-blocks & IW & (0.0, 0.2] & 0.114$\pm$0.079 & \textbf{0.084$\pm$0.059} & 8 & 0.4 \\
	page-blocks & IW & (0.2, 0.4] & 0.147$\pm$0.124 & \textbf{0.079$\pm$0.069} & 12 & 0.12 \\
	page-blocks & IW & (0.4, 0.6] & 0.092$\pm$0.061 & \textbf{0.057$\pm$0.006} & 4 & 0.33 \\
	page-blocks & SNIW & (0.0, 0.2] & 0.245$\pm$0.135 & \textbf{0.208$\pm$0.141} & 7 & 0.62 \\
	page-blocks & SNIW & (0.2, 0.4] & 0.314$\pm$0.137 & \textbf{0.171$\pm$0.245} & 3 & 0.44 \\
	page-blocks & SNIW & (0.4, 0.6] & 0.246$\pm$0.241 & \textbf{0.167$\pm$0.142} & 9 & 0.41 \\
	page-blocks & SNIW & (0.6, 0.8] & 0.504$\pm$0.214 & \textbf{0.472$\pm$0.264} & 3 & 0.88 \\
	page-blocks & SNIW & (0.8, 1.0] & 0.610$\pm$0.005 & \textbf{0.442$\pm$0.090} & 2 & 0.23 \\
	\midrule
	pendigits & IW & (0.0, 0.2] & 0.074$\pm$0.028 & \textbf{0.068$\pm$0.018} & 14 & 0.46 \\
	pendigits & IW & (0.2, 0.4] & \textbf{0.090$\pm$0.085} & 0.098$\pm$0.014 & 5 & 0.85 \\
	pendigits & IW & (0.4, 0.6] & 0.123$\pm$0.111 & \textbf{0.074$\pm$0.044} & 3 & 0.54 \\
	pendigits & IW & (0.6, 0.8] & \textbf{0.092$\pm$0.002} & 0.104$\pm$0.001 & 2 & \textbf{0.02} \\
	pendigits & SNIW & (0.0, 0.2] & 0.454$\pm$0.055 & \textbf{0.054$\pm$0.009} & 4 & \textbf{0.0} \\
	pendigits & SNIW & (0.2, 0.4] & 0.279$\pm$0.175 & \textbf{0.228$\pm$0.170} & 5 & 0.65 \\
	pendigits & SNIW & (0.4, 0.6] & 0.392$\pm$0.125 & \textbf{0.312$\pm$0.072} & 5 & 0.26 \\
	pendigits & SNIW & (0.6, 0.8] & 0.380$\pm$0.360 & \textbf{0.149$\pm$0.065} & 8 & 0.11 \\
	pendigits & SNIW & (0.8, 1.0] & 0.767$\pm$0.002 & \textbf{0.089$\pm$0.024} & 2 & \textbf{0.02} \\
	\midrule
	soybean & IW & (0.0, 0.2] & 0.089$\pm$0.046 & \textbf{0.069$\pm$0.047} & 17 & 0.22 \\
	soybean & IW & (0.2, 0.4] & 0.120$\pm$0.047 & 0.119$\pm$0.047 & 7 & 0.98 \\
	soybean & SNIW & (0.0, 0.2] & 0.155$\pm$0.072 & \textbf{0.139$\pm$0.048} & 10 & 0.58 \\
	soybean & SNIW & (0.2, 0.4] & 0.199$\pm$0.126 & 0.198$\pm$0.127 & 8 & 0.98 \\
	soybean & SNIW & (0.4, 0.6] & 0.509$\pm$0.163 & \textbf{0.315$\pm$0.099} & 6 & \textbf{0.04} \\
	\midrule
	yeast & IW & (0.0, 0.2] & 0.108$\pm$0.039 & \textbf{0.080$\pm$0.024} & 19 & \textbf{0.01} \\
	yeast & IW & (0.2, 0.4] & 0.093$\pm$0.011 & \textbf{0.071$\pm$0.036} & 5 & 0.25 \\
	yeast & SNIW & (0.0, 0.2] & 0.218$\pm$0.102 & \textbf{0.175$\pm$0.062} & 13 & 0.21 \\
	yeast & SNIW & (0.2, 0.4] & \textbf{0.148$\pm$0.144} & 0.155$\pm$0.180 & 11 & 0.92 \\
	\bottomrule
	\end{longtable}
}

\FloatBarrier

\vskip 0.2in
\bibliography{bib}

\end{document}